\documentclass{article}

\usepackage{microtype}
\usepackage{graphicx,makecell}

\usepackage{hyperref}

\usepackage[accepted]{icml2023}

\usepackage{subfigure}
\usepackage{amsmath, amsthm, amssymb, multirow, paralist}
\usepackage{booktabs}
\usepackage{url} 
\usepackage{fancybox} 
\newtheorem{thm}{Theorem}

\newtheorem{lemma}{Lemma}
\newtheorem{cor}{Corollary}
\newtheorem{definition}[cor]{Definition}

\usepackage{enumitem,color}
\usepackage{textcase}
\usepackage{graphicx}
\usepackage[T1]{fontenc}

\def \X {\mathcal{X}}

\def \R {\mathbb{R}}

\def \w {\mathbf{w}}

\def \x {\mathbf{x}}

\def \E {\mathrm{E}}

\def \x {\mathbf{x}}
\def \p {\mathbf{p}}
\def \a {\mathbf{a}}

\def \e {\mathbf{e}}

\def \c {\mathbf{c}}

\def \1 {\mathbf{1}}

\def \y {\mathbf{y}}
\def \u {\mathbf{u}}
\def \uh {\widehat{\u}}

\def \P {\mathcal{P}}

\def \B {\mathcalB}

\def \fh {\widehat f}

\def \y {\mathbf{y}}
\def \E {\mathrm{E}}
\def \x {\mathbf{x}}

\def \u {\mathbf{u}}

\def \w {\mathbf{w}}

\def \R {\mathbb{R}}

\def \q {\mathbf{q}}

\def \c {\mathbf{c}}

\def \p {\mathbf{p}}
\def \q {\mathbf{q}}

\def \f {\mathbf{f}}
\def \a {\mathbf{a}}

\def \B {\mathcal{B}}

\def \fh {\widehat{\f}}

\def \X {\mathcal{X}}

\def \P {\mathbb{P}}

\icmltitlerunning{Label Distributionally Robust Losses for Multi-class Classification}

\begin{document}

\twocolumn[
\icmltitle{Label Distributionally Robust Losses  for Multi-class Classification: \\Consistency, Robustness and Adaptivity}




\begin{icmlauthorlist}
\icmlauthor{Dixian Zhu}{uiowa}
\icmlauthor{Yiming Ying}{albany}
\icmlauthor{Tianbao Yang}{tamu}
\end{icmlauthorlist}
\icmlaffiliation{uiowa}{The University of Iowa, Iowa City, USA}
\icmlaffiliation{tamu}{Texas A\&M University, College Station, USA}
\icmlaffiliation{albany}{University at Albany, Albany, USA}

\icmlcorrespondingauthor{Dixian Zhu, Tianbao Yang}{dixian-zhu@uiowa.edu, tianbao-yang@tamu.edu}

\icmlkeywords{Machine Learning, ICML}

\vskip 0.3in
]



\printAffiliationsAndNotice{}  

\begin{abstract}
We study a family of loss functions named label-distributionally robust (LDR) losses for multi-class classification that are formulated from distributionally robust optimization (DRO) perspective, where the uncertainty in the given label information are modeled and captured by taking the worse case of distributional weights. The benefits of this perspective  are several fold: (i) it provides a unified framework to explain the classical cross-entropy (CE) loss and SVM loss and their variants, (ii) it includes a special family corresponding to the temperature-scaled CE loss, which is widely adopted but poorly understood; (iii) it allows us to achieve adaptivity to the uncertainty degree of label information at an instance level.  Our contributions include: (1) we study both consistency and robustness  by establishing top-$k$ ($\forall k\geq 1$) consistency of LDR losses for multi-class classification, and a negative result that a top-$1$ consistent and symmetric robust loss cannot achieve top-$k$ consistency simultaneously for all $k\geq 2$; (2) we propose a new adaptive LDR loss that automatically adapts the individualized temperature parameter to the noise degree of class label of each instance; (3) we demonstrate  stable and competitive performance for the proposed adaptive LDR loss on 7 benchmark datasets under 6 noisy label and 1 clean settings against 13 loss functions, and on one real-world noisy dataset. The method is open-sourced at \url{https://github.com/Optimization-AI/ICML2023_LDR}.
\end{abstract}

\setlength{\textfloatsep}{2pt}

\setlength\abovedisplayskip{1pt}
\setlength\belowdisplayskip{1pt}
\section{Introduction}
In the past decades, multi-class classification has been used in a variety of areas, including image recognition~\cite{krizhevsky2009cifar,deng2009imagenet, geusebroek2005amsterdam,Hsu02multisvm}, natural language processing~\cite{lang1995newsweeder,DBLP:journals/corr/abs-2005-14165,DBLP:journals/corr/abs-1801-06146}, etc.  Loss functions have played an important role in multi-class classification~\cite{DBLP:journals/corr/JanochaC17,kornblith2021why}. 

In machine learning (ML), the most classical loss functions include  the cross-entropy (CE) loss and the Crammer-Singer (CS) Loss (aka the SVM loss). Before the deep learning era, the CS loss have been widely studied and utilized for learning a linear or kernelized model~\cite{10.5555/944790.944813}. However, in the deep learning era the CE loss seems to dominate the CS loss in practical usage~\cite{he2016deep}. One reason is that the CE loss is consistent for multi-class classification (i.e., yielding Bayes optimal classifier with infinite number of samples), while the CS loss is generally not consistent unless under an exceptional condition that the maximum conditional probability of a class label given the input  is larger than $0.5$~\cite{10.5555/1005332.1044701}, which might hold for simple tasks such as MNIST classification but unlikely holds for complicated tasks, e.g., ImageNet classification~\cite{DBLP:journals/corr/abs-2006-07159,tsipras2020imagenet}. Nevertheless, the CE loss has been found to be vulnerable to noise in the labels~\cite{ghosh2017robust,wang2019symmetric}. Hence, efforts have been devoted to robustifying the CE loss such that it becomes noise-tolerant. A sufficient condition for the loss function to be noise-tolerant is known as the symmetric property~\cite{ghosh2017robust}, i.e.,  the sum of losses over all possible class labels for a given data point is a constant.  Following the symmetric property, multiple variants of CE loss have been proposed to enjoy the noise-tolerance property. However, a symmetric loss alone is usually not enough to achieve good performance in practice due to slow convergence~\cite{zhang2018generalized,ma2020normalized}. Hence, a loss that interpolates between the CE loss and a symmetric loss usually achieves better performance on real-world noisy datasets~\cite{wang2019symmetric,zhang2018generalized,englesson2021generalized,ma2020normalized}, e.g., on the WebVision data with 20\% noisy data~\cite{DBLP:journals/corr/abs-1708-02862}. 

These existing results can be explained from the degree of uncertainty of the given label information.  If the given label information of an instance is certain (e.g., MNIST data), then the CS loss is a good choice~\cite{DBLP:journals/corr/Tang13}; if the given label information is little uncertain (e.g., Imagenet Data), then the CE loss is a better choice; if the given label information is highly uncertain (e.g., WebVision data), then a loss that interpolates between the CE loss and a symmetric loss wins. While the parameter tuning (e.g., the weights for combining the CE loss and a symmetric loss) can achieve dataset adaptivity,  an open problem exists:  "\textit{Can we define a loss to achieve instance adaptivity, i.e., adaptivity  to the  uncertainty degree of label information of each instance?}"

To address this question, this paper proposes to capture the uncertainty in the given label information of each instance by using the distributionally robust optimization (DRO) framework. For any instance, we assign each class label a distributional weight and define the loss in the worst case of the distributional weights properly regularized by a function, which is referred to as  label distributionally robust (LDR) losses. The benefits of using the DRO framework to model the uncertainty of the given label information include: (i) it provides a unified framework of many existing loss functions, including the CE loss, the CS loss, and some of their variants, e.g., top-$k$ SVM loss. (ii)  it includes an interesting family of losses by choosing the regularization function of the distributional weights as Kullback?Leibler (KL) divergence, which is referred to as LDR-KL loss for multi-class classification. The LDR-KL loss is closely related to the temperature-scaled CE loss that has been widely used in other scenarios~\cite{HinVin15Distilling,chen2020simple}, where the temperature corresponds to the regularization parameter of the distribibutional weights. Hence, our theoretical analysis of LDR-KL loss can potentially shed light on using the temperature-scaled CE loss in different scenarios. (iii) The LDR-KL loss can interpolate between the CS loss (at one end), the (margin-aware) CE loss (in between) and a symmetric loss (at the other end) by varying the regularization parameter, hence providing more flexibility for adapting to the noisy degree of given label information. Our contributions are the following: 
\begin{itemize}[noitemsep,topsep=0pt,parsep=0pt,partopsep=0pt]
\item We 
establish a strong consistency of  the LDR-KL loss namely All-$k$ consistency that is top-$k$ consistency (corresponding to the Bayes optimal top-$k$ error)  for all $k\geq 1$, and a sufficient condition for the general family of LDR losses to be All-$k$ consistent. 
\item We show that an extreme case of the LDR-KL loss satisfies the symmetric property, making it tolerant to label noise. In addition, we establish a negative result that a non-negative symmetric loss function cannot achieve top-$k$ consistency for $k=1$ and $k=2$ simultaneously. Hence, we conclude that existing top-$1$ consistent symmetric losses cannot be All-$k$ consistent. 
\item 
We propose an adaptive LDR-KL loss, which allows us to automatically learn the personalized temperature parameter of each instance adapting to its own noisy degree. We develop an efficient  algorithm for empirical risk minimization based on the adaptive LDR-KL loss without incurring much additional computational burden. 
\item We conduct extensive experiments on multiple datasets under different noisy label settings, and demonstrate the stable and competitive performance of the proposed adaptive LDR-KL loss in all settings. 
\end{itemize}

\begin{table*}[ht]
    \centering
 \caption{Comparison between different loss functions. Notice that for SCE, we simplify the loss combination by convex combination ($\beta=1-\alpha$ and $0\le\alpha\le1$). Detailed expressions of different loss functions are included in the Appendix~\ref{sec:lss}. `NA' means not available or unknown. * means that the results are derived by us (proofs are included in the Appendix~\ref{sec:discuss-loss}). Expressions in parentheses mean the conditions about the hyper-parameters under which the properties hold.  
}
    \label{table:loss-intros}
    \resizebox{1.\textwidth}{!}{
    \begin{tabular}{l|l|l|l|l|l}
    \toprule
       Category&Loss & Top-$1$ consistent & All-$k$ consistent & Symmetric&\makecell{instance\\ adaptivity}\\
       \hline
        &CE &  Yes&  Yes& No&No\\
       Traditional &CS~\cite{10.5555/944790.944813} & No& No& No&No\\
        &WW~\cite{weston1998multi}  & No& No& No&No\\
        \hline 
        \multirow{2}*{Symmetric}&MAE (RCE, clipped)~\cite{ghosh2017robust} & Yes& No & Yes&No\\
        &NCE~\cite{ma2020normalized}&Yes&No&Yes &No\\
        &RLL~\cite{patel2021memorization}& No &  No& Yes&No\\
\hline
       Interpolation between &GCE~\cite{zhang2018generalized}  & Yes& Yes$^*$ ($0\le q <1$)& Yes ($q=1$)&No\\
       CE and MAE &SCE~\cite{wang2019symmetric}& Yes& Yes$^*$~($0\le\alpha<1$)& Yes ($\alpha=0$)&No\\
       &JS~\cite{englesson2021generalized}&NA &NA&Yes~($\pi_1=1$)&No\\
        \hline
         \multirow{2}*{Bounded}&MSE~\cite{ghosh2017robust}  &Yes & Yes$^*$ &No&No\\
        &TGCE~\cite{zhang2018generalized}  &NA &NA & No&No\\
        \hline 
         \multirow{2}*{Asymmetric}& AGCE~\cite{zhou2021asymmetric}&Yes&No&No&No\\
        & AUL~\cite{zhou2021asymmetric}& Yes&No&No&No\\
        \hline
         \multirow{2}*{LDR}&LDR-KL (this work)&\multicolumn{2}{|c|}{Yes ($\lambda>0$)}&Yes~($\lambda=\infty$)&No\\
                &ALDR-KL (this work)& \multicolumn{2}{|c|}{Yes ($\lambda_0>0, \alpha>\frac{\log K}{\lambda_0}$)}&Yes ($\lambda_0=\infty$)&Yes\\
        \bottomrule
    \end{tabular}}
    \vspace*{-0.2in}
\end{table*}
\section{Related Work}
{\bf Traditional multi-class loss functions.} Multi-class loss functions have been studied extensively in conventional ML literature~\cite{weston1998multi,tewari2007consistency,10.5555/1005332.1044701,10.5555/944790.944813}. Besides the CE loss,  well-known loss functions include the CS loss and the  Weston-Watkins (WW) loss. Consistency has been studied for these loss for achieving Bayes optimal solution under infinite amount of data. The CE loss is shown to be consistent (producing Bayes optimal classifier for minimizing zero-one loss, i.e., top-$1$ error), while the CS and WW loss  are shown to be inconsistent~\cite{10.5555/1005332.1044701}. Variants of these standard losses have been developed for minimizing top-$k$ error ($k>1$), such as (smoothed) top-$k$ SVM loss, top-$k$ CE loss~\cite{NIPS2015_0336dcba,DBLP:conf/cvpr/Lapin0S16,DBLP:journals/pami/LapinHS18,yang2020consistency}. Top-$k$ consistency has been studied in these papers showing some losses are top-$k$ consistent while other are not.  However, the relationship between these losses and the CE loss is still unclear, which hinders the understanding of their insights. Hence, it is important to provide a unified view of these different losses, which could not only help us understand the existing losses but also provide new and potentially better choices.

{\bf Robust Loss functions.} While may different methods have been developed for tackling label noise~\cite{goldberger2017training,DBLP:journals/corr/abs-1805-08193,NEURIPS2018_ad554d8c,DBLP:conf/cvpr/PatriniRMNQ17,zhang2021learning,madry2017towards,zhang2018generalized,wang2019symmetric}, of most relevance to the present work are the symmetric losses and their combinations with the CE loss or its variants. \citet{ghosh2017robust} established that a symmetric loss is robust against symmetric or uniform label noise where a true label is uniformly flipped to another class, i.e., the optimal model that minimizes the risk under the noisy data distribution also minimizes the risk under the true data distribution.  When the optimal risk under the true data distribution is zero, the symmetric loss is also noise tolerant to the non-uniform label noise.  The most well-known symmetric loss is the mean-absolute-error (MAE) loss (aka unhinged loss)~\cite{NIPS2015_45c48cce}. An equivalent (up to a constant) symmetric loss named the reverse-CE (RCE) loss is proposed~\cite{wang2019symmetric}, which reverses the predicted class probabilities and provided one-hot vector encoding the given label information in computing the KL divergence. \citet{ma2020normalized} show that a simple normalization trick can make any loss functions to be symmetric, e.g., normalized CE (NCE) loss.  However, a symmetric loss alone might suffer from slow convergence issue for training~\cite{zhang2018generalized,wang2019symmetric,ma2020normalized}. To address this issue, a  loss that interpolates between the CE loss and a symmetric loss is usually employed. This includes the symmetric-CE (SCE) loss~\cite{wang2019symmetric} that interpolates between the RCE loss and the CE loss, the generalized CE loss (GCE) that interpolates the MAE loss and the CE loss,  Jensen-Shannon divergence (JS) loss that interpolates the MAE loss and the CE loss~\cite{englesson2021generalized}. These losses differ in the way of interpolation. \citet{ma2020normalized} has generalized this  principle to combine an active loss and a passive loss and proposed several variants of robust losses, e.g., NCE + MAE.  Moreover,  the symmetric property has been relaxed to bounded condition enjoyed by the robust logistic loss (RLL)~\cite{patel2021memorization} and mean-square-error (MSE) loss~\cite{ghosh2017robust}, and also relaxed to asymmetric property~\cite{zhou2021asymmetric} enjoyed by a class of asymmetric functions.  The asymmetric loss functions have been proved to be top-1 calibrated and hence top-1 consistent~\cite{zhou2021asymmetric}. However, according to their definition the asymmetric losses cannot be top-$k$ consistent for any $k\geq 2$ because it contradicts to the top-$k$ calibration~\cite{yang2020consistency} - a necessary condition for top-$k$ consistency. We provide a summary for the related robust and traditional multi-class loss functions in Table~\ref{table:loss-formulations}.   

{\bf Distributional robust optimization (DRO) Objective.} It is worth  mentioning the difference from existing literature that leverages the DRO framework  to model the distributional shift between the training data distribution and testing data distribution~\cite{shalev2016minimizing,namkoong2017variance,qi2020practical,levy2020large,duchi2016statistics,duchi2021learning}.  
 These existing works assign a distributional weight to each instance and define an aggregate objective over all instances by taking the worse case over the distributional weights in an uncertainty set that is formed by a constraint function explicitly or a regularization implicitly. In contrast, we study the loss function of an individual instance by using the DRO framework to model the uncertainty in the label information, which models a distributional shift between the true class distribution and the observed class distribution given an instance. 

\section{The LDR Losses}

{\bf Notations.}  Let $(\x, y)\sim\P$ be a random data following underlying true distribution $\P$, where $\x\in\mathcal X\subseteq\R^d$ denotes an input data, $y\in[K]:=\{1, \ldots, K\}$ denotes its class label. Let $f(\x)\in\mathcal C\subset\R^K$ denote the prediction scores by a predictive function $f$ on data $\x$. Let $f_k(\x)$ denote the $k$-th entry of $f(\x)$. Without causing any confusion, we simply use the notation $\f =(f_1, \ldots, f_K) = f(\x)\in\mathcal C$ to denote the prediction for any data $\x$. Interchangeably, we also use one-hot encoding vector $\y\in\R^K$ to denote the label information of a data. For any vector $\q\in\R^K$, let $q_{[k]}$ denote the $k$-largest entry of $\q$ with ties breaking arbitrarily. Let $\Delta_K=\{\p\in\R^K: \sum_{k=1}^Kp_k=1, p_k\geq 0\}$ denote a $K$-dimensional simplex. Let $\delta_{k,y}(f(\x))=f_k(\x) - f_y(\x)$ denote the difference between two coordinates of the prediction scores $f(\x)$. 

\subsection{Label-Distributionally Robust (LDR) Losses}\label{sec:ldr}\vspace*{-0.05in}
A straightforward idea for learning a good prediction function is to enforce that $f_y$ to be larger than any other coordinates of $\f$ with possibly a large margin. However, this approach could mis-guide the learning process due to that the label information $\y$ is often noisy, meaning that the zero entries in $\y$ are not necessarily true zero, the entry equal to one in $\y$ is not necessarily true one. To handle such uncertainty in the label information, we propose to use a regularized DRO framework to capture the inherent uncertainty in the label information. Specifically, the proposed family of LDR losses is defined by:
\begin{align}
\hspace*{-0.1in}\psi_\lambda(\x, y) = \max_{\p\in\Omega}\sum_{k=1}^K p_k (\delta_{k,y}(f(\x)) + c_{k,y})- \lambda R(\p),\label{eqn:ldr}
\end{align}
    where  $\p$ is referred to as Distributional Weight (DW) vector,  $\Omega\subseteq\{\p\in\R^K: p_k\geq 0, \sum_kp_k\leq 1\}$ is a convex constraint set for the DW vector $\p$,   $c_{k,y} = c_y\mathbb I(y\neq k)$ denotes the margin parameter with $c_y\geq 0$, $\lambda>0$ is the DW regularization parameter, and $R(\p)$ is a regularization term of $\p$ which is set by a {\bf strongly convex function}. 
    The strong convexity of $R(\p)$ makes  the LDR losses smooth in terms of $f(\x)$ due to the duality between strongly convexity and smoothness~\cite{Kakade2009OnTD,Nesterov2005SmoothMO}. 
With a LDR loss, the risk minimization problem is formulated as following,   where  $\E$ denotes the expectation:
\begin{align}\label{eqn:psirisk}
\min_{f:\mathcal X\rightarrow \R^K}L_\psi(f):= \E_{(\x, y)\sim\P}[\psi_\lambda(\x, y)].
\end{align}

\vspace*{-0.15in}\paragraph{The LDR-KL Loss.} A special loss in the family of LDR losses is called the LDR-KL loss defined  by specifying  $\Omega=\Delta_K$, and $R(\p)$ as the KL divergence between $\p$ and the uniform probabilities $(1/K, \ldots, 1/K)$,  i.e., $R(\p)= \text{KL}(\p, 1/K) = \sum_{k=1}^Kp_k\log (p_kK)$. In this case, we can derive a closed-form expression of $\psi_\lambda(\x , y)$, denoted by $\psi^{\text{KL}}_\lambda(\x, y)$ (cf. Appendix~\ref{sec:ps}), 
\begin{align}\label{eqn:ldr-kl}
 \psi^{\text{KL}}_\lambda(\x, y) =\lambda \log\bigg[\frac{1}{K}\sum_{k=1}^K \exp\left(\frac{f_k(\x) + c_{k,y} - f_y(\x)}{\lambda}\right)\bigg].
\end{align}
It is notable that the DW regularization parameter $\lambda$ is a key feature of this loss. In many literature, this parameter is called the temperature parameter, e.g.,~\cite{HinVin15Distilling,chen2020simple}.  It is interesting to see that the LDR-KL loss covers several existing losses as special cases or extreme cases by varying the value of $\lambda$. We include the proof for the analytical forms of these cases in the appendix for completeness.

{\bf Extreme Case I.} $\lambda=0$. The loss function becomes: $\psi^{\text{KL}}_0(\x, y)=\max(0, \max_{k\neq y}f_k(\x)+c_{y} - f_y(\x))$
which is known as Crammer-Singer loss (an extension of hinge loss) for multi-class SVM~\cite{10.5555/944790.944813}. 

{\bf Extreme Case II.} $\lambda=\infty$. The loss function becomes: $\psi^{\text{KL}}_\infty(\x, y) = \frac{1}{K}\sum_{k=1}^K (f_k(\x) - f_y(\x)+ c_{k,y})$, which is similar to the Weston-Watkins (WW) loss~\cite{weston1998multi}  given by $\sum_{k\neq y}\max(0, f_k(\x) - f_y(\x) + 1)$. However, WW loss is not top-1 consistent but the LDR-KL loss when $\lambda=\infty$ is top-1 consistent~\cite{10.5555/1005332.1044701}. 

 {\bf A Special Case.} When $\lambda=1$, the LDR-KL loss becomes the CE loss (with a margin parameter): $\psi^{\text{KL}}_1(\x, y) =\log\left(\sum_k \exp\left(f_k(\x) + c_{k,y} - f_y(\x)\right)\right)$.  
It is notable that in standard CE loss, the margin parameter $c_y$ is usually set to be zero. Adding a margin parameter makes the above CE loss enjoys the large-margin benefit. This also recovers the label-distribution-aware margin loss proposed by~\cite{NEURIPS2019_621461af} for imbalanced data classification with different margin values $c_y$ for different classes.

\subsection{Consistency}
\vspace*{-0.05in}
Consistency is an important property of a loss function. In this section, we establish the consistency of LDR losses. A standard consistency is for achieving Bayes optimal top-$1$ error. We will show much stronger consistency for achieving Bayes optimal top-$k$ error for any $k\in[K]$. To this end, we first introduce some definitions and  an existing result of sufficient condition of top-$k$ consistency by top-$k$ calibration~\cite{yang2020consistency}. 

For any vector $\f\in\R^K$, we let $r_k(\f)$ denote a top-$k$ selector that selects the $k$ indices of the largest entries of $\f$ by breaking ties arbitrarily.  
Given a data $(\x, y)$, its top-$k$ error is defined as $\text{err}_k(f, \x, y)=\mathbb{I}(y\not\in r_k(f(\x)))$.
The goal of a classification algorithm under the top-$k$ error metric is to learn a predictor $f$  that minimizes the error: 
$L_{\text{err}_k}(f) = \E_{(\x, y)\sim\P}[\text{err}_k(f, \x, y)]$. 
 The predictor $f^*(\cdot):\mathcal X\rightarrow\R^K$ is top-$k$ Bayes optimal if
    $L_{\text{err}_k}(f^*) =  L^*_{\text{err}_k}:=\inf_{f\in\mathcal F}L_{\text{err}_k}(f)$. 
Since we can not directly optimize the top-$k$ error, we usually optimize a risk function using a surrogate loss, e.g.,~(\ref{eqn:psirisk}).  A central question of consistency is that whether optimizing the risk function can optimize the top-$k$ error. 
 \begin{definition}
For a fixed $k\in[K]$, a loss function $\psi$ is top-$k$ consistent if for any sequence of measurable functions $f^{(n)}:\mathcal X\rightarrow\R^K$, we have
\begin{align*}
L_\psi(f^{(n)})\rightarrow L^*_\psi\Longrightarrow L_{\text{err}_k}(f^{(n)}) \rightarrow L_{\text{err}_k}^*
\end{align*}
where $L_\psi^*= \min_{f}L_\psi(f)$. If the above holds for all $k\in[K]$, it is referred to as All-$k$ consistency. 
\end{definition}
\vspace*{-0.1in}
According to~\cite{yang2020consistency}, a necessary condition for top-$k$ consistency of a multi-class loss function is that the loss function is top-$k$ calibrated defined below. 
Since top-$k$ calibration characterizes when minimizing $\psi$ for a fixed $\x$ leads to the Bayes decision for that $\x$, we will consider $\f= f(\x)$ for any fixed $\x\in\X$ and let $\q=\q(\x) = (\Pr(y=1|\x), \ldots, \Pr(y=K|\x))$ be the underlying conditional label distribution. 
Define $L_\psi(\f, \q)= \sum_lq_l\psi(\x, l)$.  Let $P_k(\f, \q)$ denote that $\f$ is top-$k$ preserving with respect to the underlying label distribution $\q$, i.e.,  if for all $l\in[K]$, $q_l> q_{[k+1]}\Longrightarrow f_l> f_{[k+1]}$, and $q_l<q_{[k]}\Longrightarrow f_l<f_{[k]}$.

\begin{definition}
For a fixed $k\in[K]$, a loss function $\psi$ is called top-$k$ calibrated if for all $\q\in\Delta_K$  it holds that:~~
$\inf_{\f\in\R^K: \neg P_k(\f, \q)}L_\psi(\f, \q)> \inf_{\f\in\R^K}L_\psi(\f, \q)$. A loss function is called All-$k$ calibrated  if the loss function $\psi$ is  top-$k$ calibrated for all $k\in[K]$. 
\end{definition}
Unlike~\cite{yang2020consistency} that relies on top-$k$ preserving property of optimal predictions, we will use rank preserving property of optimal predictions to prove All-$k$ calibration and consistency.  $\f$ is called \textbf{rank preserving} w.r.t $\q$, i.e., if for any pair $q_i<q_j$ it holds that $f_i<f_j$.  Our lemma below shows that if the optimal predictions for minimizing $L_\psi$ is rank preserving, then $\psi$ is All-$k$ calibrated. All missing proofs are included in appendix. 
\begin{lemma}\label{lem:rank-cali}
For any $\q\in\Delta_K$, if $\f^*=\arg\min_{\f\in\R^K}L_\psi(\f, \q)$ is rank preserving with respect to $\q$, then $\psi$ is All-$k$ calibrated. 
\end{lemma}
\vspace*{-0.05in}
Our first main result of consistency is the following. 
\begin{thm}\label{thm:1}
If $c_y = c$, $\forall y\in[K]$, then   $\f^*=\arg\min_{\f\in\R^K}L_{\psi_{\lambda}^{\text{KL}}}(\f, \q)$ is rank preserving for any $\lambda\in(0,\infty)$ and hence  the LDR-KL loss $\psi_{\lambda}^{\text{KL}}$ with any $\lambda\in(0,\infty)$ is All-$k$ calibrated. Therefore, it is All-$k$ consistent.
\end{thm}

\vspace*{-0.1in}
{\bf Calibration of the general LDR losses:} we present a sufficient condition on the DW regularization function $R$ such that the resulting LDR loss $\psi_\lambda$ is All-$k$ calibrated. The result below will address an open problem in~\cite{DBLP:journals/pami/LapinHS18} regarding the consistency of smoothed top-$k$ SVM loss. 
\begin{definition}
A function $R(\p)$ is input-symmetric  if its value is the same, no matter the order of its arguments. A set $\Omega$ is symmetric, if a point $\p$ is in the set
then so is any point obtained by interchanging any two coordinates of $\p$.
\end{definition}
\begin{thm}\label{thm:2}
If $c_y = c, \forall y\in[K]$, $R$ and $\Omega$ are (input-) symmetric, $R$ is strongly convex satisfying $\partial R(0)<0$ and $[\partial R(p)]_i>0\Longrightarrow p_i\neq 0$, then the family of LDR losses $\psi_\lambda$ with any  $\lambda\in(0,\infty)$ is All-$k$ calibrated.
\end{thm}
\vspace*{-0.1in}{\bf Remark: } Examples of $R$ that satisfy the above conditions include $R(\p) = \sum_ip_i \log(Kp_i)$  and $R(\p) = \|\p - 1/K\|^2$. 

{\bf Consistency of smoothed top-$k$ SVM losses}. As an application of our result above, we restrict our attention to $\Omega=\Omega(k)=\{\p\in\R^K: \sum_{l=1}^Kp_l\leq 1, p_l\leq 1/k, \forall l\}$. When $\lambda=0$, the LDR loss becomes the top-$k$ SVM loss~\cite{DBLP:journals/pami/LapinHS18}, i.e.,  
\begin{align*}
\psi^k_{\lambda=0}(\x, y) & = \max_{\p\in\Omega(k)}\sum_l p_l (\delta_{l,y}(f(\x)) + c_{l,y}) \\
& = \frac{1}{k}\sum_{i=1}^k\max(0, \f - f_y + \c_y)_{[i]},
\end{align*}
where $\c_y=(c_{1,y},\ldots, c_{K,y})$. 
This top-$k$ SVM loss is not consistent acoording to~\cite{DBLP:journals/pami/LapinHS18}. But, we can make  it consistent by adding a strongly convex regularizer $R(\p)$ satisfying the conditions in Theorem~\ref{thm:2}. 
This addresses the open problem raised in~\cite{DBLP:journals/pami/LapinHS18} regarding the consistency of smoothed top-$k$ SVM loss for multi-class classification. 

\vspace*{-0.05in}\subsection{Robustness}
\vspace*{-0.02in}
\begin{definition}
A loss function $\psi$ is symmetric if:~ $\sum_{j=1}^K\psi\big(\x, j\big)$ is a constant for any $f,\x$.  
\end{definition}
 \vspace*{-0.1in} A loss with the symmetric property is noise-tolerant against uniform noise and class dependent noise under certain conditions~\cite{ghosh2017robust}. For example, under uniform noise if the probability of noise label $\hat y$ not equal to true label $y$ is not significantly large,  i.e., $\Pr(\hat y\neq y|y)< 1 - \frac{1}{K}$, then the optimal predictor for optimizing $\E_{\x,\hat y}[\psi(\x, \hat y)]$ is the same as optimizing the true risk $\E_{\x, y}[\psi(\x, y)]$. 
 
Next, we present a negative result showing that a non-negative symmetric loss cannot be All-$k$ consistent. 
\begin{thm}\label{thm:4}A non-negative symmetric loss function cannot be top-$1,2$ consistent simultaneously when  $K\ge3$. 
\end{thm}
\vspace*{-0.1in}{\bf Remark: } The above theorem indicates that it is impossible to design a non-negative loss that is not only symmetric but also All-$k$ consistent. Exemplar non-negative symmetric losses include MAE (or SCE or unhinged loss) and NCE, which are top-1 consistent but not All-$k$ consistent.   

Next, we show a way to break the impossibility by using a symmetric loss that is not necessarily non-negative, e.g.,  the LDR-KL loss with $\lambda=+\infty$. 
\begin{thm}\label{thm:3}
When $\lambda=\infty$, the LDR-KL loss enjoys symmetric property,  and is All-$k$ consistent when $\mathcal C=\{\f\in\R^K, \|\f\|_2\leq B\}$.
\end{thm}
\vspace*{-0.1in}{\bf Remark: } 
Theorem~\ref{thm:4} is proved for general case, and Theorem~\ref{thm:3} only holds with the extra constrain.


\vspace*{-0.1in}
\subsection{Adaptive LDR-KL (ALDR-KL) Loss}\vspace*{-0.05in}
Although LDR-KL with $\lambda=\infty$ is both robust and consistent, it may suffer from the slow convergence issue similar to the MAE loss~\cite{zhang2018generalized}.  We design an Adaptive LDR-KL (ALDR-KL) loss for the data with label noises, which can automatically adjust DW regularization parameter $\lambda$ for individual data. The motivation is: i) for those noisy data, the model is less confident and we want to make the $\lambda$ to be larger, so that loss function is more robust; ii) for those clean data, the model is more confident and we want to make the $\lambda$ to be smaller to enjoy the large margin property of the CS loss. 

To this end, we propose the following ALDR-KL loss: 
\begin{align}\label{eqn:aldr-kl}
\psi_{\alpha, \lambda_0}^{\text{KL}}(\x, y) =& \max_{\lambda\in\R_+}\max_{\p\in\Delta_K}\sum_{k=1}^K p_k (\delta_{k,y}(f(\x)) + c_{k, y})\nonumber\\
&- \lambda \text{KL}(\p, \frac{1}{K}) - \frac{\alpha}{2} (\lambda-\lambda_0)^2.
\end{align}
where $\alpha>0$ is a hyper-parameter, and $\lambda_0$  is a relatively large value.  
To intuitively understand this loss, we consider when the given label is clean and noisy. If the given label is clean, then it is better to push the largest $f_k(\x), k\neq y$ to be small, hence a spiked $\p$ would be better, as a result $\text{KL}(\p, 1/K)$ would be large;  so by maximizing $\lambda$ it will push $\lambda$ to be smaller. On the other hand, if the given label is noisy, then it is better to use uniform $\p$ (otherwise it may wrongly penalize more on predictions for the ``good" class labels), as a result $\text{KL}(\p, 1/K)$ would be small, then by maximizing over $\lambda$  it will push $\lambda$ to be close to the prior $\lambda_0$. 

To elaborate this intuition, we conduct an experiment on a simple synthetic data. We generate data for three classes from  four normal distributions with means at (0.8,0.8), (0.8,-0.8), (-0.8,0.8), (-0.8,-0.8) and standard deviations equal to (0.3, 0.3) as shown in Figure~\ref{fig:syn} (left). To demonstrate the adpative $\lambda$ mechanism of ALDR-KL, we first pretrain the model by optimizing the CE loss by 1000 epochs with momentum optimizer and learning rate as 0.01. Then, we add an extra hard but clean example from the $\circ$ class 
(the blue point in the dashed circle),  and a noisy data 
from the $\circ$ class but mislabeled as the $\diamond$ class (the red point in the dashed circle), and then finetune the model with extra 100 epochs by optimizing our ALDR-KL loss with Algorithm~\ref{alg:1} and optimizing CE loss with momentum SGD.  
We set the $\lambda_0=10,\alpha=0.05$ for ALDR-KL loss. The learned model of using ALDR-KL loss (right) is more robust than that of using CE loss (middle). We also show the corresponding learned $\lambda_T$ values at the last iteration for the two added examples, which show that the noisy data has a large $\lambda=8.4$ and the clean data has a small $\lambda=0.001$.    We also plot the averaged $\lambda_t$ and the KL-divergence values $\text{KL}(\p_t, 1/K)$ during the training process in Figure~\ref{fig:lamswithkls} of Appendix, which clearly shows the negative relationship between them. {Additionally, we provide more synthetic experiments for LDR-KL loss by training from scratch in Appendix~\ref{sec:more-synthetic}, where the observations are consistent with the conclusion here.} 

The theoretical properties of ALDR-KL are stated below. 
\begin{thm}\label{thm:5}The ALDR-KL loss is All-$k$ consistent when $ c_y=0,\forall y, \lambda_0\in(0,\infty), \alpha>\frac{\log K}{\lambda_0}$ or $\lambda_0=\infty, \mathcal C=\{\f\in\R^K, \|\f\|_2\leq B\}$,  and is Symmetric when  $\lambda_0=\infty$. 
\end{thm}
\vspace*{-0.05in}\begin{algorithm}[t]
\caption{Stochastic Optimization for ALDR-KL loss}\label{alg:1}
\begin{algorithmic}[1]
\REQUIRE $\eta,\alpha, \beta, \lambda_0$
\STATE Random initialize model parameters $\w_0$; initialize $\lambda^i_0$ as $\lambda_0$ for each data sample, $i\in[n]$; initialize $\mathbf m_0=0$.
\FOR{$t=1,...T$}
\STATE Sample mini-batch of data indices $\B_t$
\FOR{each sample $i\in\B_t$}
\STATE Compute $\p^i_{t-1}$ by E.q.~\ref{eqn:estimate-p} with $\lambda_{t-1}^i$, $f(\x_i, \w_{t-1})$.

\STATE Update $\lambda_{t}^i=[\lambda_0 - \frac{1}{\alpha}\text{KL}(\p^i_{t-1}, \frac{1}{K})]_+$


\STATE Compute $G(\lambda^i_{t},\w_{t-1},\x_i)$ by E.q.~\ref{eqn:stochastic-w}

\ENDFOR

\STATE Compute mini-batch gradient estimation by:\\
$G(\w_{t-1})=\frac{1}{|\B_t|}\sum_{i\in\B_t}G(\lambda^i_{t},\w_{t-1},\x_i)$
\STATE Compute $\mathbf m_{t} = \beta\mathbf m_{t-1} + (1-\beta)G(\w_{t-1})$
\STATE Update $\w_{t}=\w_{t-1} - \eta \mathbf m_{t}$
\ENDFOR
\end{algorithmic}
\end{algorithm}
\begin{figure}
    \centering
    \hspace*{-0.1in}\includegraphics[scale=0.22]{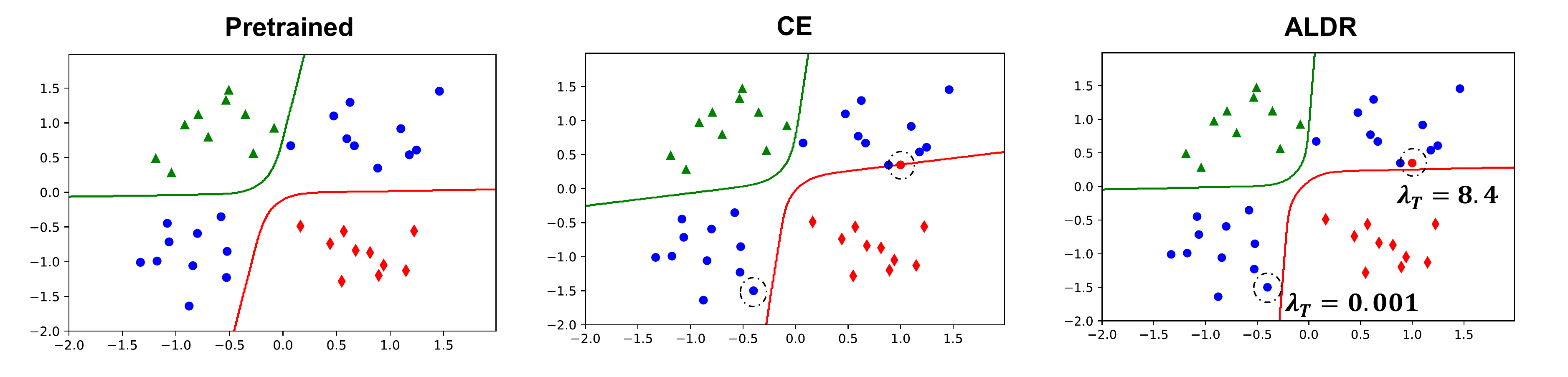}
     \vspace*{-0.2in}\caption{Left: pretrained model on a synthetic data; Middle: a new model learned by optimizing the CE loss after adding two examples in dashed circle; Right:  a new model learned by optimizing the ALDR-KL loss after adding two examples in dashed circle. Shape denotes the true label, and color denotes the given label.   
     }\label{fig:syn}
\end{figure}

{\bf Optimization for ALDR-KL loss:} Using ALDR-KL loss and with a finite set of examples $\{(\x_i, y_i)\}_{i=1}^n$, the empirical risk minimization problem becomes: 
\begin{align}
\min_{\w}\frac{1}{n}\sum\nolimits_{i=1}^n \psi_{\alpha, \lambda_0}^{\text{KL}}(\x_i, y_i),
\end{align}
where $\w$ denotes the parameter of the prediction function $f(\cdot; \w)$.  Below, we will present how to solve this problem without causing significant more burden  than optimizing traditional CE loss.  Recall that the formulation for LDR-KL in (\ref{eqn:ldr-kl}), it is easy to show that:
\begin{align}
&\psi_{\alpha, \lambda_0}^{\text{KL}}(\x, y) =\max_{\lambda\in\R_+}- \frac{\alpha}{2} (\lambda-\lambda_0)^2\nonumber\\ +&\lambda \log\left(\sum_{k=1}^K \exp\left(\frac{f_k(\x) + c_{k,y} - f_y(\x)}{\lambda}\right)\right)\nonumber,
\end{align}
where we have used the  closed-form solution for $\p$ given $\lambda$: \begin{align}\label{eqn:estimate-p}
p^*_k(\lambda)\propto \exp\left(\frac{f_k(\x; \w) + c_{k,y} - f_y(\x; \w)}{\lambda}\right), \forall k.
\end{align}
To compute the $\lambda$,  first by optimality condition ignoring the non-negative constraint we have $-\text{KL}(\p^*, \frac{1}{K})-\alpha(\lambda-\lambda_0)=0.$ This motivate us to update $\lambda$ alternatively with $\p$. Given $\lambda^i_{t-1}$ for an individual data $\x_i$ at the $(t-1)$-th iteration, we compute $\p^i_{t-1}$ according to~(\ref{eqn:estimate-p}) then we  update $\lambda^i_{t}$ according to the above equation with $\p^*$ replaced by $\p^i_{t-1}$ and a projection onto non-negative orthant $\R_+$: 
\begin{align}\label{eqn:update-lambda}\lambda^i_t = \left[\lambda_0 - \frac{1}{\alpha}\text{KL}(\p^i_{t-1}, \frac{1}{K})\right]_+.\end{align}
For updating the model parameter $\w$, at iteration $t$ suppose we have the current estimation for $\lambda_{t}$, we can compute the stochastic gradient estimation for $\w_{t-1}$ as: \begin{align}\label{eqn:stochastic-w}
    &G(\lambda^i_{t},\w_{t-1}, \x_i) =\\ &\nabla_{\w}\lambda^i_{t} \log \sum_{k=1}^K \exp \left( \frac{f_k(\x_i,\w_{t-1}) + c_{k,y} - f_y(\x_i,\w_{t-1})}{\lambda^i_{t}}\right)  \nonumber
\end{align}
Finally, we provide the algorithm using the momentum update for optimizing ALDR-KL loss in Alg.~\ref{alg:1}. It is notable that other updates (e.g., SGD, Adam) can be used for updating $\w_t$ as well. It is worth noting that we could also use gradient method to optimize $\lambda$ and $\p$, but it will involve tuning two extra learning rates. Empirically we verify that Alg.~\ref{alg:1} converges well for optimizing the ALDR-KL loss, which is presented in Appendix~\ref{sec:exP} for interested readers {and the convergence analysis is left as a future work.}


\section{Experiments}

 In this section, we present the datasets, experimental settings, experimental results, and ablation study regarding the adaptiveness in terms of $\lambda$ for ALDR-KL loss. For all the experiments unless specified otherwise, we manually add  label noises to the training and validation data, but keep the testing data clean. A similar setup was used in~\cite{zhang2018generalized}, which  simulates the noisy datasets in practice. We compare with all loss functions in Table~\ref{table:loss-formulations} except for NCE, which is replaced by NCE+RCE following~\cite{ma2020normalized} and NCE+AGCE, NCE+AUL following~\cite{zhou2021asymmetric}. Detailed expressions of these losses with hyper-parameters are shown in~Table~\ref{table:loss-formulations}. 
We do not compare the Generalized JS Loss (GJS) because it relies on extra unsupervised data augmentation and need more computational cost~\cite{englesson2021generalized}. For all experiments, we utilize the identical optimization procedures and  models, only with different losses.

 \vspace*{-0.1in}
\subsection{Benchmark Datasets with Synthetic Noise}\vspace*{-0.05in}
We conduct experiments on 7 benchmark datasets, namely ALOI, News20, Letter, Vowel~\cite{libsvm}, Kuzushiji-49, CIFAR-100 and Tiny-ImageNet~\cite{clanuwat2018deep,deng2009imagenet}. The number of classes vary from 11 to 1000. 
The statistics of the datasets are summarized in Table~\ref{table:data} in the Appendix. We manually add different noises to the datasets. Similar to previous works~\cite{zhang2018generalized,wang2019symmetric}, we consider two different noise settings, uniform (symmetric) noise and class dependent (asymmetric) noise. For uniform noise, we impose a uniform random noise with the probability $\xi\in \{0.3, 0.6, 0.9\}$ such that a data label is changed to any label in $\{1,...,K\}$ with a probability $\xi$. We would like to point out that the uniform noise rate $0.9$ should be tolerable according the theory of  a symmetric loss, which can tolerate uniform noise rate up to $(1-1/K)>0.9$, where $K$ is larger than 10 for the considered datasets.  For class dependent noise, with the probability in $\xi\in\{0.1, 0.3, 0.5\}$, a data label could be changed to another label by following pre-defined transition rules. For letter dataset: $B\leftrightarrow D$, $C\leftrightarrow G$, $E\leftrightarrow F$, $H\leftrightarrow N$, $I\leftrightarrow L$, $K\leftrightarrow X$, $M\leftrightarrow W$, $O\leftrightarrow Q$, $P\leftrightarrow R$, $U\leftrightarrow V$. For news20 dataset: \textit{comp.os.ms-windows.misc} $\leftrightarrow$ \textit{comp.windows.x}, \textit{comp.sys.ibm.pc.hardware} $\leftrightarrow$ \textit{comp.sys.mac.hardware}, \textit{rec.autos} $\leftrightarrow$ \textit{rec.motorcycles}, \textit{rec.sport.baseball} $\leftrightarrow$ \textit{rec.sport.hockey}, \textit{sci.crypt} $\leftrightarrow$ \textit{sci.electronics}, \textit{soc.religion.christian} $\leftrightarrow$ \textit{talk.politics.misc}. For vowel dataset: $i\leftrightarrow I$, $E\leftrightarrow A$, $a:\leftrightarrow Y$, $C:\leftrightarrow O$, $u:\leftrightarrow U$. For ALOI, Kuzushiji-49, CIFAR-100 and Tiny-ImageNet datasets, we simulate class-dependent noise by flipping each class into the next one circularly with probability $\xi$.

For all datasets except for Kuzushiji-49, CIFAR-100 and Tiny-ImageNet, we uniformly randomly split 10\% of the whole data as testing data, and the remaining as training set. For Kuzushiji-49 and CIFAR-100 dataset, we use its original train/test splitting. For Tiny-ImageNet, we use its original train/validation splitting and treat the validation data as the testing data in this work. For News20 dataset, because the original dimension is too large, we apply PCA  to reduce the dimension to 80\% variance level.

\vspace*{-0.1in}
\subsection{Experimental Settings}\vspace*{-0.05in}
Since ALOI, News20, Letter and Vowel datasets are tabular data, we use a 2-layer feed forward neural network, with a number of neurons in the hidden layer same as the number of features or the number of classes (whichever is smaller) as the backbone model for these four datasets. ResNet18 is utilized as the backbone model for Kuzushiji-49, CIFAR100 and Tiny-ImageNet image datasets~\cite{he2016deep}. We apply 5-fold-cross-validation to conduct the training and evaluation, and report the mean and standard deviation for the testing top-$k$ accuracy, where $k\in\{1,2,3,4,5\}$. We use early stopping according to the highest accuracy on the  validation data. For Kuzushiji-49 dataset, the top-$k$ accuracy is first computed in class level, and then averaged across the classes according to~\cite{clanuwat2018deep} because the classes are imbalanced. 

We fix the weight decay as 5e-3, batch size as 64, and total running epochs as 100 for all the datasets except Kuzushiji, CIFAR100 and Tiny-ImageNet (we run 30 epochs for them because the data sizes are large). We utilize the momentum optimizer with the initial learning rate tuned in \{1e-1, 1e-2, 1e-3\} for all experiments. 
We decrease the learning rate by ten-fold at the end of the 50th and 75th epoch for 100-total-epoch experiments, and at the end of 10th and 20th epoch for 30-total-epoch experiments. 
For CE, MSE and MAE loss, there is no extra parameters need to be tuned; for CS and WW loss, we tune the margin parameter at \{0.1, 1, 10\}; for RLL loss, we tune the $\alpha$ parameter at \{0.1, 1, 10\}; for GCE and TGCE loss, we tune the power parameter $q$ at \{0.05, 0.7, 0.95\} and set truncate parameter as 0.5 for TGCE; for SCE, as suggested in the original paper, we fix A=-4 and tune the simplified balance parameter between CE and RCE at \{0.05, 0.5, 0.95\}; for JS loss, we tune the balance parameter at \{0.1, 0.5, 0.9\}; for NCE+(RCE,AGCE,AUL) loss, we keep the default parameters for each individual loss, and tune the two weight parameters $\alpha/\beta$ in \{0.1/9.9, 5/5, 9.9/0.1\} following the original papers. 
 For LDR-KL and ALDR-KL loss, we tune the DW regularization parameter $\lambda$ or the $\lambda_0$ at $\{0.1, 1, 10\}$.  We normalize model logits $\f(\x;\w)$ by $\frac{\|\f(\x; \w)\|_1}{K}$ so that different $\lambda$ values can achieve the different intended effects for LDR-KL and ALDR-KL loss. The margin parameter is fixed as 0.1 for all experiments. For ALDR-KL, we set the $\alpha=\frac{2\log K}{\lambda_0}$, so that $\lambda_{t}^i=\lambda_0 - \frac{1}{\alpha}\text{KL}(\p_{t-1}(\x_i)\|\frac{1}{K})\in[\lambda_0/2, \lambda_0]$, which maintains an appropriate adaptive adjustment level from $\lambda_0$.

\subsection{Leaderboard on seven benchmark data}
\vspace*{-0.05in}We compare 15 losses on 7 benchmark datasets over 6 noisy and 1 clean settings for 5 metrics with 3,675 numbers. However, due to limit of space it is difficult to include the results here. Instead, we focus on the overall results on all datasets in all settings by presenting a leaderboard in Table~\ref{tab:ranks}. For each data in each setting, we rank different losses from $1$ to $15$ according to their performance from high to low (the smaller the rank value meaning the better of the performance). 
We average the ranks across all the benchmark datasets in all settings for top-$k$ accuracy with different $k$ and we also provide an overall rank.

We present the full results and their analysis for different performance metrics in the Appendix~\ref{sec:exP}. From the leaderboard in Table~\ref{tab:ranks}, we observe that the ALDR-KL loss is the strongest and followed by LDR-KL loss. Besides, the CE loss and two variants (SCE, GCE) are more competitive than the other loss functions. The symmetric MAE loss is non-surprisingly the worst  which implies consistency is also an important factor for learning with noisy data. 

\begin{table}[t]
    \centering
        \caption{leaderboard for comparing 15 different loss functions on 7 datasets in 6 noise and 1 clean settings. The reported numbers are averaged ranks of performance. The smaller the better.  }
    \label{tab:ranks}
    \resizebox{.48\textwidth}{!}{
    \begin{tabular}{c|ccccc|c}
    \toprule
       Loss Function & top-1 & top-2& top-3 & top-4& top-5  &  overall \\
       \hline

    ALDR-KL & \textbf{2.163} & \textbf{2.449} & \textbf{2.184} & \textbf{2.224} & \textbf{2.633} & \textbf{2.331} \\
    LDR-KL & 2.429 & 2.571 & 2.816 & 2.939 & 2.959 & 2.743 \\
    CE    & 4.714 & 4.592 & 4.224 & 4.449 & 4.388 & 4.473 \\
    SCE   & 4.898 & 4.776 & 4.429 & 4.837 & 4.327 & 4.653 \\
    GCE   & 5.816 & 5.347 & 6.041 & 5.571 & 5.551 & 5.665 \\
    TGCE  & 6.204 & 6.224 & 6.265 & 6.224 & 6.306 & 6.245 \\
    WW    & 7.673 & 6.592 & 6.082 & 5.959 & 5.551 & 6.371 \\
    JS    & 7.816 & 7.837 & 7.857 & 7.98  & 8.245 & 7.947 \\
    CS    & 8.673 & 8.878 & 8.816 & 8.592 & 8.735 & 8.739 \\
    RLL   & 10.224 & 10.204 & 10.224 & 10.49 & 10.49 & 10.326 \\
    NCE+RCE  & 9.694 & 10.633 & 10.878 & 10.612 & 10.714 & 10.506 \\
    NCE+AUL  & 10.449 & 10.959 & 11.102 & 11.224 & 11.122 & 10.971 \\
    NCE+AGCE  & 11.714 & 11.837 & 11.837 & 11.51 & 11.714 & 11.722 \\
    MSE   & 12.796 & 12.653 & 12.776 & 12.959 & 12.898 & 12.816 \\
    MAE   & 14.735 & 14.449 & 14.469 & 14.429 & 14.367 & 14.49 \\

        \bottomrule
    \end{tabular}}
    \vspace*{-0.14in}
    \centering
        \caption{Top-1 Accuracy on  ILSVRC12 validation data with models trained on mini-webvision.}
    \label{tab:webvision}
    \resizebox{.48\textwidth}{!}{
    \begin{tabular}{c|cccccccc}
    \toprule
       Loss & NCE+RCE & SCE& GCE& JS&CE& AGCE &  LDR-KL &ALDR-KL \\
       \hline
       ACC &  60.24 & 62.44 & 62.76  &  65.00&  66.40  & 67.52  & 69.00 & \textbf{69.92}   \\
        \bottomrule
    \end{tabular}}
    \vspace*{0.05in}
\end{table}

\begin{figure}[t]
\vspace{-0.1 in}
    \centering
    \subfigure[$\lambda_0=1$, CD(0.1)]{\includegraphics[scale=0.2]{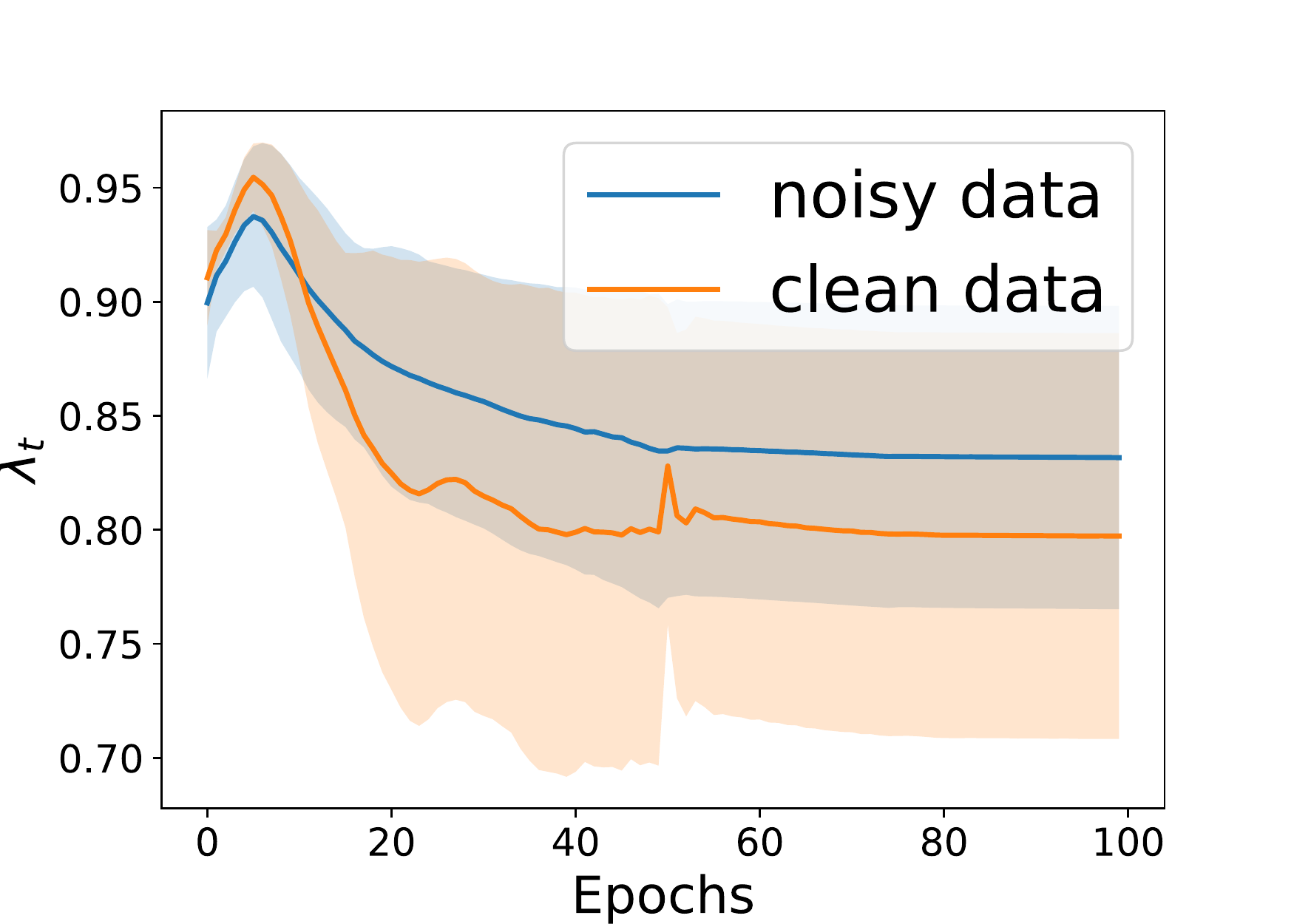}}
    \subfigure[$\lambda_0=10$, CD(0.1)]{\includegraphics[scale=0.2]{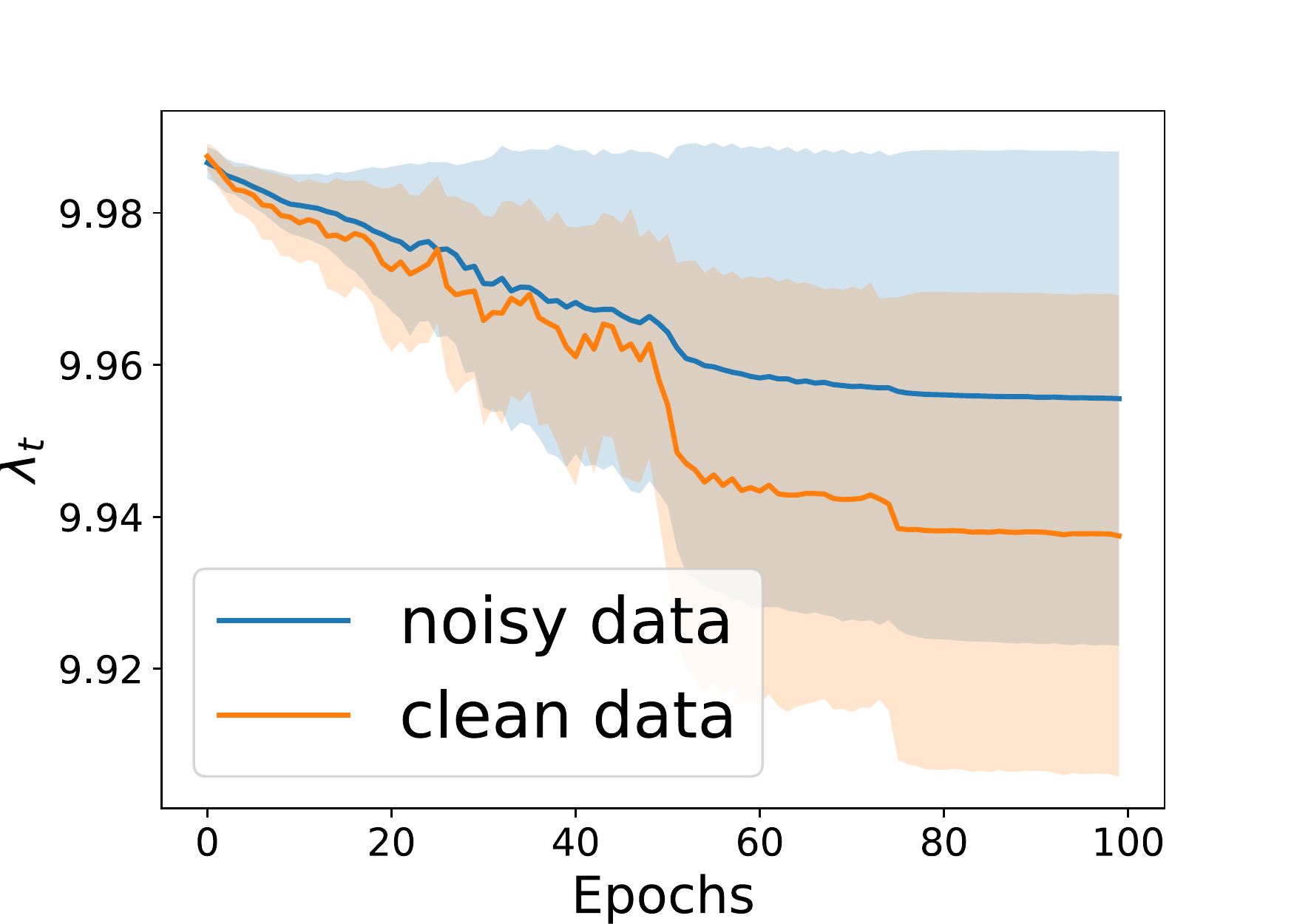}}
    \vspace*{-0.1in}
    \subfigure[$\lambda_0=1$, U(0.3)]{\includegraphics[scale=0.2]{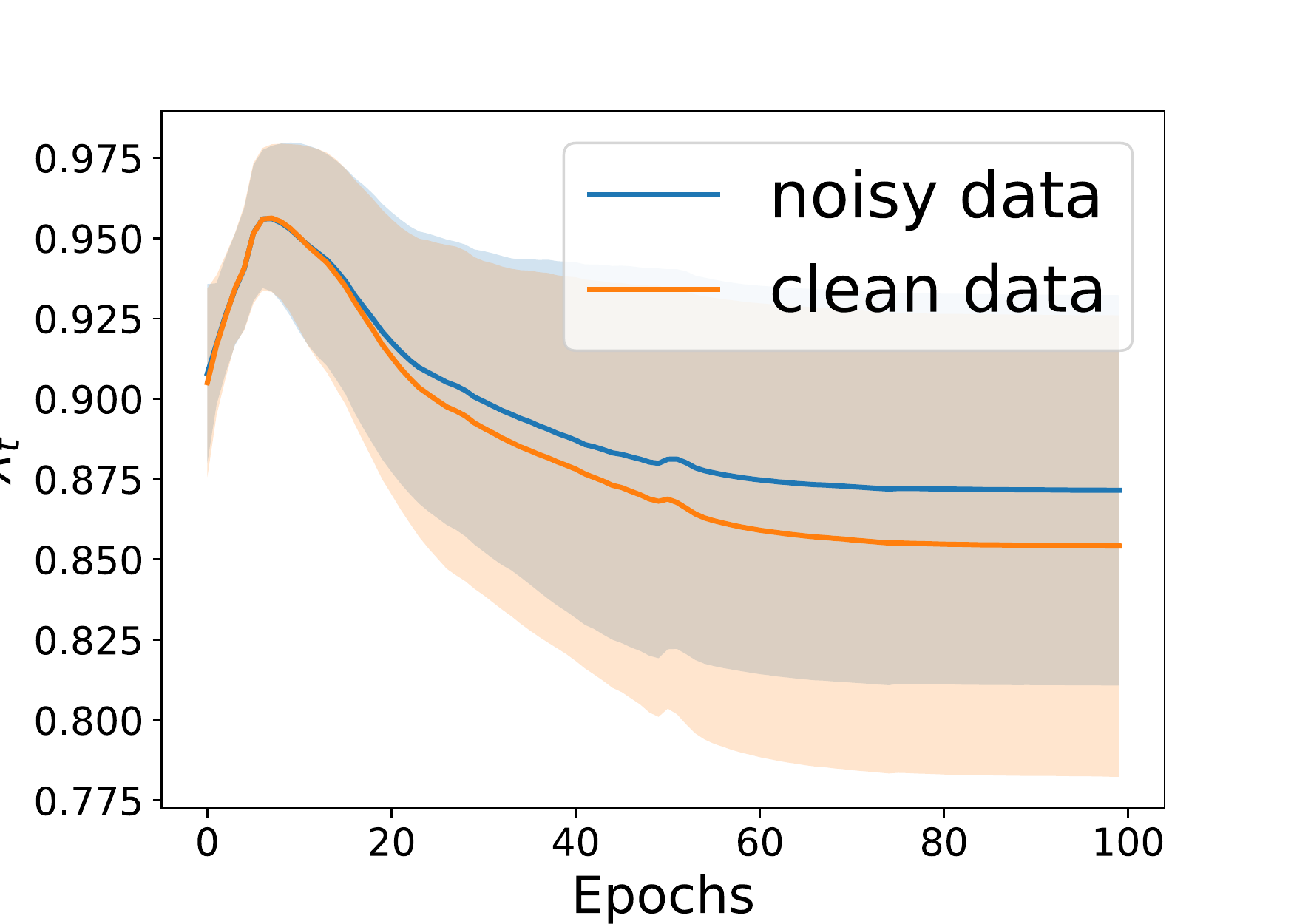}}
    \subfigure[$\lambda_0=10$, U(0.3)]{\includegraphics[scale=0.2]{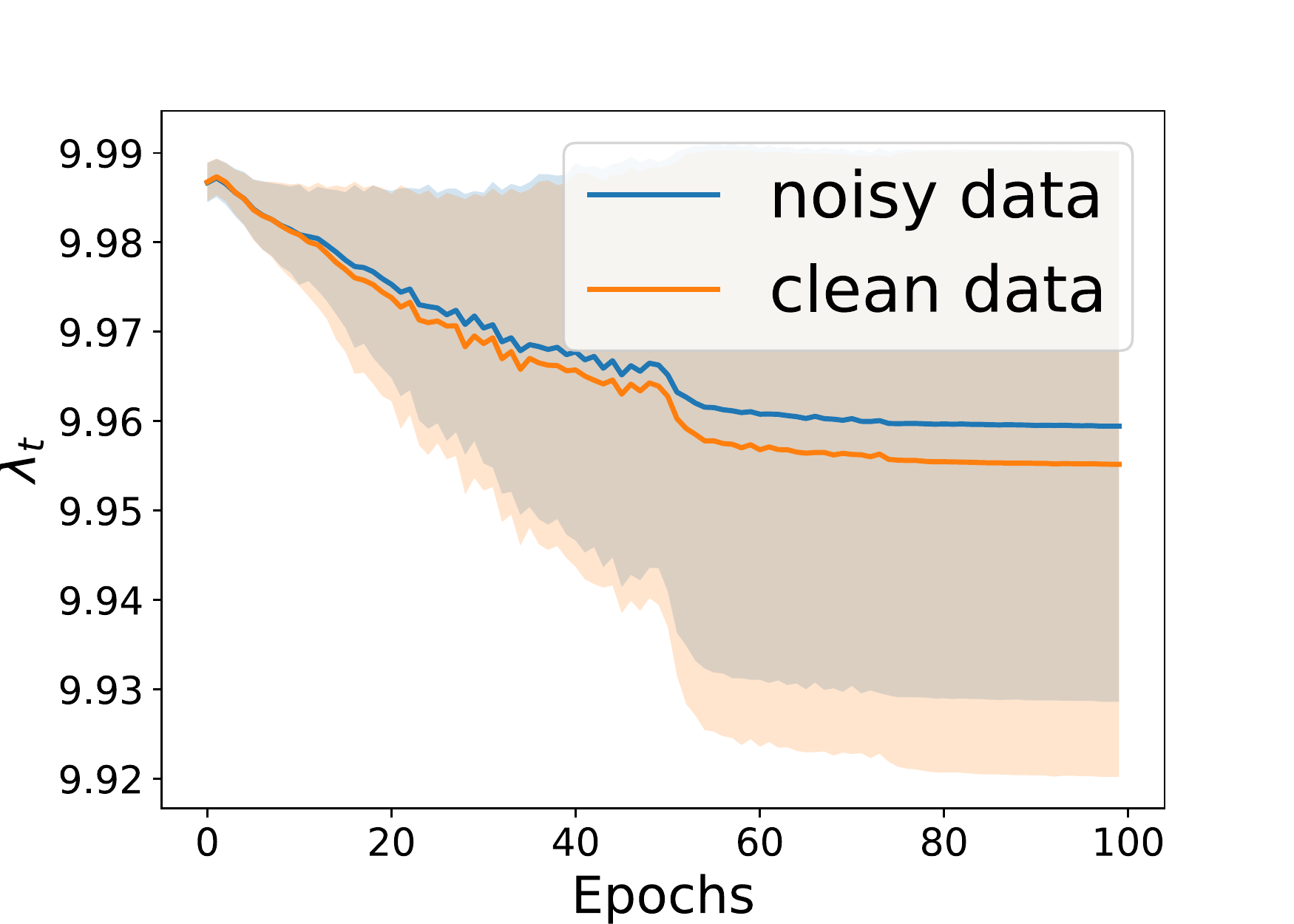}}
        \vspace{-0.05in}\caption{ Averaged $\lambda_t$ values for noisy samples and clean samples with error bar  for ALDR-KL loss on Vowel dataset, `U' is short for uniform noise, `CD' is short for class-dependent noise.} \label{fig:adaptive}
\end{figure}

\subsection{Real-world Noisy Data}\vspace*{-0.05in}
We evaluate the proposed methods on a real-world noisy dataset, namely webvision~\cite{DBLP:journals/corr/abs-1708-02862}. It contains more than 2.4 million of web images crawled from the Internet by using queries generated from the same 1,000 semantic concepts as the benchmark ILSVRC 2012 dataset. We mainly follow the settings on the previously work~\cite{zhou2021asymmetric,ma2020normalized,englesson2021generalized} by using a mini version, i.e., taking the first 50 classes from webvision as training data, and evaluating the results on ILSVRC12 validation data.  More details about the experimental settings are summarized in the Appendix~\ref{sec:webvision-setup}.  ResNet50 is utilized as the backbone model~\cite{he2016deep}. The results  summarized in Table~\ref{tab:webvision} show ALDR-KL performs the best followed by LDR-KL. We notice that~\cite{englesson2021generalized} reports a higher result for JS than that in Table~\ref{tab:webvision}, which is due to different setups, especially with more epochs, larger weight decay and tuning of learning rates.  Following the same experimental setup as~\cite{englesson2021generalized}, our ALDR-KL and LDR-KL achieve 70.80 and 70.64 for top-1 accuracy, respectively, while JS's is reported as 70.36. 


\vspace*{-0.05in}
\subsection{Adaptive $\lambda$ for ALDR-KL loss}\vspace*{-0.05in}
Lastly, we conduct experiments to illustrate the adaptive $\lambda_t^i$ values learned by using ALDR-KL loss on the Vowel dataset with uniform noise ($\xi=0.3$) and class dependent noise ($\xi=0.1$). We maintain the adaptive $\lambda_t^i$ values for noisy data samples and clean data samples separately at each epoch. Then we output the mean and standard deviation for the adaptive $\lambda_t^i$ values for noisy data and clean data. $\lambda_0\in\{1,10\}$ are adopted for ALDR-KL at this subsection. All the settings are identical with previous experimental setups except initial learning rate is fixed as $1e-2$ for this specific justification experiment. The results are presented in Figure~\ref{fig:adaptive}. From the figures, we observe that ALDR-KL applies larger DW regularization parameter $\lambda$ for the noisy data, which makes ALDR-KL loss closer to symmetric loss function, and therefore it is more robust to certain types of noises (e.g. uniform noise, simple non-uniform noise and class dependent noise)~\cite{ghosh2017robust}. Besides, the adaptive $\lambda_t^i$ values are generally decreasing for both noisy and clean data because the model is more certain about its predictions with more training steps.


\subsection{Running Time}
 We provide the running time analysis for the LDR-KL and ALDR-KL loss in Table~\ref{tab:running-time} in Appendix~\ref{sec:running-time}. We observe that the running time of optimizing ALDR-KL is slightly higher than that of optimizing the CE loss, it is still comparable with optimizing other robust losses such as TGCE, SCE, RLL, etc. 

\vspace*{-0.05in}
\section{Conclusions}
\vspace*{-0.05in}We proposed a novel family of Label Distributional Robust (LDR) losses. We studied consistency and robustness of proposed losses. We also proposed an adaptive LDR loss to adapt to the noisy degree of individual data. Extensive experiments demonstrate the effectiveness of our new losses, LDR-KL and ALDR-KL losses.  

\section{Acknowledgement}
{This work is partially supported by NSF Career Award 2246753, NSF Grant 2246757 and NSF Grant 2246756. The work of Yiming Ying is partially supported by NSF (DMS-2110836, IIS-2110546, and IIS-2103450).  }

\bibliography{reference}
\bibliographystyle{icml2023}

\appendix

%

%

\onecolumn

\section{Explicit expressions of all relevant losses}\label{sec:lss}
\begin{table*}[ht]
    \centering
 \caption{Comparison between different loss functions. Notice that $A$ from SCE loss is a negative value, which is suggested to be set as $-4$ in the original paper~\cite{wang2019symmetric}. For JS loss, $\mathbf m=\pi_1\e_y+(1-\pi_1)\p$. NA means not available or unknown. * means that the results are derived by us (proofs are included in the appendix~\ref{sec:discuss-loss}). 
}
    \label{table:loss-formulations}
    \resizebox{1.\textwidth}{!}{
    \begin{tabular}{l|l|l|l|l|l|l}
    \toprule
       Category&Loss & $p_y=\frac{\exp(f_y)}{\sum_k \exp(f_k)}$, $\p=(p_1, \ldots, p_K)$, $\e_y$ one-hot vector & Top-$1$ consistent & All-$k$ consistent & Symmetric&\makecell{instance\\ addaptivity}\\
       \hline
        &CE & $\psi_{\text{CE}}(\f, y) =-\log\left(p_y\right)$ &  Yes&  Yes& No&No\\
       Traditional &CS~\cite{10.5555/944790.944813} & $\psi_{\text{CS}}(\f, y, c)=\max(0, \max_{k\neq y}f_k- f_y + c)$ & No& No& No&No\\
        &WW~\cite{weston1998multi} & $\psi_{\text{WW}}(\f, y, c)=\sum_{k\neq y}\max(0, f_k - f_y + c)$ & No& No& No&No\\
        \hline 
        \multirow{2}*{Symmetric}&MAE (RCE, clipped)~\cite{ghosh2017robust} & $\psi_{\text{MAE}}(\f, y) = 2(1-p_y)$ & Yes& No & Yes&No\\
        &NCE~\cite{ma2020normalized}&$\psi_{\text{NCE}}(\f, y) =\log\left(p_y\right)/\sum_k\log\left(p_k\right)$&Yes&No&Yes &No\\
        &RLL~\cite{patel2021memorization} & $\psi_{\text{RL}}(\f, y, \alpha) = -\log(\alpha+p_y)+\frac{1}{K-1}\sum_{i\neq y}\log(\alpha+p_i)$ & No &  No& Yes&No\\
\hline
       Interpolation between &GCE~\cite{zhang2018generalized} & $\psi_{\text{GCE}}(\f, y, q) = (1-p_y^q)/q$ where ($q\in[0,1]$) & Yes& Yes$^*$ ($0\le q <1$)& Yes ($q=1$)&No\\
       CE and MAE &SCE~\cite{wang2019symmetric} & $\psi_{\text{SCE}}(\f, y, \alpha, A) = \alpha\psi_{\text{CE}}(\f, y) - (1-\alpha)\frac{A}{2}\psi_{\text{MAE}}(\f, y)$ & Yes& Yes$^*$~($0\le\alpha<1$)& Yes ($\alpha=0$)&No\\
       &JS~\cite{englesson2021generalized}&$\psi_{\text{JS}}(\f, y, \pi_1) = (\pi_1\text{KL}(\e_y, \mathbf m) + (1-\pi_1)\text{KL}(\p, \mathbf m))/Z$&NA &NA&Yes~($\pi_1=1$)&No\\
        \hline
         \multirow{2}*{Bounded}&MSE~\cite{ghosh2017robust} & $\psi_{\text{MSE}}(\f, y) = 1-2p_y+\|\p\|_2^2$ &Yes & Yes$^*$ &No&No\\
        &TGCE~\cite{zhang2018generalized} &  $\psi_{\text{TGCE}}(\f, y, q, k) = (1-\max(p_y, k)^q)/q$  where ($q\in[0,1]$)  &NA &NA & No&No\\
        \hline 
         \multirow{2}*{Asymmetric}& AGCE~\cite{zhou2021asymmetric}& $\psi_{\text{AGCE}}(\f, y, q) = ((a+1)^q-(a+p_y)^q)/q$, ($a>0, q>0$)&Yes&No&No&No\\
        & AUL~\cite{zhou2021asymmetric}& $\psi_{\text{AUL}}(\f, y, q) = ((a-p_y)^q-(a-1)^q)/q$ ($a>1, q>0$)&Yes&No&No&No\\
        \hline
         \multirow{2}*{LDR}&LDR-KL& $\psi^{\text{KL}}_\lambda(f, \x, y) =\lambda\log\left(\sum_{k=1}^K \exp\left(\frac{f_k + c_{y,k} - f_y}{\lambda}\right)\right)$&\multicolumn{2}{|c|}{Yes$^*$ 
 ($\lambda>0$)}&Yes$^*$~($\lambda=\infty$)&No\\
                &ALDR-KL& $\psi^{\text{KL}}_{\alpha,\lambda_0}(f, \x, y)=\max_{\lambda\geq 0}\psi^{\text{KL}}_\lambda(f, \x, y)  - \frac{\alpha}{2}(\lambda-\lambda_0)^2$& \multicolumn{2}{|c|}{Yes$^*$ ($\lambda_0>0, \alpha>\frac{\log K}{\lambda_0}$)}&Yes$^*$ ($\lambda_0=\infty$)&Yes\\
        \bottomrule
    \end{tabular}}
    \vspace*{-0.2in}
\end{table*}
    
\section{Proof for LDR-KL Formulation and Special cases.}\label{sec:ps}
\begin{proof}
Let $\q=\f-\f_y+\c_y$, we want to solve:
$$\max_{\p\in\Delta}\p^\top\q-\lambda\sum_i\p_i\log K\p_i.$$

Using Langrangian duality theory to handle the constraint $\sum_ip_i=1$, we have
$$\min_\eta\max_{\p\in\Delta}\p^\top\q-\lambda\sum_i\p_i\log K\p_i - \eta (\sum_ip_i-1).$$

By maximizing over $\p$, we have $p^*_i = \exp((q_i-\eta)/\lambda)$. With the Karush?Kuhn?Tucker (KKT) condition $\sum_ip_i^*=1$, we have $p^*_i = \frac{\exp(q_i/\lambda)}{\sum_i \exp(q_i/\lambda)}$. By plugging this back we obtain the formulation of LDR-KL loss, i.e, 
\begin{align*}
   \psi^{\text{KL}}_\lambda(\x, y) =\lambda \log\left(\frac{1}{K}\sum_{k=1}^K \exp\left(\frac{f_k + c_{k,y} - f_y}{\lambda}\right)\right).
\end{align*}

{\bf Extreme Case $\lambda=0$:}, with $\lambda=0$, the optimal $\p_*$ will be one-hot vector that is only 1 for the  $k$ such that $f_k(\x) - f_y(\x) + c_{k,y}$ is largest. Since $f_y(\x) - f_y(\x) + c_{y,y}=0$. Hence, we have $\sum_kp_k^*(f_k(\x) - f_y(\x) + c_{k,y})= \max(0, \max_{k\neq y}f_k(\x) - f_y(\x) + c_y)$.  

{\bf Extreme Case $\lambda=\infty$: }, with $\lambda=\infty$, the optimal $\p_*=1/K$. Hence, we have $\psi^{\text{KL}}_\lambda(\x, y)  = \sum_kp_k^*(f_k(\x) - f_y(\x) + c_{k,y}) - \lambda\sum_ip_i^*\log p_i^*K= \frac{1}{K}\sum_{k=1}^K(f_k(\x) - f_y(\x) + c_{k,y})$.  

\end{proof}

\section{Proof for Lemma~\ref{lem:rank-cali}}

\begin{proof}
Fix $\x$, with $\q=(\Pr(y=1|\x), \ldots, \Pr(y=K|\x))$, we will prove that if $\f$ is rank consistent with respect to $\q$, then $P_k(\f, \q)$, i.e., $\f$ is top-$k$ preserving with respect to $\q$ for any $k\in[K]$. Fix $k$,  let $r_k(\f)$ be any top-$k$ selector of $\f$. First we have $1 - \sum_{m\in r_k(\f)}\q_m\geq 1 - \sum_{m=1}^j\q_{[k]}$. If $f$ is rank consistent, then the equality holds, i.e., $ \sum_{m\in r_k(\f)}q_m =\sum_{m=1}^j\q_{[k]} $. This is because that if the equality does not hold, then there exists $i\in r_k(f)$ and $j\in[K]\backslash r_k(f)$ (i.e., $f_j\leq f_i$) such that $q_j > q_i$. This would not happen as $\f$ is rank consistent with $\q$ according to assumption. 

The following argument follows~\cite{yang2020consistency}. If $\neg P_k(\f, \q)$, then there exists $i\in[K]$ such that $q_i > q_{[k+1]}$ but $f_i\leq  f_{[k+1]}$, or $q_i < q_{[k]}$ but $f_i\geq f_{[k]}$. In the first case, there is an $r_k$ such that $i\not\in r_k(\f)$, because there are at least $k$ indices $j\in[K]$, $j\ne i$ such that $f_j > f_i$. In the second case, there is an $r_k$ such that $i\in r_k(\f)$, because $f_i$ is one of the top $k$ values of $\f$. In either case, there is an $r_k$ such that$ \sum_{m\in r_k(\f)}\q_m <\sum_{m=1}^j\q_{[k]} $. This contradicts to the conclusion derived before. Hence, we can conclude that if $f$ is rank consistent with $\q$, we have $P_k(\f,q)$ for any $k\in[K]$. As a result, we conclude that $\inf_{\f\in\R^K: \neg P_k(\f, \q)}L_\psi(\f, \q)> L_\psi(\f^*, \q)$ if $\f^*$ is rank consistent with repect to $\q$. 
\end{proof}

\section{Proof for Theorem~1}
\begin{proof}
For simplicity, we consider $\lambda=1$. The extension to $\lambda>0$ is trivial. Let us consider the optimization problem: 
\begin{align*}
\min_{\f\in\R^K}F(\f)&:=\sum_kq_k \psi^\lambda_k(\f) = \sum_k q_k \log\left(\sum_l \exp(f_l - f_k + c(l,k))\right)
\end{align*}
By the first-order optimality condition $\partial F(f)/\partial f_j=0$, we have:
\begin{align*}
 - q_j&\frac{\sum_{l\neq j} \exp(f_l - f_j + c(l,j))}{\sum_l \exp(f_l - f_j + c(l,j))} + \sum_{k\neq j}q_k \frac{ \exp(f_j - f_k + c(j,k))}{\sum_l \exp(f_l - f_k + c(l,k))} = 0
\end{align*}
Move the negative term to the right and add $\frac{q_j}{\sum_l \exp(f_l - f_j + c(l,j))}$ to both sides:
\begin{align*}
  &q_j\frac{\sum_{l} \exp(f_l - f_j + c(l,j))}{\sum_l \exp(f_l - f_j + c(l,j))}  \\
  & = \sum_{k\neq j} \frac{q_k \exp(f_j - f_k + c(j,k))}{\sum_l \exp(f_l - f_k + c(l,k))} + \frac{q_j}{\sum_l \exp(f_l - f_j + c(l,j))}\\
  & = \sum_{k} \frac{ q_k\exp(f_j - f_k + c(j,k))}{\sum_l \exp(f_l - f_k + c(l,k))} \\
  &= \exp(f_j)\sum_{k} \frac{ q_k}{\sum_l \exp(f_l + c(l,k) - c(j,k))}
  \end{align*}
Moreover:
\begin{align*}
&\sum_{k} \frac{q_k}{\sum_l \exp(f_l + c(l,k) - c(j,k))}\\ 
& =  \frac{q_j}{\exp(f_j) + \sum_{l\neq j}\exp(f_l + c_j)} + \sum_{k\neq j}\frac{q_k}{\exp(f_k - c_k) + \sum_{l\neq k}\exp(f_l)}\\
&=\frac{q_j}{\exp(f_j) + \sum_{l\neq j}\exp(f_l + c_j)} + \sum_{k\neq j}\frac{q_k\exp(c_k)}{\exp(f_k) + \sum_{l\neq k}\exp(f_l + c_k)}\\
& = \sum_k \frac{q_k \exp(c(j,k))}{\exp(f_k) + \sum_{l\neq k}\exp(f_l + c_k)} \\
&= \frac{\sum_k q_k \exp(c(j,k))}{Z}
\end{align*}
where $Z>0$ is independent of $j$. 
Then
\begin{align*}
\exp(f_j^*) = \frac{q_jZ}{ \sum_k q_k \exp(c(j,k))} = \frac{q_jZ}{q_j  +  \sum_{k\neq j} q_k \exp(c_k)} 
\end{align*}
If $c_k =c$, we have
\begin{align*}
\exp(f_j^*) &= \frac{q_jZ}{ \sum_k q_k \exp(c(j,k))} \\
&= \frac{q_jZ}{q_j  +  \sum_k q_k \exp(c_k) - q_j\exp(c_j) }\\
&=  \frac{q_jZ}{  \sum_kq_k\exp(c_k)  - q_j(\exp(c_j) - 1)}
\end{align*}
Hence the larger $q_j$, the larger $f_j^*$. For $q_i<q_j$, we have
\begin{align*}
&1/\exp(f_i^*) - 1/\exp(f_j^*) \\
&= \frac{\exp(c)}{q_iZ} - \frac{\exp(c)-1}{Z} -  \frac{\exp(c)}{q_jZ} + \frac{\exp(c)-1}{Z}>0
\end{align*}
It implies that $f_i^*<f_j^*$. 
\end{proof}

\section{Proof of Theorem~\ref{thm:2}}
\begin{definition}
For a function $R: \Omega\rightarrow \R\cap\{-\infty, \infty\}$ taking values on extended real number line, its convex conjugate is defined as 
\begin{align*}
R^*(\u) = \max_{\p\in\Omega}\p^{\top}\u - R(\p)
\end{align*}
\end{definition}
\begin{lemma}\label{lem:1}
If $R(\cdot)$ and $\Omega$ are (input-) symmetric, then $R^*$ is also input-symmetric. 
\end{lemma}

\begin{proof}
\begin{align*}
R^*(\u) = \max_{\p\in\Omega}\p^{\top}\u - R(\p)
\end{align*}
For any $i, j$, let $\uh$ be a shuffled version of $\u$ such that $\uh_i = \u_j, \uh_j = \u_i, \uh_k=\u_k, \forall k\neq i, j$. Then 
\begin{align*}
R^*(\uh) &= \max_{\p\in\Omega}\p^{\top}\uh - R(\p) = \max_{\p\in\Omega}\sum_k p_k \uh_k - R(\p)\\
& =  \max_{\p\in\Omega}p_i \u_j + p_j \u_i + \sum_{k\neq i, j}p_k\u_k - R(\p)\\
& = \max_{\p \in\Omega}\widehat\p^{\top}\u - R(\p)  = \max_{\widehat\p \in\Omega}\widehat\p^{\top}\u - R(\widehat\p)\\
&= R^*(\u)
\end{align*}
where $\widehat\p$ is a shuffled version of $\p$. 
\end{proof}

\begin{proof}{\bf{Proof of Theorem~2}}
, without loss of generalization, we assume $\lambda=1$. 
\begin{align*}
\psi_y(\f)& = \max_{\p\in\Omega}\sum_k p_k (f_k - f_y + c_{k,y}) - R(\p)\\
&= R^*(\f - f_y\mathbf 1 +  \c_y)
\end{align*}
where $\c_y= c(1- \e_y)$, $\e_y$ is the $y$-th column of identity matrix. Let us consider the optimization problem: 
\begin{align*}
\f^*&=\arg\min_{\f\in\R^K}F(\f):=\sum_kq_k \psi_k(\f) \\
&= \sum_k q_kR^*(\f - f_k\mathbf 1  +  \c_k)
\end{align*}
Consider $\q$ such that $q_1<q_2$, we prove that $\f^*$ is order preserving. We prove by contradiction. Assume $\f^*_1>\f^*_2$. Define $\fh$ as $\fh_1=\f^*_2$ and $\fh_2=\f^*_1$ and $\fh_k = \f^*_k, k>2$. 
Then 
\begin{align*}
F(\fh)& =\sum_k q_kR^*(\fh - \widehat f_k\mathbf 1 +  \c_k) \\
&= q_1R^*(\fh - \widehat f_1\mathbf 1 +  \c_1)  + q_2R^*(\fh - \widehat f_2\mathbf 1 +  \c_2)+\sum_{k>2}q_k R^*(\fh - \widehat f_k\mathbf 1 + \c_k) \\
& = q_1R^*(\fh - f_2^*\mathbf 1 +  \c_1)  + q_2R^*(\fh -f_1^*\mathbf 1 + \c_2) +\sum_{k>2}q_k R^*(\f - f_k^*\mathbf 1  +  \c_k)
\end{align*}
Then 
\begin{align*}
F(\fh) - F(\f^*) &= q_1 (R^*(\fh - f_2^*\mathbf 1 +\c_1) - R^*(\f - f_1^*\mathbf 1  + \c_1)) + q_2 (R^*(\fh -f_1^*\mathbf 1 +  \c_2)- R^*(\f -f_2^*\mathbf 1 +  \c_2))
\end{align*}
Since $[\fh - f_2^*\mathbf 1 +  \c_1]_1 = 0 $, $[\fh - f_2^*\mathbf 1 +  \c_1]_2=f_1^* - f^*_2  + c$, $[\fh - f^*_2\mathbf 1  + \c_1]_k = f^*_k -f^*_2  + c, k>2$, similarly $[\f -f^*_2\mathbf 1  +  \c_2]_1 = f^*_1 - f^*_2 + c$, $[\f -f^*_2\mathbf 1  +  \c_2]_2 = 0$, $[\f^* -f^*_2\mathbf 1  + \c_2]_k = f^*_k - f^*_2  +  c, k>2$, hence $R^*(\fh - f^*_2\mathbf 1 +  \c_1)  = R^*(\f -f^*_2\mathbf 1  + \c_2)$. Similarly $R^*(\f^* - f^*_1\mathbf 1  + \c_1) = R^*(\fh -f_1^*\mathbf 1  +  \c_2)$. Hence, we have
\begin{align*}
F(\fh) - F(\f^*)  = (q_2- q_1)&(R^*(\f^* -f^*_1\mathbf 1 +  \c_1)- R^*(\f^* -f^*_2\mathbf 1  +  \c_2)) 
\end{align*}
Define $\f^\sharp=\f^* -f^*_1\mathbf 1 +  \c_1$ with $[\f^\sharp]_1=0, [\f^\sharp]_k=f^*_k-f^*_1 + c, k>1$, and $\f^\dagger=\f^* -f^*_2\mathbf 1  +  \c_2$ with $[\f^\dagger]_2=0, [\f^\dagger]_k=f^*_k-f^*_2 + c, k\neq 2$. Let $\f^\ddagger$ be a shuffled version of $\f^\dagger$ with $[\f^\ddagger]_1=0, [\f^\ddagger]_k=f^*_1 -f^*_2  +  c ,k=2, [\f^\ddagger]_k=f^*_k -f^*_2  +  c, k>2$.  Under the assumption that $f^*_1>f^*_2$, we have $\f^\sharp\leq \f^\ddagger$. Then
\begin{align*}
F(\fh) - F(\f^*)  = (q_2- q_1)(R^*(\f^\sharp)- R^*(\f^\ddagger))
\end{align*}
Note that 
\[
R^*(\f) =\max_{\p\in\Omega}\p^{\top}\f - R(\p) 
\]
Below, we prove $R^*(\f^\sharp)- R^*(\f^\ddagger)<0$. We will consider two scenarios. First, consider $f^*_2 - f^*_1 + c>0$.  Let $\p^*= \arg\max_{\p\in\Omega}\sum_{k>1}p_k [\f^\sharp]_k - R(\p)$. We have $\f^\sharp\in \nabla R(\p^*)$. Since $[\f^\sharp]_2>0$, we have $p^*_2>0$ (due to the assumption on $R$).  Hence $R^*(\f^\sharp)= \max_{\p\in\Omega}\sum_{k>1}p_k [\f^\sharp]_k - R(\p) = \sum_{k>1}p^*_k [\f^\sharp]_k - R(\p^*)<  \sum_{k>1}p^*_k [\f^\ddagger]_k - R(\p^*)\leq R^*(\f^\ddagger)$. Next, consider $f^*_2 - f^*_1 + c\leq 0$.  By the convexity of $R^*$, we have $R^*(\f^\sharp)- R^*(\f^\dagger)\leq  \nabla R^*(\f^\sharp)( \f^\sharp  -\f^\dagger) = [\nabla R^*(\f^\sharp)]_1(f_2- f_1 - c) +[\nabla R^*(\f^\sharp)]_2(f^*_2- f^*_1 +  c)+ \sum_{k>2}[\nabla R^*(\f^\sharp)]_k(f_2 -f_1) \leq c([\nabla R^*(\f^\sharp)]_2 - [\nabla R^*(\f^\sharp)]_1) +  \sum_k[\nabla R^*(\f^\sharp)]_k(f^*_2 -f^*_1) <0$ due to that (i) $[\nabla R^*(\f^\sharp)]_2\leq [\nabla R^*(\f^\sharp)]_1$ due to $[\f^\sharp]_2\leq [\f^\sharp]_1$; (ii) and $\sum_k[\nabla R^*(\f^\sharp)]_k = \sum_k p^*_k >0$. This is because $\p^*$ cannot be all zeros otherwise increasing the first coordinate of $\p^*$ will decrease the  value of $R(\p)$ and therefore increase the value of $Q$. However, such a solution $\p^*$ is impossible since for the function  $Q(\p) = \p^{\top}\f^{\sharp} - R(\p) $, which has a unique maximizer  due to strong concavity of $Q(\p)$. This will prove that $R^*(\f^\sharp)- R^*(\f^\ddagger)<0$. As a result, we prove that if $f^*_1>f^*_2$, then $F(\fh)< F(\f^*)$, which is impossible since $\f^*$ is the minimizer of $F$. Hence, we have $f^*_1\leq f^*_2$. 

Next, we prove that $f^*_1<f^*_2$. Assume $f^*_1 = f^*_2$, we establish a contradiction. By the first-order optimality condition, we have 
\begin{align*}
\frac{\partial F(f)}{\partial f_k} & =q_k \frac{\psi_k(\f)}{\partial f_k} + \sum_{l\neq k}q_l\frac{\partial \psi_l(\f)}{\partial f_k}\\
&= q_k (\e_k - 1)^{\top}\nabla R^*(\f - f_k\mathbf 1 - \c_k)+\sum_{l\neq k}q_l\e_k^{\top}\nabla R^*(\f - f_l\mathbf 1 - \c_l) = 0
\end{align*}
Hence
\begin{align*}
q_k \mathbf 1^{\top}\nabla R^*(\f^* - f^*_k\mathbf 1 - \c_k)  =\e_k^{\top} \sum_{l}q_l\nabla R^*(\f^* - f^*_l\mathbf 1 - \c_l) 
\end{align*}
Next, we prove that for any $\f$ such that $[\f]_1=[\f]_2$, then $[\nabla R^*(\f)]_1 = [\nabla R^*(\f)]_2$. To this end, we consider the equation,  $R(\p) + R^*(\f) = \p^{\top}\f$, where $\p = \arg\max_{\p\in\Omega}Q(\p):= \p^{\top}\f - R(\p)$ satisfying $\p = \nabla R^*(\f)\in\Omega$. If  $[\nabla R^*(\f)]_1 \neq  [\nabla R^*(\f)]_2$, which means $p_1\neq p_2$. We can define another vector $\p'=(p_2, p_1, p_3, \ldots, p_K)\in\Omega, \p'\neq \p$, which satisfies $Q(\p') = \p'^{\top}\f - R(\p') = Q(\p)$, which contradicts to the fact that $Q(\cdot)$ has a unique maximizer. Then  for any $\f$ with $[\f]_1=[\f]_2$, we have $[\f - f_l\mathbf 1 - \c_l]_1 =[\f - f_l\mathbf 1 - \c_l]_2$. As a result, 
\begin{align*}
q_1 \mathbf 1^{\top}\nabla R^*(\f^* - f^*_1\mathbf 1 - \c_1) = q_2 \mathbf 1^{\top}\nabla R^*(\f^* - f^*_2\mathbf 1 - \c_2) 
\end{align*}
It is clear that the above equation is impossible if $f_1^*=f_2^*$ since  $\nabla R^*(\f)$ is symmetric, $q_1<q_2$ and $\mathbf 1^{\top}\nabla R^*(\f)\neq 0, \forall \f$ with $[f]_1=0$. The later is due to the assumption $\nabla R(0)<0$, the optimal solution  $\p = \arg\max_{\p\in\Omega}Q(\p):= \p^{\top}\f - R(\p)$ satisfying $\p=\nabla R^*(\f)\in\Omega$ cannot be all zeros (otherwise we can find a better solution by increasing the first component of $\p$ with a larger value  of $\p^{\top}\f - R(\p)$.
\end{proof}

\section{Proof for Theorem~\ref{thm:4}}
Let $\q=\q(\x)$ be class distribution for a given $\x$. Let $\pi^q_k$ denote the index of $q$ whose corresponding value is the $k$-th largest entry. We define $\psi(\f, l) = \psi(f, \x, l)$. 
\begin{lemma}\label{lem:2}
If $\psi$ is top-$1,2$ consistent simultaneously when class number $K\ge3$, then $\exists~\q=\q(\x)$  such that $q_{[3]}<q_{[2]}<q_{[1]}$ makes $\psi(\f^*,\pi^q_1)<\psi(\f^*,\pi^q_2)<\psi(\f^*,\pi^q_3)$, here $\f^* = \inf L_\psi(\f, \q) = \sum_lq_l\psi(\f, l)$. 
\end{lemma}
\begin{proof}
Without loss of generality, we consider $q_1>q_2>q_3$ and prove that if $\psi(\f^*,1)<\psi(\f^*,2)<\psi(\f^*,3)$ does not hold, then we can construct $\q'$ such that $q'_{[3]}<q'_{[2]}<q'_{[1]}$ and $\psi(\f^*,\pi^{q'}_1)<\psi(\f^*,\pi^{q'}_2)<\psi(\f^*,\pi^{q'}_3)$.

If $\psi(\f^*,1)\geq \psi(\f^*,2)$ and we construct $\q'$ such that  $q'_{[3]}<q'_{[2]}<q'_{[1]}$ and $\psi(\hat\f^*,\pi^{q'}_1)<\psi(\hat\f^*,\pi^{q'}_2)$, where  $\hat\f^*$ be the optimal solution to $ \inf L_\psi(\f, \q') = \sum_lq'_l\psi(\f, l)$.  Then we construct $\q'$ by switching the 1st and 2nd entry of $\q$.  Then we have $L_\psi(\f^*, \q)\geq L_\psi(\f^*,\q')$ because $\psi(\f^*, 1)q_1 + \psi(\f^*, 2)q_2 \ge\psi(\f^*, 1)q'_1 + \psi(\f^*, 2)q'_2$ due to $q'_1 = q_2, q'_2 = q_1$ and $q_1>q_2$. Due to that $\psi$ is top-1, 2 consistent,  it is top-1,2 calibrated and hence $P_k(\f^*, \q)$ and $P_k(\hat \f^*, \q')$ hold for $k=1, 2$. As a result, we can prove that $\f^*_1>\f^*_2>\f^*_3$ and $\hat\f^*_2>\hat\f^*_1>\f^*_3$. As a result, $\f_*\neq \hat \f_*$. Next, we prove that  $\psi(\hat\f^*,\pi^{q'}_1)< \psi(\hat\f^*,\pi^{q'}_2)$. If $\psi(\hat\f^*,\pi^{q'}_1)\geq  \psi(\hat\f^*,\pi^{q'}_2)$, then  $L_\psi(\hat\f^*, \q')\geq L_\psi(\hat\f^*,\q)$ because $\psi(\hat\f^*, 1)q'_1 + \psi(\hat\f^*, 2)q'_2 \ge\psi(\hat\f^*, 1)q_1 + \psi(\hat\f^*, 2)q_2$ due to $\psi(\hat\f^*, 2)=\psi(\hat\f^*, \pi^{q'}_1)$ and  $\psi(\hat\f^*, 1)=\psi(\hat\f^*, \pi^{q'}_2)$ and $q_1>q_2$. Hence, we derive a fact that $L_\psi(\f^*, \q)\geq L_\psi(\f^*,\q')\geq L_\psi(\hat\f^*, \q')\geq L_\psi(\hat\f^*,\q)$ under the assumption that  $\psi(\hat\f^*,\pi^{q'}_1)\geq  \psi(\hat\f^*,\pi^{q'}_2)$, which is impossible as $\hat \f^*$ is not the optimal solution to $\inf L_\psi(\f, \q) = \sum_lq_l\psi(\f, l)$.  Thus, we have  $\psi(\hat\f^*,\pi^{q'}_1)<  \psi(\hat\f^*,\pi^{q'}_2)$. As a result, there exists $\q'$ such that $q'_{[3]}<q'_{[2]}<q'_{[1]}$ and $\psi(\f^*,\pi^{q'}_1)<\psi(\f^*,\pi^{q'}_2)$. 

If $\psi(\f^*,3)\geq \psi(\f^*,2)$ and construct $\q'$ such that $q'_{[3]}<q'_{[2]}<q'_{[1]}$ and $\psi(\hat\f^*,\pi^{q'}_2)<\psi(\hat\f^*,\pi^{q'}_3)$. This case can be proved similarly as above. 

\end{proof}

Based on above lemma, we can finish the proof of Theorem~3.  For any $f(\x)$, let $\psi_f = (\psi(f, \x, 1), \ldots, \psi(f, \x, K))$. Then $\min_f \sum_l\psi(f, \x, l)q_l = \min_{\psi_f\in\R^K}\psi_f^{\top}\q$. Since $\psi_f^{\top}1 = C$ due to the symmetry property. As a result, there exists $\psi_{f_*} \in \arg\min_{\psi_f\in\R^K}\psi_f^{\top}\q, s.t. \quad\psi_f^{\top}1=C, \psi_f\geq 0$ such that it is at a vertex of the feasible region (a generalized simplex),  i.e., there only exist one $i\in[K]$ such that $[\psi_{f_*}]_i \neq 0$.  This holds for any $\q$. Combining with the above lemma, we can conclude a contradiction, i.e., a non-negative symmetric loss cannot be top-1, 2 consistent simultaneously.

\section{Proof for Theorem~\ref{thm:3}}
\begin{proof}
When $\lambda=\infty$, the loss function becomes: $\psi^{\text{KL}}_\infty(\x, y) = \frac{1}{K}\sum_{k=1}^K (f_k(\x) - f_y(\x)+ c(k,y))$. To show it is symmetric: $\sum_{j=1}^K\psi^{\text{KL}}_\infty\big(\x,y=j\big)=C,  ~\forall f(\x)$, we have 
\begin{align*}
    \sum_{j=1}^K\psi^{\text{KL}}_\infty\big(\x,y=j\big) = \frac{1}{K}\sum_{j=1}^K\sum_{k=1}^K (f_k(\x) - f_j(\x) + c(k,j))=\frac{1}{K}\sum_{j=1}^K\sum_{k=1}^K c(k,j) = C
\end{align*}
where $C$ is a constant only depending on the predefined margins.

To prove All-$k$ consistency, we prove that minimizing the expected loss function is rank preserving, i.e., $\f^*=\arg\min_{\f\in\mathcal C}\sum_kq_k \psi_k(\f)$ is rank preserving, where $ \psi_y(\f)=\psi^{\text{KL}}_\infty(\x, y)= \frac{1}{K}\sum_{k=1}^K ([\f]_k - [\f]_y+ c(k,y))$. We can write $ \psi_y(\f) = \frac{1}{K}\a_y^{\top}\f + \frac{(K-1)c_y}{K}$, where $[\a_y]_{y}=1-K, [\a_y]_k=1, \forall k\neq y$. As a result, we have $\sum_kq_k \psi_k(\f)= 
 (\sum_k q_k\a_k)^{\top}\f+ \frac{(K-1)\sum_kq_kc_k}{K}$, where $[\sum_kq_k\a_k]_i = 1 - Kq_i$. Then minimizing $\arg\min_{\|\f\|_2\leq B}\sum_kq_k \psi_k(\f) = \arg\max_{\|\f\|_2\leq B}\f^{\top}\hat\a$, where $[\hat\a]_i = Kq_i - 1$. As a result, $\f^*= \frac{B\hat\a}{\|\hat\a\|_2}$. Hence, if $q_i>q_j$, then $\f^*_i>\f^*_j$, which proves the rank preserving property of the optimal solution. Then, the conclusion follows the Lemma~\ref{lem:rank-cali}. 

\end{proof}


\section{Proof for Theorem~\ref{thm:5}}

\begin{proof}
We first consider $\alpha, \lambda_0\in(0, \infty)$. For the ALDR-KL loss,  
\begin{align}
\psi_{\alpha, \lambda_0}^{\text{KL}}(\x, y) =& \max_{\lambda\in\R_+}\max_{\p\in\Delta_K}\sum_{k=1}^K p_k (\delta_{k,y}(f(\x)) + c_{k, y})\nonumber\\
&- \lambda \text{KL}(\p, \frac{1}{K}) - \frac{\alpha}{2} (\lambda-\lambda_0)^2\\
 =& \max_{\lambda\in\R_+}\lambda \log \frac{1}{K}\sum_{k=1}^K\exp(\frac{f_k - f_y + c(k, y)}{\lambda})  - \frac{\alpha}{2} (\lambda-\lambda_0)^2\nonumber
\end{align}
Then the expected loss is given by 
\begin{align*}
&F(\f)=\sum_k q_k \psi_{\alpha, \lambda_0}^{\text{KL}}(\x, k) = \sum_k q_k \max_{\lambda\in\R_+}\lambda \log \frac{1}{K}\sum_{l=1}^K\exp(\frac{f_l - f_k + c(l, k)}{\lambda})  - \frac{\alpha}{2} (\lambda-\lambda_0)^2\\
&= \sum_k q_k \max_{\lambda\in\R_+}(-f_k  + \log \frac{1}{K}\sum_{l=1}^K\exp(\frac{f_l + c(l, k)}{\lambda})  - \frac{\alpha}{2} (\lambda-\lambda_0)^2).
\end{align*}
Assuming that $c_{l,k}=0$ (i.e., no margin is used), then we have
\begin{align*}
&F(\f)=\sum_k q_k \psi_{\alpha, \lambda_0}^{\text{KL}}(\x, k)= \sum_k -q_kf_k+ \sum_k q_k \max_{\lambda\in\R_+} \log \frac{1}{K}\sum_{l=1}^K\exp(\frac{f_l}{\lambda})  - \frac{\alpha}{2} (\lambda-\lambda_0)^2\\
& =  \sum_k -q_kf_k+ \max_{\lambda\in\R_+}  \log \frac{1}{K}\sum_{l=1}^K\exp(\frac{f_l}{\lambda})  - \frac{\alpha}{2} (\lambda-\lambda_0)^2
\end{align*}

Similar to Theorem~\ref{thm:1}, we need to show the optimal solution $\f^*$ to the above problem is rank preserving as $\q$. Denote by the optimal solution to $\lambda$ as $\lambda_*$ and to $\p$ as $\p_*$. Then $[\p_*]_j=\exp((f_j - f_y+c_{j,y})/\lambda_*)/\sum_k\exp((f_k - f_y +c_{k,y})/\lambda_*)$. Then $\partial F/\partial \f$ is same as $\partial F_{\lambda^*}/\partial\f$. Hence as long as $\lambda_*>0$ is finite, we can follow the proof of Theorem~\ref{thm:1} to prove the rank preserving. We know that $\lambda_* = [\lambda_0 - \frac{1}{\alpha}\text{KL}(\p_*, 1/K)]_+$. Hence, if $\lambda_0$ is finite, $\text{KL}(\p_*, 1/K)$ is clearly finite, so $\lambda_*$ is finite. $\lambda_*>0$ is ensured if $\frac{\log K}{\lambda_0}<\alpha$ as $\max_{\p\in\Delta} \text{KL}(\p, 1/K)=\log K$.  Then the conclusion follows. 

If $\lambda_0=\infty$, then we have $\lambda_*=\infty$; otherwise the loss is negative infinity. As a result, $\p_*=1/K$. Then it reduces to that of Theorem~\ref{thm:3}. 

\end{proof}

\section{Efficient Computation of LDR-$k$--KL loss.}
 As an interesting application of our general LDR-$k$ loss, we restrict our attention to the use of KL divergence for the regularization term $R(\p) =\sum_{i=1}p_i\log(Kp_i)$ in the definition of~E.q. \ref{eqn:sk},
\begin{align}\label{eqn:sk}
\hspace*{-0.1in}\psi^{k}_\lambda(\x, y) = &\max_{\p\in\Omega(k)}\sum_l p_l (\delta_{l,y}(f(\x)) + c_{l,y})- \lambda R(\p) 
\end{align}
 to which we refer as LDR-$k$-KL loss. In particular,  we consider how to efficiently compute the LDR-$k$-KL loss. We present an efficient solution with $O(K\log K)$ complexity. 

\begin{thm}\label{thm:ldr-k-opt}
Consider the problem $\max_{\p\in\Omega(k)}\sum_i p_i q_i - \lambda \sum_{i=1}p_i\log(Kp_i)$. To compute the optimal solution,  let the $\q$ to be sorted and $[i]$ denote the index for its $i$-th largest value. Given $a\in[K]$ define a vector $\p(a)$ as
 \begin{align*}
     [\p(a)]_{[i]}&=\frac{1}{k}, i<a, \\
     [\p(a)]_{[i]}&=\exp(\frac{\q_{[i]}}{\lambda}-1)\times\min(\frac{1}{K}, \frac{1-\frac{a-1}{k}}{\sum_{j=a}^K\exp(\frac{\q_{[j]}}{\lambda}-1)}),~~i\ge a
 \end{align*}
The optimal solution $\p_*$ is given by $\p(a)$ such that $a$ is the smallest number in $\{1,\ldots, K\}$ satisfying $[\p(a)]_{[a]}\le\frac{1}{k}$. The overall time complexity is $O(K\log (K))$.
\end{thm}
Given that the algorithm would be simple with the presented theorem~\ref{thm:ldr-k-opt}, we do not present a formal algorithm box here. Instead, we simply describe it as follows:
\begin{enumerate}
    \item scan (from largest to smallest) through vector $\q$ to get cumulative summation $\sum_{j=a}^K\exp(\frac{\q_{[j]}}{\lambda}-1)$ for all possible values of $a$.
    \item apply the theorem~\ref{thm:ldr-k-opt} to calculate $[\p(a)]_{[a]}$ for all possible values of $a$. It is obvious that the cost is linear with $K$.
\end{enumerate}

It is worth noting that the bottleneck for computing the analytical solution is sorting for $\q$, which usually costs O$(K\log K)$. Hence, the computation of the LDR-$k$-KL loss can be efficient conducted. 
\begin{proof}
Apply Lagrangian multiplier: \begin{align*}
\max_{\p\ge0}\min_{\beta\ge0,\boldsymbol\gamma\ge0}\p^\top\frac{\q}{\lambda}&-\sum_i\p_i\log K\p_i+\beta(1-\sum_i\p_i)+\sum_i\gamma_i(\frac{1}{k}-\p_i)
\end{align*}
Stationary condition of $\p$: $$K\p^*_i=\exp(\frac{\q_i}{\lambda}-\beta-\gamma_i-1)$$

Dual form:$$\min_{\beta\geq 0,\boldsymbol\gamma\geq 0}\sum_i\frac{1}{K}\exp(\frac{\q_i}{\lambda}-\beta-\gamma_i-1)+\beta+\sum_i\frac{\gamma_i}{k}$$

Stationary condition for $\beta$ and $\boldsymbol\gamma$:
$$\gamma^*_i=\max\big(\log\frac{k}{K}\exp(\frac{\q_i}{\lambda}-\beta^*-1),0\big)$$
$$\beta^*=\max\big(\log\sum_i\frac{1}{K}\exp(\frac{\q_i}{\lambda}-\gamma^*_i-1), 0\big)$$

{\bf Checkpoint I:} for the $[a]$ that $\boldsymbol\gamma_{[a]}^*=0$ and $\boldsymbol\gamma_{[a-1]}^*>0$, where $2\le a\le k+1$, (or $a=1$, $\boldsymbol\gamma_{[a-1]}^*$ is undefined):
\begin{itemize}
\item We first consider the case when $\beta^*=0$, then the $[\p(a)]^*_{[i]}=\frac{1}{k}$ for $i<a$, and $[\p(a)]^*_{[i]}=\frac{1}{K}\exp(\frac{\q_{[i]}}{\lambda}-1)$ for $i\geq a$.

\item Otherwise, if $\beta^*>0$:
\begin{align*}
[\p(a)]^*_{[a]}&=\frac{1}{K}\exp(\frac{\q_{[a]}}{\lambda}-\beta^*-1)=\frac{\exp(\frac{\q_{[a]}}{\lambda}-1)}{\sum_j\exp(\frac{\q_j}{\lambda}-\gamma_j^*-1)}\\
&=\frac{\exp(\frac{\q_{[a]}}{\lambda}-1)}{\sum_{j=a}^K\exp(\frac{\q_{[j]}}{\lambda}-\gamma_{[j]}^*-1)}\times\frac{\sum_{j=a}^K\exp(\frac{\q_{[j]}}{\lambda}-\gamma_{[j]}^*-1)}{\sum_{j=1}^K\exp(\frac{\q_{j}}{\lambda}-\gamma_j^*-1)}\\
&=\frac{\exp(\frac{\q_{[a]}}{\lambda}-1)}{\sum_{j=a}^K\exp(\frac{\q_{[j]}}{\lambda}-1)}(1-\frac{a-1}{k})
\end{align*}
It is similar to show for any $i$ that $i>a$:
$$[\p(a)]^*_{[i]}=\frac{\exp(\frac{\q_{[i]}}{\lambda}-1)}{\sum_{j=a}^K\exp(\frac{\q_{[j]}}{\lambda}-1)}(1-\frac{a-1}{k})$$
For those $i<a$ that $\gamma_{[i]}^*>0$:
\begin{align*}
[\p(a)]^*_{[i]}&=\frac{1}{K}\exp(\frac{\q_i}{\lambda}-\beta^*-\gamma_i^*-1)\\
&=\frac{\exp(\frac{\q_i}{\lambda}-\gamma_i^*-1)}{K\exp(\beta^*)}\\
&=\frac{\frac{K}{k}\exp(\beta^*-\frac{\q_i}{\lambda}+1)\exp(\frac{\q_i}{\lambda}-1)}{K\exp(\beta^*)}=\frac{1}{k}
\end{align*}

So we can get the analytical solution as far as we know $a$ subject to $\gamma_{[a]}^*=0$ and $\gamma_{[a-1]}^*>0$. The $[\p(a)]^*_{[i]}=\frac{1}{k}$ for $i<a$, and $[\p(a)]^*_{[i]}=\frac{\exp(\frac{\q_{[i]}}{\lambda}-1)}{\sum_{j=a}^K\exp(\frac{\q_{[j]}}{\lambda}-1)}(1-\frac{a-1}{k})$ for $i\geq a$.
\end{itemize}

{\bf Checkpoint II: }next, we show that as far as we find the smallest $a$ such that $\boldsymbol\gamma_{[a]}^*=0$ and $\boldsymbol\gamma_{[a-1]}^*>0$, then it is the optimal solution.
\begin{itemize}

\item Suppose $a'<a$: because $\forall a'~~s.t.~~ k+1\ge a>a'\ge1$,
assume $[\p(a')]_{[a']}\le\frac{1}{k}\implies a'$ is another smaller valid $a$, which violates the pre-condition;
therefore, $[\p(a')]_{[a']}>[\p(a)]_{[a']}=\frac{1}{k}$, the $a^*$ must lead to a $\p^*$ that violate the $\Omega^k$ constrain, hence can't be the optimal solution, contradiction.
\item Suppose $a'>a$: 

i) if $\beta^*(a)=0$, then by pre-condition $[\boldsymbol\gamma(a)]^*_{[a]}=0$ and $[\boldsymbol\gamma(a')]^*_{[a]}>0,~ \text{consequently}~\beta^*(a')=0$, which deviates from the optimality of the objective function: $$\min_{\beta\geq 0,\boldsymbol\gamma\geq 0}\sum_i\frac{1}{K}\exp(\frac{\q_i}{\lambda}-\beta-\boldsymbol\gamma_i-1)+\beta+\sum_i\frac{\boldsymbol\gamma_i}{k}$$ (notice that $[\boldsymbol\gamma(a)]^*_{[i]}=[\boldsymbol\gamma(a')]^*_{[i]}>0, \forall i <a$ in order to hold the $\Omega(k)$ constrain).

ii) if $\beta^*(a)>0$, then  $[\boldsymbol\gamma(a)]^*_{[a]}=0$ and $[\boldsymbol\gamma(a')]^*_{[a]}>0$ $\implies\beta^*(a')<\beta^*(a)$ $\implies\forall i\ge a',~ [\p(a')]_{[i]}> [\p(a)]_{[i]}$.

On the other hand, $\forall j< a'$, $[\p(a')]_{[j]}=\frac{1}{k}\ge [\p(a)]_{[j]}$ $\implies \sum_{i=a'}^K[\p(a')]_{[i]}\le 1-\sum_{i=1}^{a'-1}[\p(a')]_{[i]}\le 1-\sum_{i=1}^{a'-1}[\p(a)]_{[i]}=\sum_{i=a'}^K[\p(a)]_{[i]}$, which is contradictory with $\forall i\ge a',~~ [\p(a')]_{[i]}> [\p(a)]_{[i]}$.
\end{itemize}
\end{proof}

\section{Extensive Discussion for Loss Functions}\label{sec:discuss-loss}

From our theorem~\ref{thm:5}, we conclude that MAE, RCE, NCE and RLL are not All-$k$ consistent. We have already proved the All-$k$ consistency for LDR loss family by theorem~\ref{thm:1} and theorem~\ref{thm:2} (e.g. LDR-KL loss).

Next, we provide proofs for the All-$k$ consistency of GCE~($0\le q <1$), SCE~($0\le \alpha <1$) and MSE. 

\textbf{Proof of All-$k$ consistency for GCE~($0\le q <1$)}

To make the presentation clearer, we rename the $q$ parameter in GCE loss function as $\lambda$, in order to not confuse with underlying class distribution notation $\q$.
Objective: $$\min_{\p\in\Delta_K}L_{\psi_{\text{GCE}}}(\p,\q)=\sum_{i=1}^Kq_i\frac{1-p_i^\lambda}{\lambda}$$
\begin{proof}
    We will first show $\p^*$ is rank consistent with $\q$; then, show $\f^*$ is rank consistent with $\p^*$; finally, it is easy to see that $\f^*$ is rank consistent with $\q$ by transitivity rule.

    By Lagrangian multiplier:
    $$\min_\p\max_{\alpha,\beta_i\ge 0}\sum_{i=1}^Kq_i\frac{1-p_i^\lambda}{\lambda}+\alpha(\sum_{i=1}^Kp_i-1)-\sum_{i=1}^K\beta_ip_i$$

    The stationary condition for $\p$:
    $$p_i^*=\left(\frac{q_i}{\alpha-\beta_i}\right)^{\frac{1}{1-\lambda}}$$
    Notice that the underlying class distribution $\q\in\Delta_K$, and $0\le\lambda<1\implies 1\le\frac{1}{1-\lambda}$; moreover, with primal feasibility for KKT condition, $p^*_i\ge0$ and $\sum_{i=1}^Kp^*_i=1$ (it also practically holds because the model requires the output $\p\in\Delta_K$, e.g. softmax); then, consider the complementary slackness: $p^*_i>0\implies \beta^*_i=0$; besides, $p^*_i\ge0, q_i\ge0\implies\alpha>0$; 
    therefore, $p^*_i\propto q_i^{\frac{1}{1-\lambda}}$ and consequently $\p^*$ is rank consistent with $\q$.

    Next, it is clear that $p_i=\frac{\exp(f_i)}{\sum_k \exp(f_k)}\implies \exp(f^*_i)\propto p^*_i$; therefore, $\f^*$ is rank consistent with $\p^*$ (notice that the proof is not restricted with softmax). By transitivity, $\f^*$ is rank consistent with $\q$.
\end{proof}

\textbf{Proof of All-$k$ consistency for SCE~($0< \alpha \le1$)}

Without loss of generality, we consider the case $A=-1$ but it could be generalized to any $A<0$. We rename the $\alpha$ parameter for SCE loss as $1-\lambda$ for a better presentation. Objective:
$$\min_{\p\in\Delta_K}L_{\psi_{\text{SCE}}}(\p,\q)=\sum_{i=1}^Kq_i\left(-(1-\lambda)\log p_i +\lambda(1-p_i)\right)$$
 By Lagrangian multiplier:
    $$\min_\p\max_{\alpha,\beta_i\ge 0}\sum_{i=1}^Kq_i\left(-(1-\lambda)\log p_i +\lambda(1-p_i)\right)+\alpha(\sum_{i=1}^Kp_i-1)-\sum_{i=1}^K\beta_ip_i$$
The stationary condition for $\p$:
    $$p_i^*=\frac{(1-\lambda)q_i}{\alpha-\beta_i-\lambda q_i}$$
    $$\frac{1}{p_i^*}=\frac{\alpha-\beta_i-\lambda q_i}{(1-\lambda)q_i}=\frac{\alpha-\beta_i}{(1-\lambda)q_i}-\frac{\lambda}{1-\lambda},~~~0<p_i^*<1$$
By primal feasibility for KKT condition, $p^*_i\ge0$ and $\sum_{i=1}^Kp^*_i=1$. By complementary slackness: $p^*_i>0\implies \beta^*_i=0$. Besides, $p^*_i\ge0, q_i\ge0\implies\alpha>0$. Notice that $0<1-\lambda\le 1$. Therefore, $(\frac{1}{p_i^*}+\frac{\lambda}{1-\lambda})\propto\frac{1}{q_i}$, and consequently $\p^*$ is rank consistent with $\q$.

Next, it is obvious that $p_i=\frac{\exp(f_i)}{\sum_k \exp(f_k)}\implies \exp(f^*_i)\propto p^*_i$; therefore, $\f^*$ is rank consistent with $\p^*$. By transitivity, $\f^*$ is rank consistent with $\q$.

Notice that the proof is for the empirical form of SCE. For the theoretical form: $L_{\psi_{\text{SCE}}}(\p,\q)=H(\q,\p)+H(\p,\q)$, where $H(\q,\p)=-\E_q[\log p]$ denotes cross entropy, it can be proved with the similar logic.




\textbf{Proof of All-$k$ consistency for MSE}

Objective:
$$\min_{\p\in\Delta_K}L_{\psi_{\text{MSE}}}(\p,\q)=\sum_{i=1}^Kq_i(1-2p_i+\|\p\|_2^2)$$
By Lagrangian multiplier:
$$\min_\p\max_{\alpha,\beta_i\ge 0}\sum_{i=1}^Kq_i(1-2p_i+\|\p\|_2^2)+\alpha(\sum_{i=1}^Kp_i-1)-\sum_{i=1}^K\beta_ip_i$$
The stationary condition for $\p$:
$$p_i^*=q_i+\frac{\beta_i-\alpha}{2}$$
By the primal feasibility for KKT condition, we have $p^*_i\ge0$ and $\sum_{i=1}^Kp^*_i=1$. By complementary slackness: $p^*_i>0\implies \beta^*_i=0$. Furthermore, $p^*_i=0$, we can prove $\beta_i^*=0$ holds; otherwise, $\exists j,~s.t.~ p^*_j=0, \beta_j^*>0$, with primal feasibility $\sum_{i=1}^K\left(q_i+\frac{\beta^*_i-\alpha^*}{2}\right)=1\implies\alpha^*>0$ given that $\sum_{i=1}^Kq_i=1$. 
As a consequence, for the $p^*_i=0$, we conclude $p^*_i\le q_i$; for the $p^*_i>0$, because $\beta^*_i=0,\alpha^*>0$, we conclude $p^*_i=q_i+\frac{\beta_i^*-\alpha^*}{2}<q_i$; therefore, $\sum_{i=1}^Kp^*_i<\sum_{i=1}^Kq_i=1$, a contraction.
Hence, $\p^*=\q$ and is rank consistent with $\q$.

Next, it is obvious that $p_i=\frac{\exp(f_i)}{\sum_k \exp(f_k)}\implies \exp(f^*_i)\propto p^*_i$; therefore, $\f^*$ is rank consistent with $\p^*$. By transitivity, $\f^*$ is rank consistent with $\q$.

\section{Dataset Statistics}
\begin{table}[H]
    \centering
        \caption{statistics for the benchmark datasets}
    \label{table:data}
    \begin{tabular}{c|ccc}
    \toprule
       Dataset  &  \# of samples & \# of features & \# of classes\\
       \hline
        ALOI & 108000& 128& 1,000\\
        News20 & 15935& 62060& 20\\
        Letter & 15000& 16& 26\\
        Vowel & 990& 10& 11\\
        Kuzushiji-49 & 232365 / 38547 & 28x28& 49\\
        CIFAR-100& 50000 / 10000 & 32x32x3&100\\
        Tiny-ImageNet & 100000 / 10000 & 64x64x3& 200\\
        \bottomrule
    \end{tabular}
    \vspace*{-0.2in}
\end{table}

\section{More Experimental Results}\label{sec:exP}
\subsection{Relationship between $\lambda$ and KL divergence values $\text{KL}(\p, 1/K)$ for ALDR-KL on Synthetic Data}

We present the averaged learned $\lambda_t$ and $\text{KL}(\p_t,\frac{\mathbf 1}{K})$ values for ALDR-KL loss at Figure~\ref{fig:lamswithkls}. It justifies our motivation that we prefer to give the uncertain data with larger $\lambda$ value.
\begin{figure}[h]
    \centering
    \includegraphics[scale=0.7]{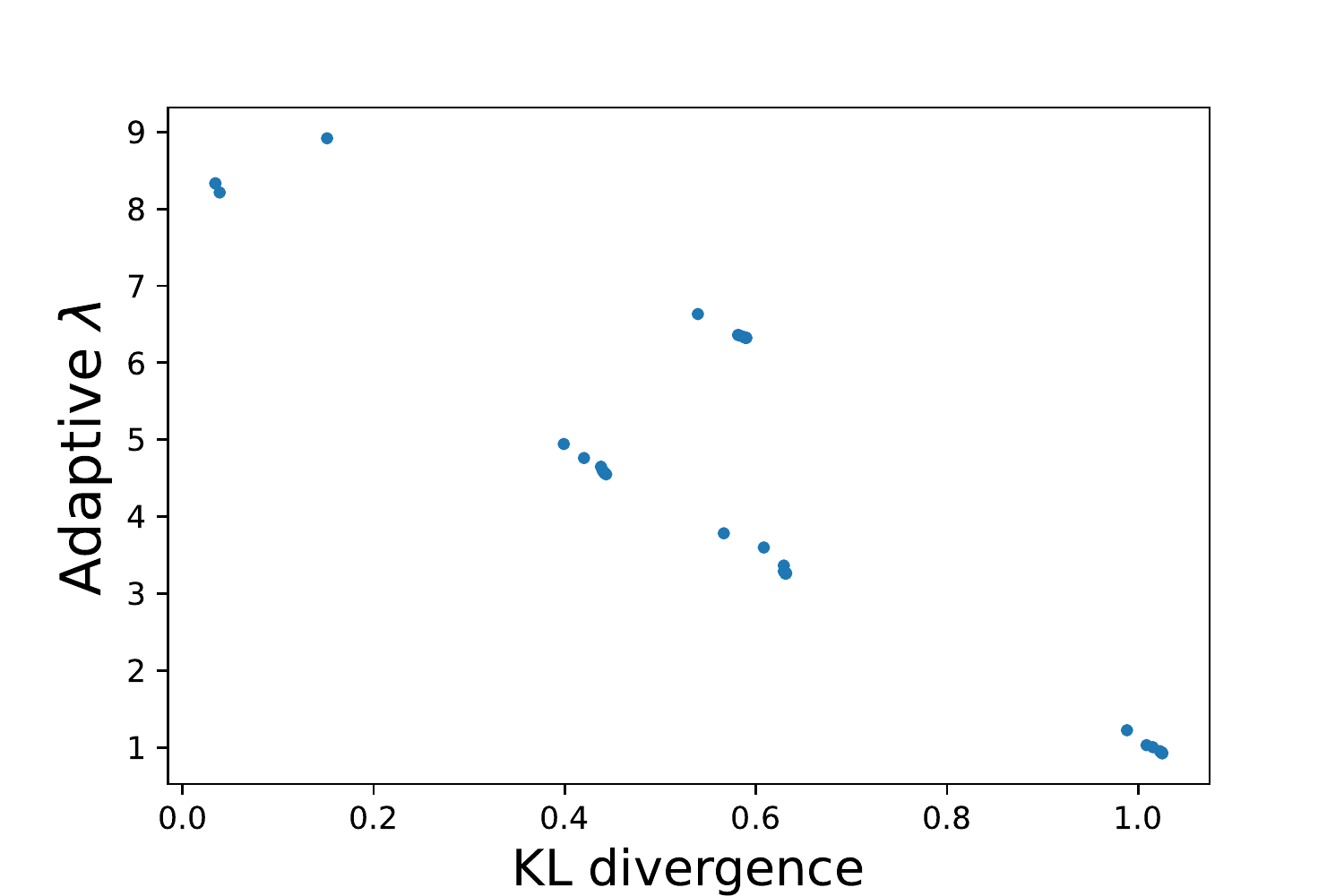}
    \caption{The averaged $\lambda_t$ and $\text{KL}(\p_t,\frac{\mathbf 1}{K})$ values across all iterations for each instance on the synthetic dataset.}
    \label{fig:lamswithkls}
\end{figure}

\subsection{Comprehensive  Accuracy Results on Benchmark datasets}
The top-$1, 2, 3, 4, 5$ accuracy results are summarized at Table~\ref{table:top-1}, \ref{table:top-11}, \ref{table:top-2}, \ref{table:top-22}, \ref{table:top-3}, \ref{table:top-33}, \ref{table:top-4}, \ref{table:top-44}, \ref{table:top-5},  \ref{table:top-55}.  From the results, we can see that (i) the performance of ALDR-KL and LDR-KL are stable; (ii) MAE has very poor performance, which is consistent with many earlier works; (ii) CE is still competitive in most cases.  
\begin{table*}[p]
    \centering
        \caption{Testing Top-1 Accuracy (\%) with mean and standard deviation for tabular dataset. The highest values are marked as bold.}
    \label{table:top-1}
    \resizebox{.95\textwidth}{!}{
\begin{tabular}{l|l|l|lll|lll}
    \toprule
    \multicolumn{1}{l}{Dataset} & Loss  & Clean & Uniform 0.3 & Uniform 0.6 & Uniform 0.9 & CD 0.1 & CD 0.3 & CD 0.5 \\
    \hline

    \multirow{15}[0]{*}{Vowel} & ALDR-KL & 71.3(3.1) & \textbf{58.2(6.38)} & 39.8(6.25) & \textbf{15.2(6.13)} & 69.1(3.36) & \textbf{61.2(2.08)} & 39.4(3.06) \\
          & LDR-KL & \textbf{74.1(2.18)} & 56.4(5.32) & 35.4(4.65) & 13.7(4.02) & \textbf{73.1(2.68)} & 58.6(6.39) & \textbf{41.0(3.7)} \\
          & CE    & 63.0(3.48) & 53.9(3.42) & \textbf{41.0(2.9)} & 14.3(4.88) & 65.9(3.52) & 56.4(4.88) & 35.2(2.34) \\
          & SCE   & 64.4(3.69) & 55.4(4.35) & 37.4(4.19) & 11.9(2.96) & 65.5(2.16) & 55.6(3.94) & 34.7(3.17) \\
          & GCE   & 64.2(4.41) & 52.1(4.81) & 39.2(2.51) & 11.9(6.65) & 64.6(2.47) & 56.4(3.58) & 39.0(3.23) \\
          & TGCE  & 64.6(3.61) & 52.1(4.81) & 38.6(2.74) & 11.9(6.65) & 64.6(2.47) & 56.4(3.58) & 38.8(3.48) \\
          & WW    & 58.0(4.22) & 47.9(1.51) & 34.1(6.83) & 11.3(6.56) & 57.0(1.87) & 50.9(2.08) & 37.2(3.4) \\
\bottomrule
    \end{tabular}}
    \vspace*{-0.2in}
\end{table*}

\begin{table*}[p]
    \centering
        \caption{Testing Top-1 Accuracy (\%) with mean and standard deviation for image dataset. The highest values are marked as bold.}
    \label{table:top-11}
    \resizebox{1.\textwidth}{!}{
    \begin{tabular}{l|l|l|lll|lll}
     \toprule
    \multicolumn{1}{l}{Dataset} & Loss  & Clean & Uniform 0.3 & Uniform 0.6 & Uniform 0.9 & CD 0.1 & CD 0.3 & CD 0.5 \\

    \hline

    \multirow{15}[0]{*}{Kuzushiji-49} & ALDR-KL & 96.6(0.132) & 94.5(0.248) & \textbf{91.5(0.154)} & \textbf{48.2(1.54)} & 96.0(0.256) & \textbf{95.5(0.11)} & \textbf{54.4(5.84)} \\
          & LDR-KL & 96.6(0.06) & \textbf{94.8(0.145)} & 91.2(0.295) & 48.1(1.27) & 95.9(0.089) & 95.2(0.132) & 54.1(3.69) \\
          & CE    & 96.9(0.09) & 94.0(0.102) & 85.5(0.219) & 29.4(1.45) & 96.2(0.064) & 95.2(0.098) & 47.7(3.17) \\
          & SCE   & 96.9(0.074) & 94.4(0.186) & 85.5(0.235) & 29.1(0.836) & 96.2(0.096) & 95.2(0.17) & 49.7(2.1) \\
          & GCE   & 96.9(0.069) & 94.5(0.161) & 85.5(0.347) & 28.6(0.684) & 96.2(0.076) & 95.2(0.085) & 48.8(1.69) \\
          & TGCE  & 96.8(0.019) & 94.4(0.115) & 85.5(0.221) & 28.7(0.359) & 96.2(0.05) & 95.2(0.114) & 48.2(2.22) \\
          & WW    & \textbf{97.1(0.091)} & 90.0(0.341) & 80.1(0.236) & 26.3(1.45) & 95.5(0.452) & 94.5(0.381) & 49.8(1.22) \\
\bottomrule
    \end{tabular}}
    \vspace*{-0.2in}
\end{table*}

\begin{table*}[p]
    \centering
        \caption{Testing Top-2 Accuracy (\%) with mean and standard deviation for tabular dataset. The highest values are marked as bold.}
    \label{table:top-2}
    \resizebox{.95\textwidth}{!}{
    \begin{tabular}{l|l|l|lll|lll}
    \toprule
    \multicolumn{1}{l}{Dataset} & Loss  & Clean & Uniform 0.3 & Uniform 0.6 & Uniform 0.9 & CD 0.1 & CD 0.3 & CD 0.5 \\
    \hline

    \multirow{15}[0]{*}{Vowel} & ALDR-KL & \textbf{93.9(2.3)} & \textbf{83.0(2.67)} & \textbf{61.0(4.27)} & 23.4(5.97) & \textbf{92.7(1.96)} & \textbf{93.3(2.27)} & \textbf{86.3(1.98)} \\
          & LDR-KL & 93.5(1.76) & 81.4(3.29) & 60.6(4.52) & 23.6(5.44) & 91.9(1.11) & 90.1(1.74) & 86.1(1.62) \\
          & CE    & 88.7(1.49) & 77.8(2.12) & 57.8(1.34) & 19.8(5.37) & 90.3(2.83) & 85.9(3.5) & 82.6(3.52) \\
          & SCE   & 88.5(1.51) & 77.4(2.27) & 57.2(6.6) & 21.0(3.52) & 88.5(4.07) & 86.1(3.09) & 80.2(2.08) \\
          & GCE   & 88.7(1.18) & 79.0(1.49) & 58.0(2.6) & 20.4(6.17) & 88.9(4.74) & 86.3(3.42) & 80.2(2.75) \\
          & TGCE  & 88.5(1.03) & 79.0(1.49) & 57.6(2.3) & 20.4(6.17) & 88.3(4.5) & 86.3(3.42) & 79.8(3.06) \\
          & WW    & 84.6(2.06) & 75.8(3.78) & 55.2(4.36) & \textbf{25.5(3.91)} & 86.7(3.09) & 82.8(3.44) & 76.0(3.34) \\
\bottomrule
    \end{tabular}}
    \vspace*{-0.2in}
\end{table*}

\begin{table*}[p]
    \centering
        \caption{Testing Top-2 Accuracy (\%) with mean and standard deviation for image dataset. The highest values are marked as bold.}
    \label{table:top-22}
    \resizebox{1.\textwidth}{!}{
    \begin{tabular}{l|l|l|lll|lll}
     \toprule
    \multicolumn{1}{l}{Dataset} & Loss  & Clean & Uniform 0.3 & Uniform 0.6 & Uniform 0.9 & CD 0.1 & CD 0.3 & CD 0.5 \\

    \hline

    \multirow{15}[0]{*}{Kuzushiji-49} & ALDR-KL & 98.2(0.072) & \textbf{97.1(0.122)} & 94.1(1.39) & \textbf{56.2(1.17)} & 97.2(0.068) & 96.5(0.306) & 95.8(0.159) \\
          & LDR-KL & 98.3(0.022) & 97.0(0.144) & \textbf{95.0(0.136)} & 55.8(1.09) & 97.2(0.132) & 96.6(0.079) & 95.9(0.207) \\
          & CE    & 98.4(0.105) & 96.9(0.091) & 89.3(0.318) & 41.5(2.13) & 97.4(0.055) & 96.7(0.136) & 96.1(0.112) \\
          & SCE   & 98.4(0.133) & 97.1(0.127) & 89.1(0.335) & 41.0(1.29) & 97.4(0.034) & 96.7(0.069) & 96.1(0.1) \\
          & GCE   & 98.5(0.089) & 97.2(0.104) & 89.1(0.315) & 38.5(1.23) & 97.4(0.057) & 96.7(0.062) & 96.1(0.09) \\
          & TGCE  & 98.4(0.047) & 97.0(0.261) & 89.2(0.243) & 39.9(1.58) & 97.4(0.071) & 96.7(0.107) & 96.1(0.081) \\
          & WW    & \textbf{98.6(0.129)} & 96.1(0.127) & 87.7(0.35) & 38.5(1.47) & \textbf{97.6(0.055)} & \textbf{97.1(0.067)} & \textbf{96.8(0.079)} \\
\bottomrule
    \end{tabular}}
    \vspace*{-0.2in}
\end{table*}

\begin{table*}[p]
    \centering
        \caption{Testing Top-3 Accuracy (\%) with mean and standard deviation for tabular dataset. The highest values are marked as bold.}
    \label{table:top-3}
    \resizebox{.95\textwidth}{!}{
    \begin{tabular}{l|l|l|lll|lll}
    \toprule
    \multicolumn{1}{l}{Dataset} & Loss  & Clean & Uniform 0.3 & Uniform 0.6 & Uniform 0.9 & CD 0.1 & CD 0.3 & CD 0.5 \\
    \hline

    \multirow{15}[0]{*}{Vowel} & ALDR-KL & \textbf{97.2(0.404)} & \textbf{89.9(2.02)} & \textbf{71.9(5.73)} & \textbf{33.9(8.04)} & \textbf{97.8(1.34)} & \textbf{97.8(1.96)} & \textbf{93.5(1.51)} \\
          & LDR-KL & \textbf{97.2(1.18)} & 88.3(2.08) & 63.4(4.44) & 32.9(7.58) & 96.4(1.21) & 96.6(1.03) & 92.9(2.12) \\
          & CE    & 94.7(2.25) & 89.3(2.9) & 72.1(4.12) & 33.7(3.04) & 97.4(0.495) & 94.7(2.34) & 91.3(1.03) \\
          & SCE   & 96.6(1.03) & 89.7(1.62) & 69.5(5.21) & 29.5(4.62) & 95.6(1.87) & 95.0(1.69) & 90.9(2.63) \\
          & GCE   & 95.0(1.43) & 88.9(2.12) & 73.3(2.36) & 28.7(5.44) & 94.3(1.87) & 95.6(1.03) & 90.5(1.37) \\
          & TGCE  & 95.4(1.87) & 87.1(3.75) & 73.9(2.16) & 28.7(5.44) & 94.3(1.87) & 95.6(1.03) & 90.5(1.37) \\
          & WW    & 95.8(2.06) & 88.5(0.495) & 68.9(7.27) & \textbf{33.9(4.07)} & 95.6(1.51) & 93.3(2.97) & 91.1(2.06) \\
\bottomrule
    \end{tabular}}
    \vspace*{-0.2in}
\end{table*}

\begin{table*}[p]
    \centering
        \caption{Testing Top-3 Accuracy (\%) with mean and standard deviation for image dataset. The highest values are marked as bold.}
    \label{table:top-33}
    \resizebox{1.\textwidth}{!}{
    \begin{tabular}{l|l|l|lll|lll}
     \toprule
    \multicolumn{1}{l}{Dataset} & Loss  & Clean & Uniform 0.3 & Uniform 0.6 & Uniform 0.9 & CD 0.1 & CD 0.3 & CD 0.5 \\

    \hline

    \multirow{15}[0]{*}{Kuzushiji-49} & ALDR-KL & 98.6(0.147) & \textbf{97.8(0.094)} & 95.2(0.551) & \textbf{60.8(1.21)} & 98.1(0.053) & 97.8(0.062) & 96.9(0.187) \\
          & LDR-KL & 98.6(0.048) & 97.7(0.071) & \textbf{95.7(0.729)} & 60.2(1.29) & 98.1(0.063) & 97.7(0.057) & 96.9(0.181) \\
          & CE    & 98.8(0.074) & 97.7(0.116) & 91.0(0.304) & 47.0(2.47) & 98.1(0.059) & 97.7(0.068) & 97.1(0.15) \\
          & SCE   & 98.9(0.113) & 97.7(0.172) & 90.5(0.534) & 46.1(2.06) & 98.1(0.042) & 97.7(0.025) & 97.0(0.111) \\
          & GCE   & 98.8(0.142) & 97.8(0.09) & 90.6(0.285) & 44.7(1.8) & 98.0(0.045) & 97.7(0.05) & 97.0(0.114) \\
          & TGCE  & 98.7(0.126) & 97.8(0.102) & 90.5(0.216) & 45.6(1.1) & 98.0(0.044) & 97.7(0.064) & 97.0(0.135) \\
          & WW    & \textbf{99.0(0.105)} & 97.1(0.092) & 91.5(0.487) & 45.2(1.23) & \textbf{98.2(0.066)} & \textbf{98.1(0.081)} & \textbf{97.7(0.108)} \\
\bottomrule
    \end{tabular}}
    \vspace*{-0.2in}
\end{table*}

\begin{table*}[p]
    \centering
        \caption{Testing Top-4 Accuracy (\%) with mean and standard deviation for tabular dataset. The highest values are marked as bold.}
    \label{table:top-4}
    \resizebox{.95\textwidth}{!}{
    \begin{tabular}{l|l|l|lll|lll}
    \toprule
    \multicolumn{1}{l}{Dataset} & Loss  & Clean & Uniform 0.3 & Uniform 0.6 & Uniform 0.9 & CD 0.1 & CD 0.3 & CD 0.5 \\
    \hline

    \multirow{15}[0]{*}{Vowel} & ALDR-KL & \textbf{98.8(0.99)} & 91.7(0.99) & 79.8(3.06) & 37.8(6.25) & \textbf{98.8(0.756)} & \textbf{99.0(1.11)} & \textbf{98.2(0.756)} \\
          & LDR-KL & 97.6(0.808) & 92.1(2.34) & 73.5(6.8) & 39.2(4.11) & 97.8(1.18) & 97.6(1.51) & 96.8(2.96) \\
          & CE    & 97.8(0.756) & 93.3(1.87) & 78.2(2.9) & 36.6(6.56) & 98.2(1.74) & 97.8(0.756) & 95.8(0.756) \\
          & SCE   & 97.6(0.808) & 93.3(1.64) & 75.6(6.86) & 37.2(4.84) & 98.2(0.756) & 97.4(1.51) & 94.7(1.18) \\
          & GCE   & 97.8(0.756) & \textbf{93.9(1.69)} & 80.6(2.74) & 37.2(5.32) & 97.6(1.03) & 97.8(0.404) & 95.4(1.21) \\
          & TGCE  & 97.8(0.756) & 93.1(1.49) & 80.0(3.34) & 37.2(5.32) & 98.0(0.639) & 97.8(0.404) & 95.2(1.18) \\
          & WW    & 98.0(0.0) & 93.3(1.37) & 78.2(3.65) & 39.2(7.04) & 97.2(0.99) & 97.6(1.37) & 96.4(1.03) \\
\bottomrule
    \end{tabular}}
    \vspace*{-0.2in}
\end{table*}

\begin{table*}[p]
    \centering
        \caption{Testing Top-4 Accuracy (\%) with mean and standard deviation for image dataset. The highest values are marked as bold.}
    \label{table:top-44}
    \resizebox{1.\textwidth}{!}{
    \begin{tabular}{l|l|l|lll|lll}
     \toprule
    \multicolumn{1}{l}{Dataset} & Loss  & Clean & Uniform 0.3 & Uniform 0.6 & Uniform 0.9 & CD 0.1 & CD 0.3 & CD 0.5 \\

    \hline

    \multirow{15}[0]{*}{Kuzushiji-49} & ALDR-KL & 98.8(0.15) & \textbf{98.2(0.059)} & \textbf{95.8(0.238)} & \textbf{64.1(1.81)} & 98.4(0.071) & 98.2(0.039) & 97.9(0.038) \\
          & LDR-KL & 98.8(0.035) & 97.9(0.314) & 95.7(0.862) & 63.7(0.856) & 98.4(0.029) & 98.1(0.04) & 97.9(0.097) \\
          & CE    & 99.1(0.102) & 98.0(0.114) & 92.3(0.062) & 51.0(2.52) & 98.4(0.044) & 98.1(0.08) & 97.8(0.062) \\
          & SCE   & 99.0(0.123) & 98.1(0.077) & 91.7(0.734) & 50.1(2.15) & 98.4(0.049) & 98.1(0.029) & 97.8(0.061) \\
          & GCE   & 99.0(0.138) & 98.2(0.05) & 91.5(0.295) & 48.9(1.81) & 98.4(0.021) & 98.0(0.041) & 97.8(0.052) \\
          & TGCE  & 99.0(0.143) & 97.9(0.321) & 91.4(0.28) & 49.9(0.791) & 98.4(0.035) & 98.0(0.058) & 97.8(0.101) \\
          & WW    & \textbf{99.1(0.138)} & 97.6(0.108) & 93.5(0.487) & 50.0(1.34) & \textbf{98.5(0.096)} & \textbf{98.4(0.065)} & \textbf{98.4(0.12)} \\
\bottomrule
    \end{tabular}}
    \vspace*{-0.2in}
\end{table*}

\begin{table*}[p]
    \centering
        \caption{Testing Top-5 Accuracy (\%) with mean and standard deviation for tabular dataset. The highest values are marked as bold.}
    \label{table:top-5}
    \resizebox{.95\textwidth}{!}{
    \begin{tabular}{l|l|l|lll|lll}
    \toprule
    \multicolumn{1}{l}{Dataset} & Loss  & Clean & Uniform 0.3 & Uniform 0.6 & Uniform 0.9 & CD 0.1 & CD 0.3 & CD 0.5 \\
    \hline

    \multirow{15}[0]{*}{Vowel} & ALDR-KL & 98.6(0.808) & 94.7(0.404) & 85.1(4.49) & 44.2(3.75) & 99.2(0.756) & \textbf{99.2(0.756)} & \textbf{98.8(0.404)} \\
          & LDR-KL & 98.6(0.808) & 94.9(2.12) & 81.0(2.59) & 47.1(4.22) & 98.8(0.404) & 98.8(1.18) & 98.0(0.639) \\
          & CE    & 98.8(1.18) & 95.8(1.62) & 85.9(3.67) & 42.0(5.21) & 98.6(0.808) & 98.4(0.495) & 97.2(0.756) \\
          & SCE   & 98.8(0.756) & 95.4(2.6) & 83.2(4.59) & 47.7(4.84) & \textbf{99.4(0.495)} & 98.4(0.495) & 96.8(0.404) \\
          & GCE   & 98.8(0.756) & 95.4(2.68) & 85.7(3.16) & 46.9(4.07) & 98.4(0.495) & 97.6(1.03) & 97.2(0.404) \\
          & TGCE  & 98.8(0.756) & 95.4(2.68) & 86.5(3.23) & 46.9(4.07) & 98.6(0.495) & 97.2(1.18) & 97.2(0.404) \\
          & WW    & 98.8(0.756) & \textbf{96.0(1.43)} & \textbf{86.7(1.74)} & 48.3(6.4) & 98.2(0.756) & 98.8(0.99) & 98.2(0.99) \\
\bottomrule
    \end{tabular}}
    \vspace*{-0.2in}
\end{table*}

\begin{table*}[p]
    \centering
        \caption{Testing Top-5 Accuracy (\%) with mean and standard deviation for image dataset. The highest values are marked as bold.}
    \label{table:top-55}
    \resizebox{1.\textwidth}{!}{
    \begin{tabular}{l|l|l|lll|lll}
     \toprule
    \multicolumn{1}{l}{Dataset} & Loss  & Clean & Uniform 0.3 & Uniform 0.6 & Uniform 0.9 & CD 0.1 & CD 0.3 & CD 0.5 \\

    \hline

    \multirow{15}[0]{*}{Kuzushiji-49} & ALDR-KL & 98.9(0.129) & \textbf{98.4(0.059)} & \textbf{96.5(0.143)} & \textbf{67.3(1.97)} & \textbf{98.6(0.065)} & 98.4(0.037) & 98.1(0.151) \\
          & LDR-KL & 98.9(0.033) & 98.1(0.352) & \textbf{96.5(0.571)} & 66.4(1.01) & 98.5(0.052) & 98.4(0.06) & 98.2(0.053) \\
          & CE    & 99.1(0.096) & 98.3(0.063) & 93.2(0.226) & 54.4(1.57) & 98.5(0.055) & 98.4(0.056) & 98.2(0.051) \\
          & SCE   & 99.1(0.138) & 98.1(0.464) & 93.0(0.377) & 53.0(1.94) & 98.6(0.042) & 98.4(0.048) & 98.1(0.046) \\
          & GCE   & 99.1(0.148) & \textbf{98.4(0.042)} & 92.4(0.31) & 51.9(2.02) & 98.5(0.024) & 98.3(0.07) & 98.1(0.045) \\
          & TGCE  & 99.0(0.112) & 98.0(0.341) & 92.3(0.329) & 53.1(0.717) & 98.5(0.046) & 98.3(0.053) & 98.1(0.079) \\
          & WW    & \textbf{99.2(0.157)} & 97.9(0.065) & 94.8(0.308) & 53.7(0.948) & 98.7(0.137) & \textbf{98.7(0.1)} & \textbf{98.6(0.152)} \\
\bottomrule
    \end{tabular}}
    \vspace*{-0.2in}
\end{table*}

\subsection{Setups for Mini-Webvision Experiments}\label{sec:webvision-setup}

We adopt the noisy Google resized version from Webvision version 1.0 as training dataset (only for the first 50 classes)~\cite{DBLP:journals/corr/abs-1708-02862} and do evaluation on the ILSVRC 2012 validation data (only for the first 50 classes). The total training epochs is set as 250. The weight decay is set as 3e-5. SGD optimizer is utilized with nesterov momentum as 0.9, initial learning rate as 0.4 that is epoch-wise decreased by 0.97. Typical data augmentations including color jittering, random width/height shift and random horizontal flip are applied. For our proposed LDR-KL and ALDR-KL, we apply SoftPlus function to make model output $\f$ to be positive. The results are presented at Table~\ref{tab:webvision}.

\subsection{Leaderboard with Standard Deviation}
In addition to the averaged rank for the leaderboard at Table~\ref{tab:ranks}, we provide the leaderboard with standard deviation at Table~\ref{tab:ranks+std}.

\begin{table}[h]
    \centering
\caption{leaderboard for comparing 15 different loss functions on 7 datasets in 6 noise and 1 clean settings. The reported numbers are averaged ranks of performance. The smaller the better.}
    \label{tab:ranks+std}
    \begin{tabular}{c|ccccc|c}
    \toprule
       Loss Function & top-1 & top-2& top-3 & top-4& top-5  &  overall \\
       \hline

    ALDR-KL & \textbf{2.163(1.754)} & \textbf{2.469(2.149)} & \textbf{2.245(1.933)} & \textbf{2.224(1.951)} & \textbf{2.673(2.543)} & \textbf{2.355(2.092)} \\
    LDR-KL & 2.449(1.819) & 2.571(1.841) & 2.776(2.043) & 3.0(2.138) & 2.959(2.109) & 2.751(2.006) \\
    CE    & 4.714(1.852) & 4.592(2.465) & 4.224(2.131) & 4.449(2.148) & 4.388(2.184) & 4.473(2.171) \\
    SCE   & 4.959(1.958) & 4.816(1.769) & 4.449(1.63) & 4.918(2.174) & 4.469(1.63) & 4.722(1.857) \\
    GCE   & 5.796(2.39) & 5.347(2.462) & 6.061(2.18) & 5.551(2.167) & 5.571(2.241) & 5.665(2.304) \\
    TGCE  & 6.204(2.204) & 6.204(1.948) & 6.286(2.109) & 6.204(1.84) & 6.306(1.929) & 6.241(2.011) \\
    WW    & 7.673(2.198) & 6.592(2.465) & 6.061(2.377) & 5.898(2.35) & 5.469(2.492) & 6.339(2.496) \\
    JS    & 7.816(2.946) & 7.837(2.637) & 7.857(2.295) & 7.98(2.114) & 8.245(2.005) & 7.947(2.43) \\
    CS    & 8.653(4.943) & 8.857(4.699) & 8.755(4.506) & 8.551(4.764) & 8.714(4.84) & 8.706(4.754) \\
    RLL   & 10.184(2.569) & 10.224(1.951) & 10.224(2.41) & 10.49(2.34) & 10.408(2.166) & 10.306(2.3) \\
    NCE+RCE  & 9.694(3.065) & 10.633(2.686) & 10.878(2.553) & 10.612(2.798) & 10.714(2.483) & 10.506(2.756) \\
    NCE+AUL  & 10.449(2.167) & 10.98(1.436) & 11.102(1.297) & 11.245(1.17) & 11.163(1.633) & 10.988(1.605) \\
    NCE+AGCE  & 11.735(1.613) & 11.837(1.448) & 11.878(1.56) & 11.51(1.704) & 11.673(1.038) & 11.727(1.496) \\
    MSE   & 12.776(2.216) & 12.612(2.64) & 12.735(2.126) & 12.939(1.921) & 12.878(2.067) & 12.788(2.21) \\
    MAE   & 14.735(0.563) & 14.429(1.666) & 14.469(1.739) & 14.429(2.06) & 14.367(2.087) & 14.486(1.72) \\
    \bottomrule
    \end{tabular}%
 
\end{table}%

\subsection{Convergence for ALDR-KL Optimization}
We show the practical convergence curves (mean value with standard deviation band) for Algorithm 1 for ALDR-KL loss at this section. `CD' is short for class-dependent noise; `U' is short for uniform noise; `$c$' is the margin value; initial learning rate $\eta_0$ is tuned in $\{1e-1,1e-2,1e-3\}$. The convergence curves on ALOI dataset is shown in Figure~\ref{fig:aldr-aloi-cvg}; the convergence curves on News20 dataset is shown in Figure~\ref{fig:aldr-news20-cvg}; the convergence curves on Letter dataset is shown in Figure~\ref{fig:aldr-letter-cvg}; the convergence curves on Vowel dataset is shown in Figure~\ref{fig:aldr-vowel-cvg}.

\begin{figure}[p]
\centering

\subfigure[CD(0.1), $c=0.1$]{\includegraphics[scale=0.1]{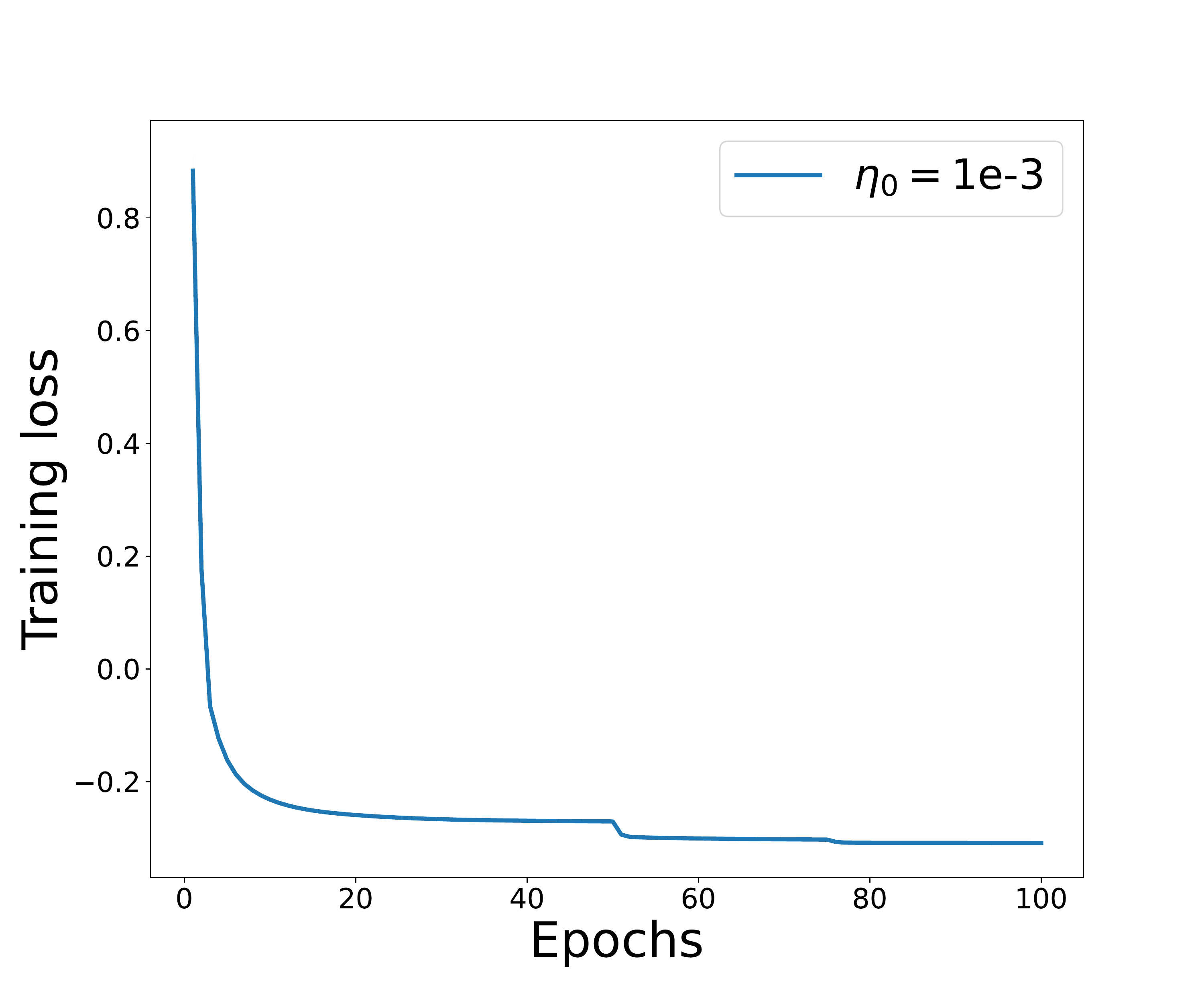}}
\subfigure[CD(0.1), $c=1$]{\includegraphics[scale=0.1]{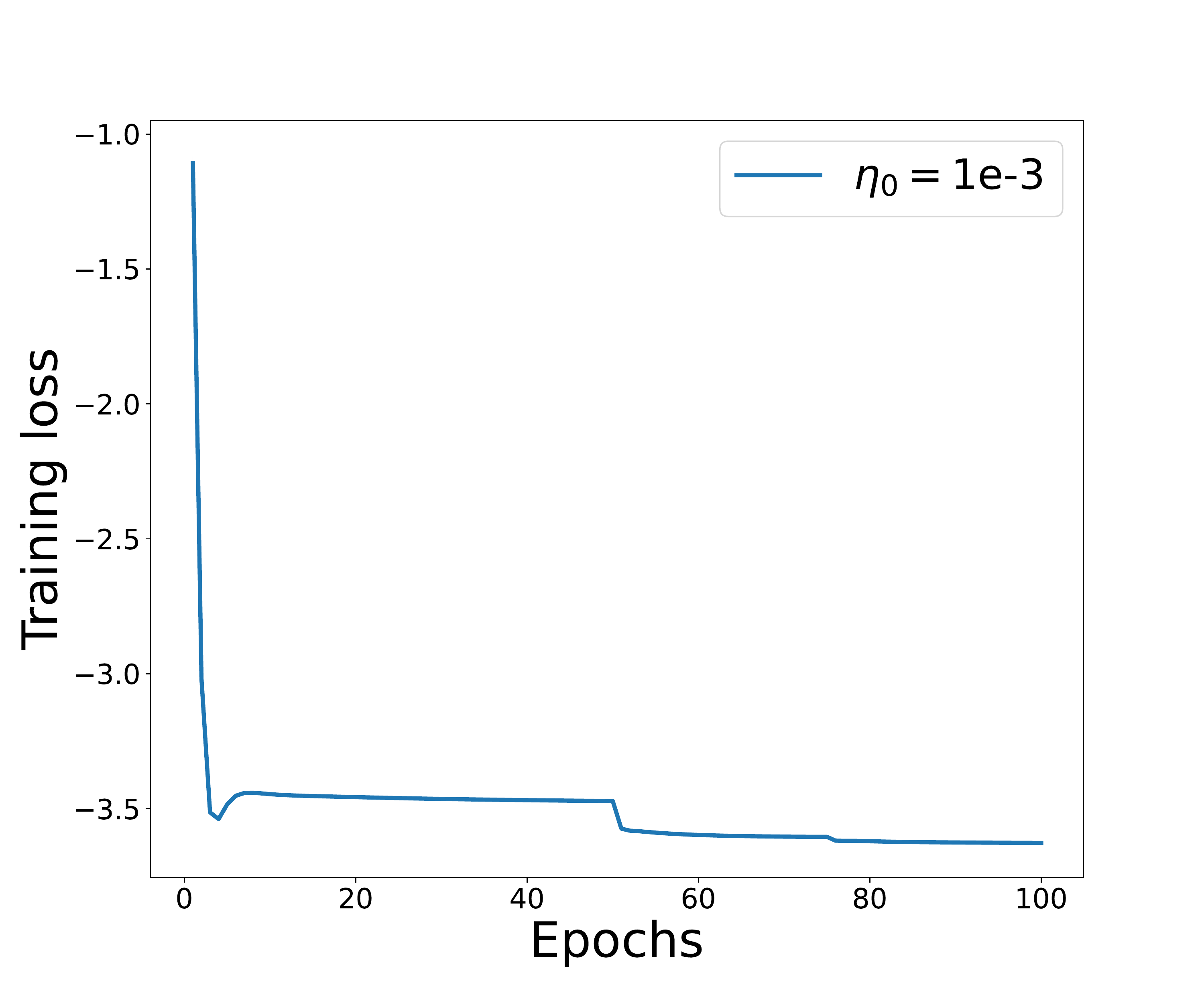}}
\subfigure[CD(0.1), $c=10$]{\includegraphics[scale=0.1]{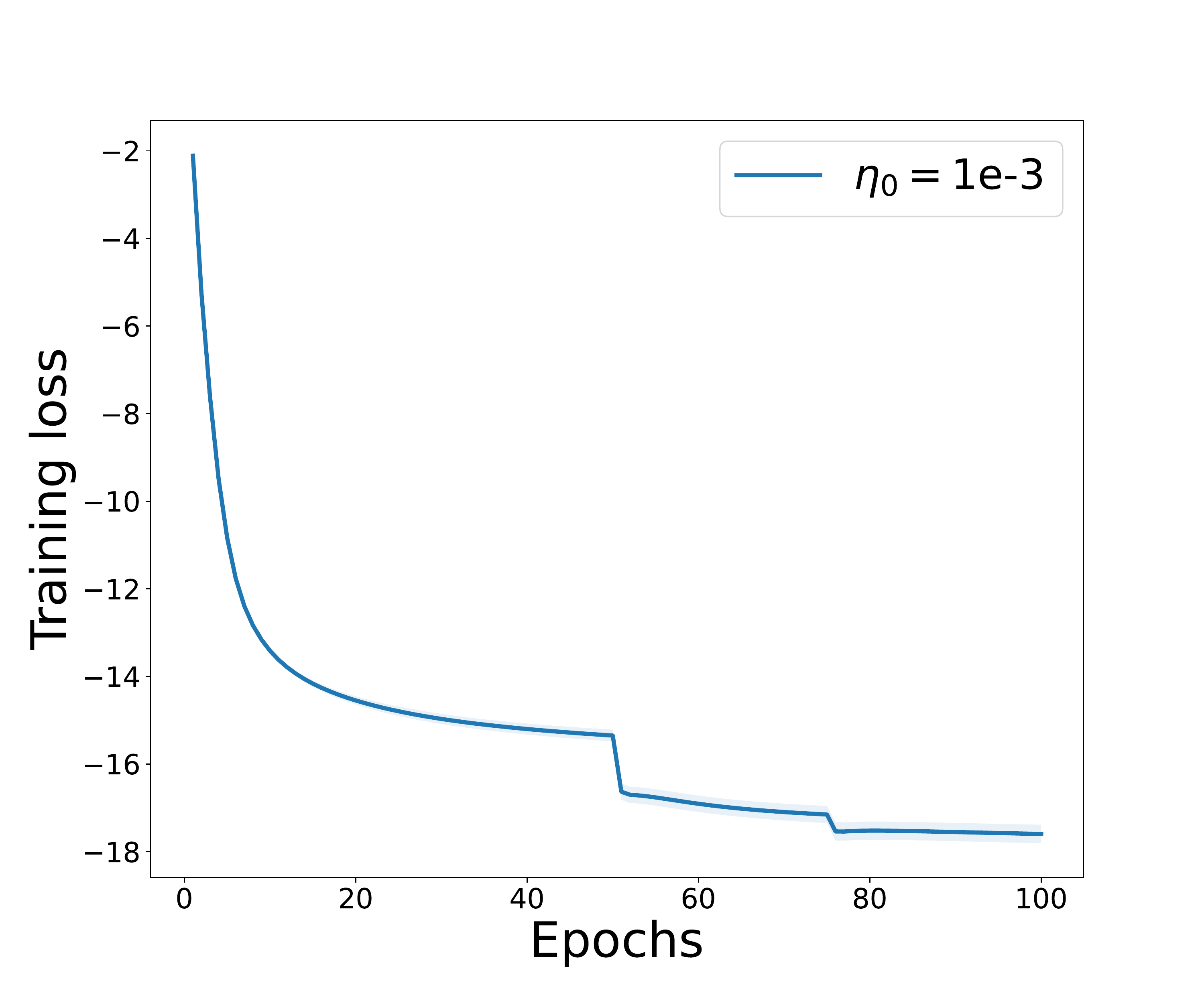}}

\subfigure[CD(0.3), $c=0.1$]{\includegraphics[scale=0.1]{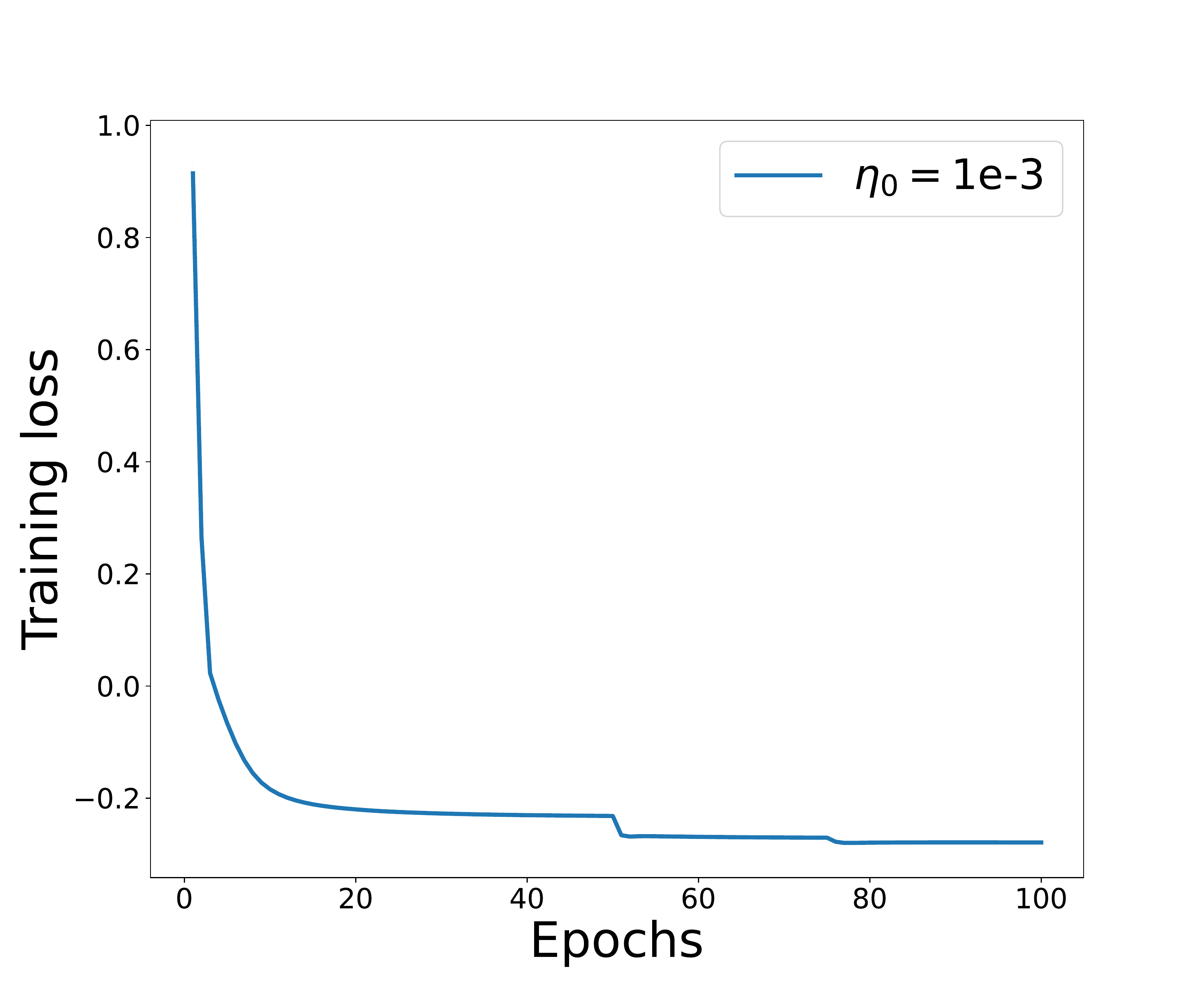}}
\subfigure[CD(0.3), $c=1$]{\includegraphics[scale=0.1]{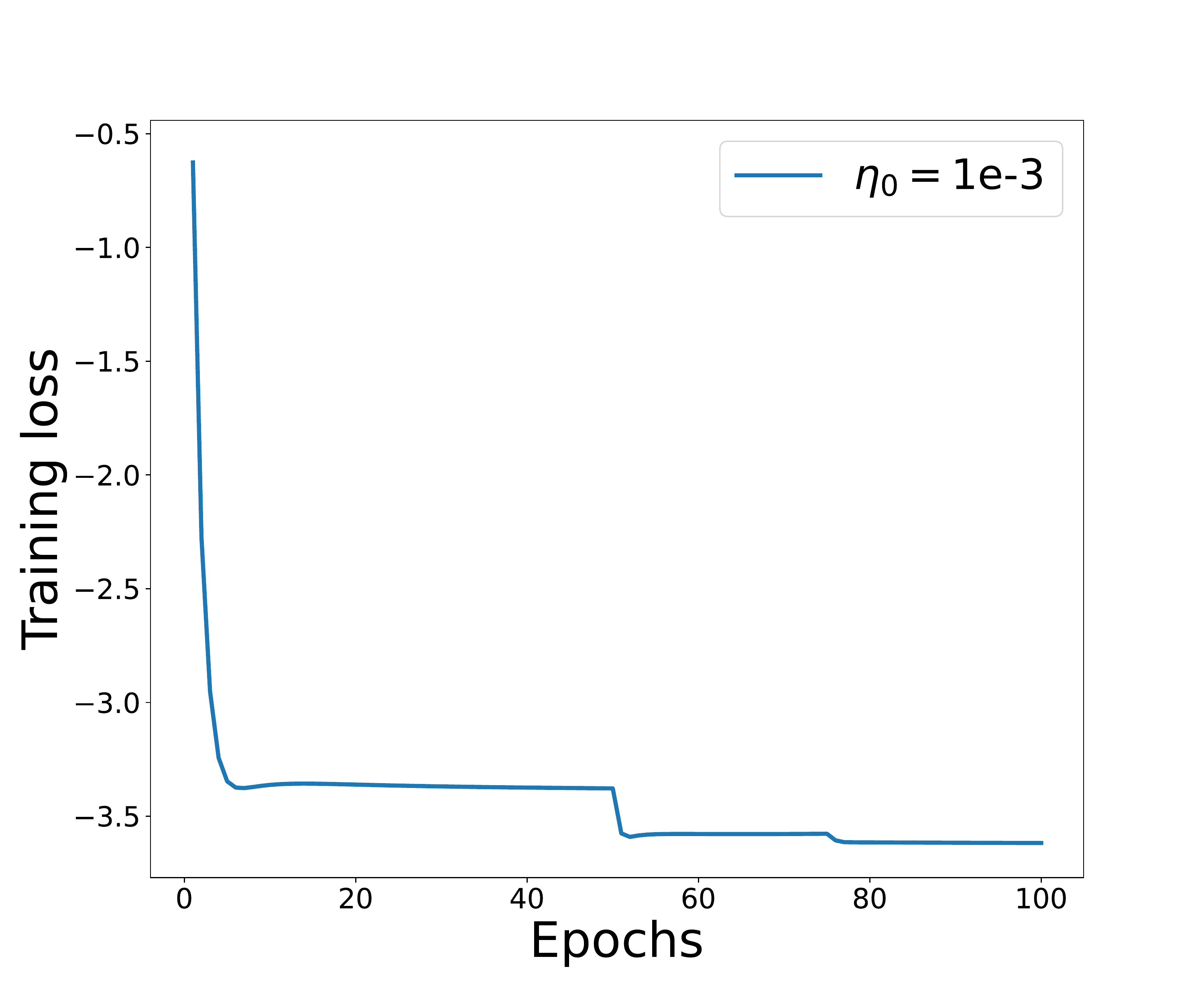}}
\subfigure[CD(0.3), $c=10$]{\includegraphics[scale=0.1]{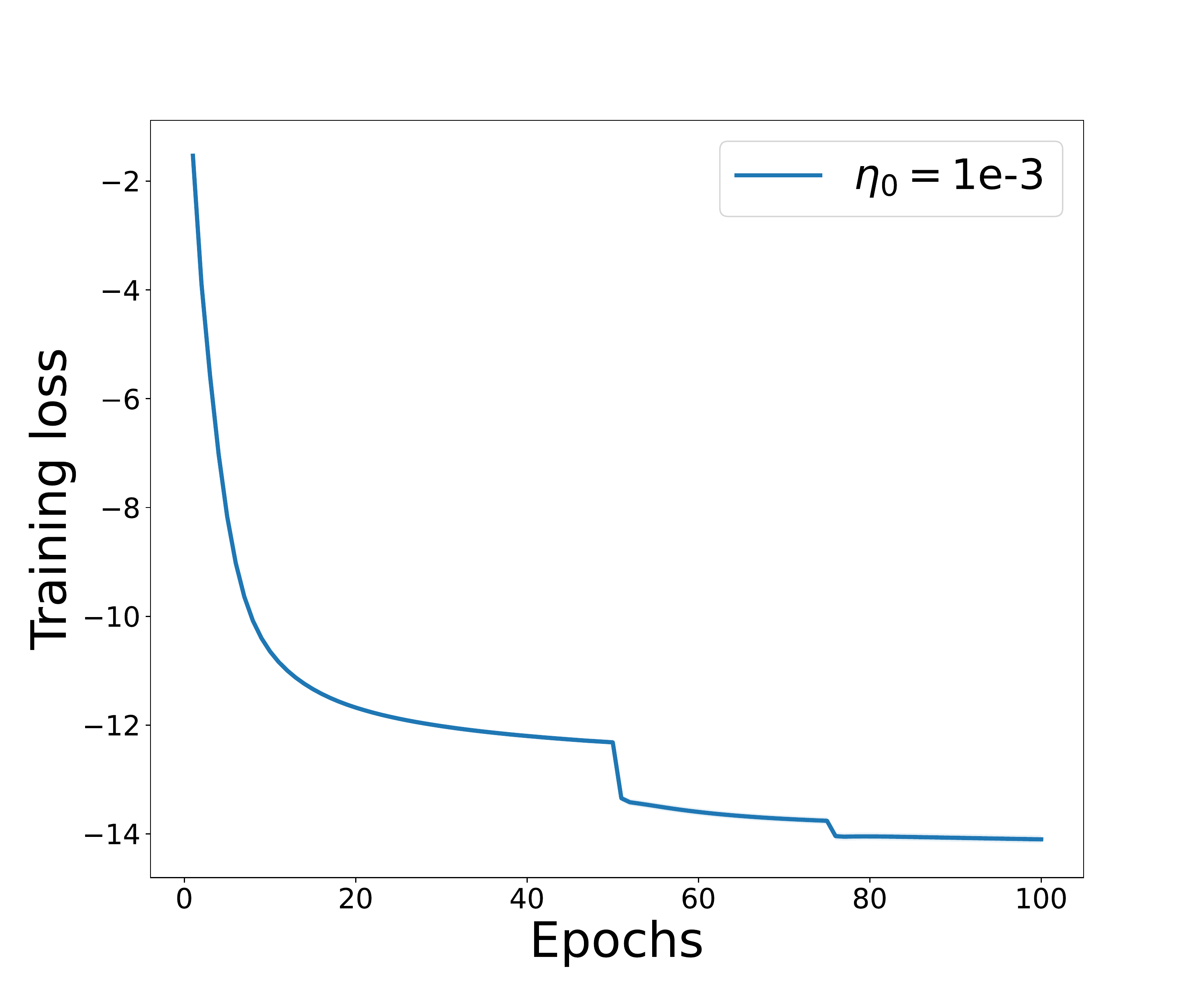}}

\subfigure[CD(0.5), $c=0.1$]{\includegraphics[scale=0.1]{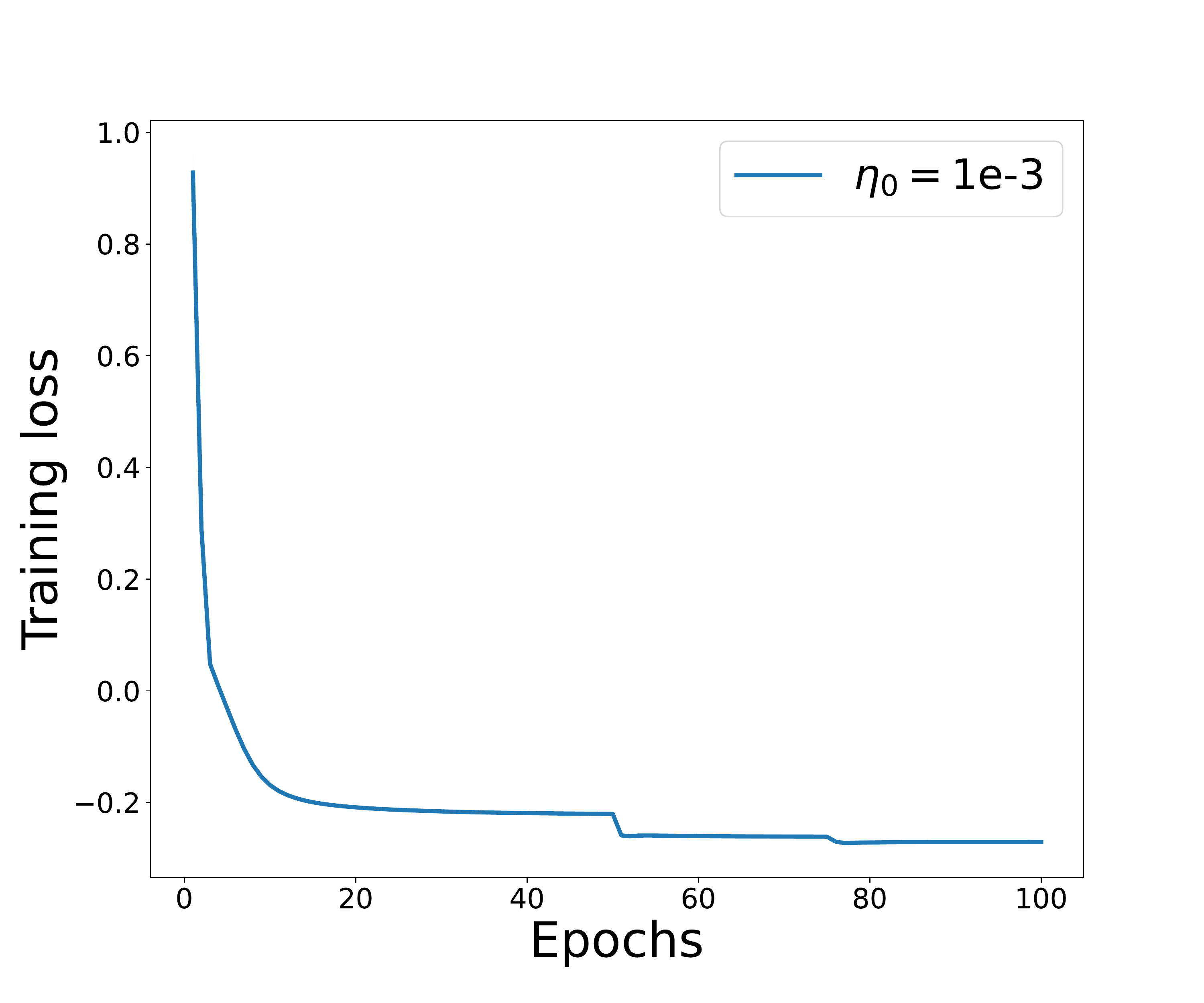}}
\subfigure[CD(0.5), $c=1$]{\includegraphics[scale=0.1]{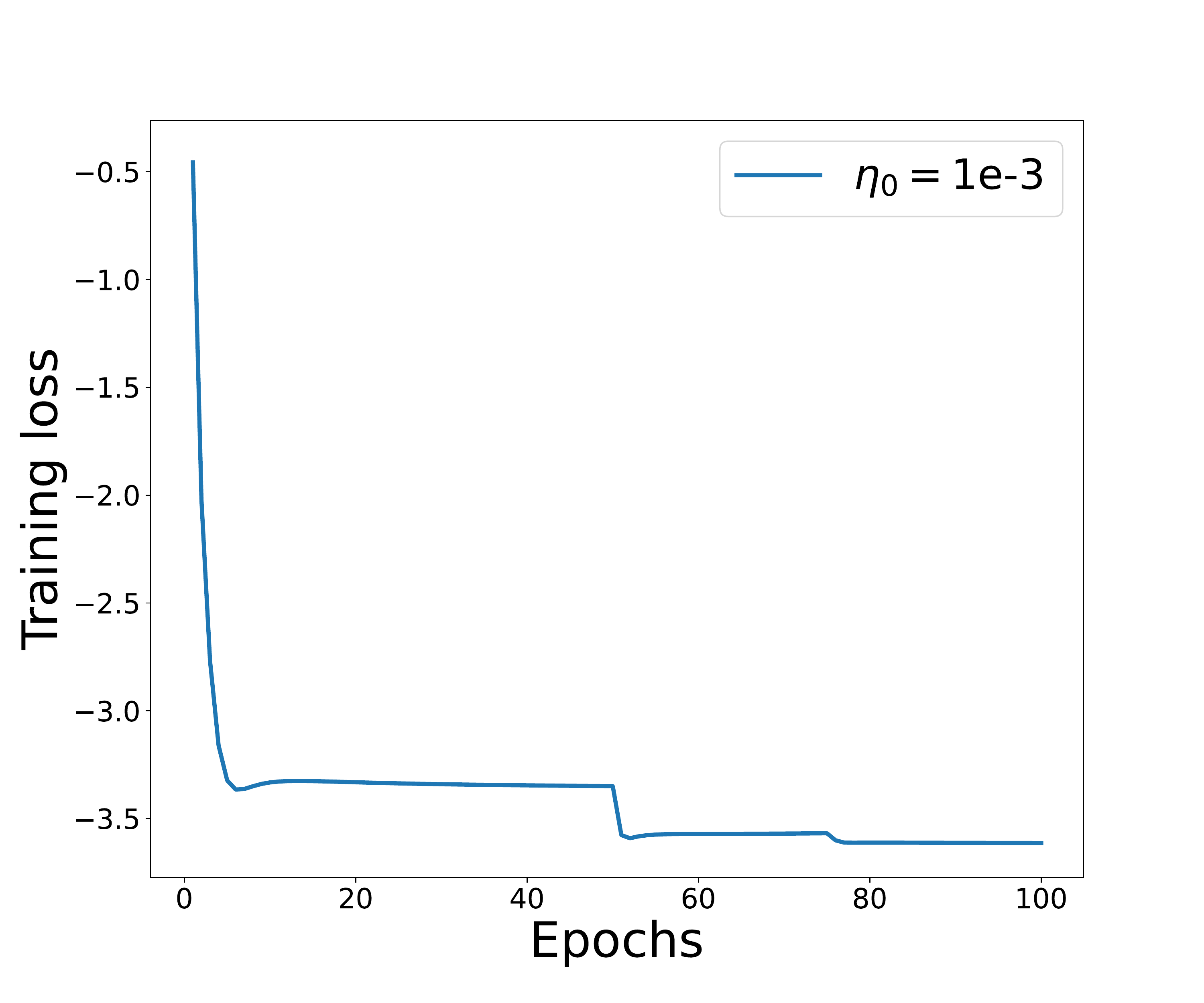}}
\subfigure[CD(0.5), $c=10$]{\includegraphics[scale=0.1]{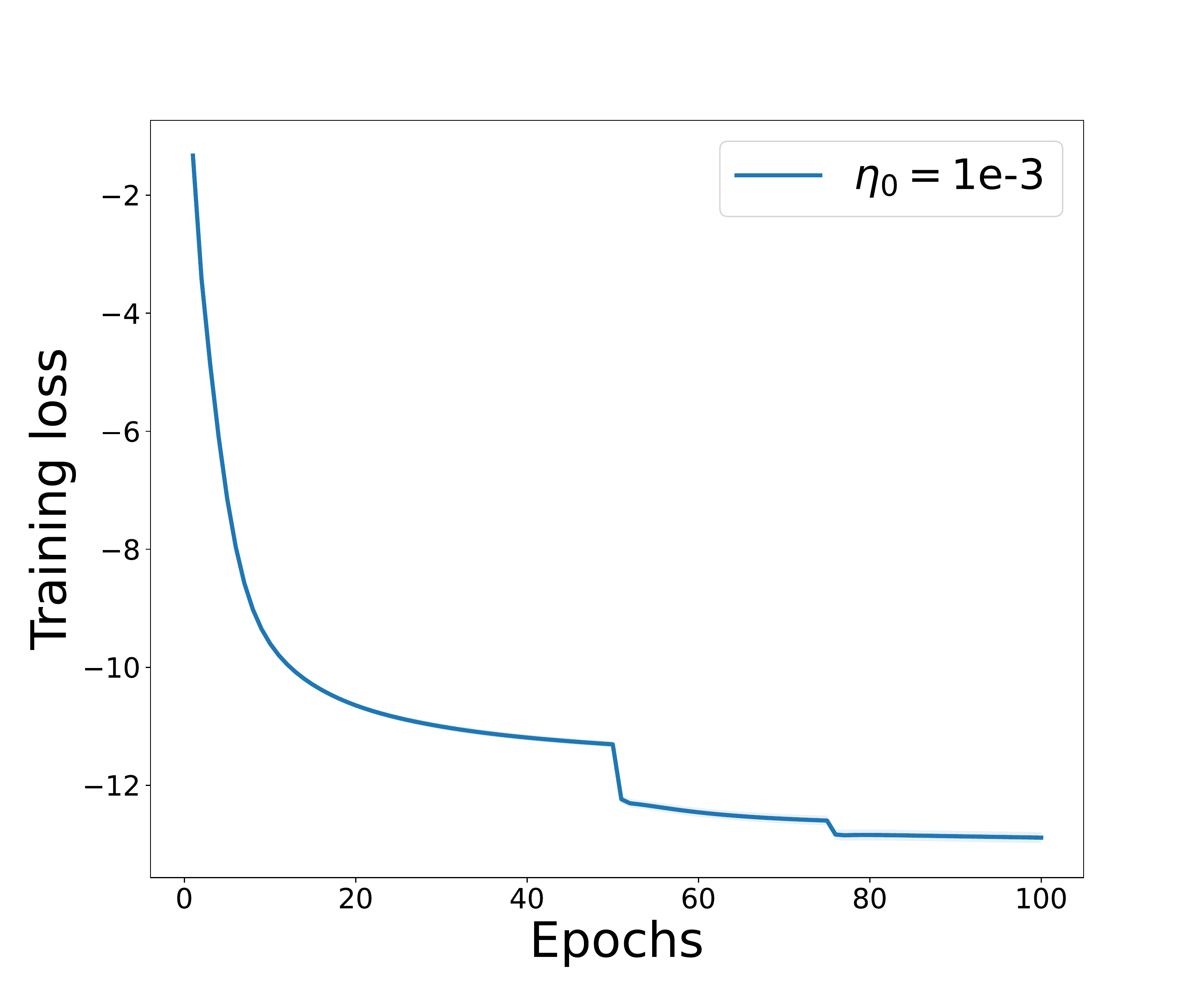}}

\subfigure[U(0.3), $c=0.1$]{\includegraphics[scale=0.1]{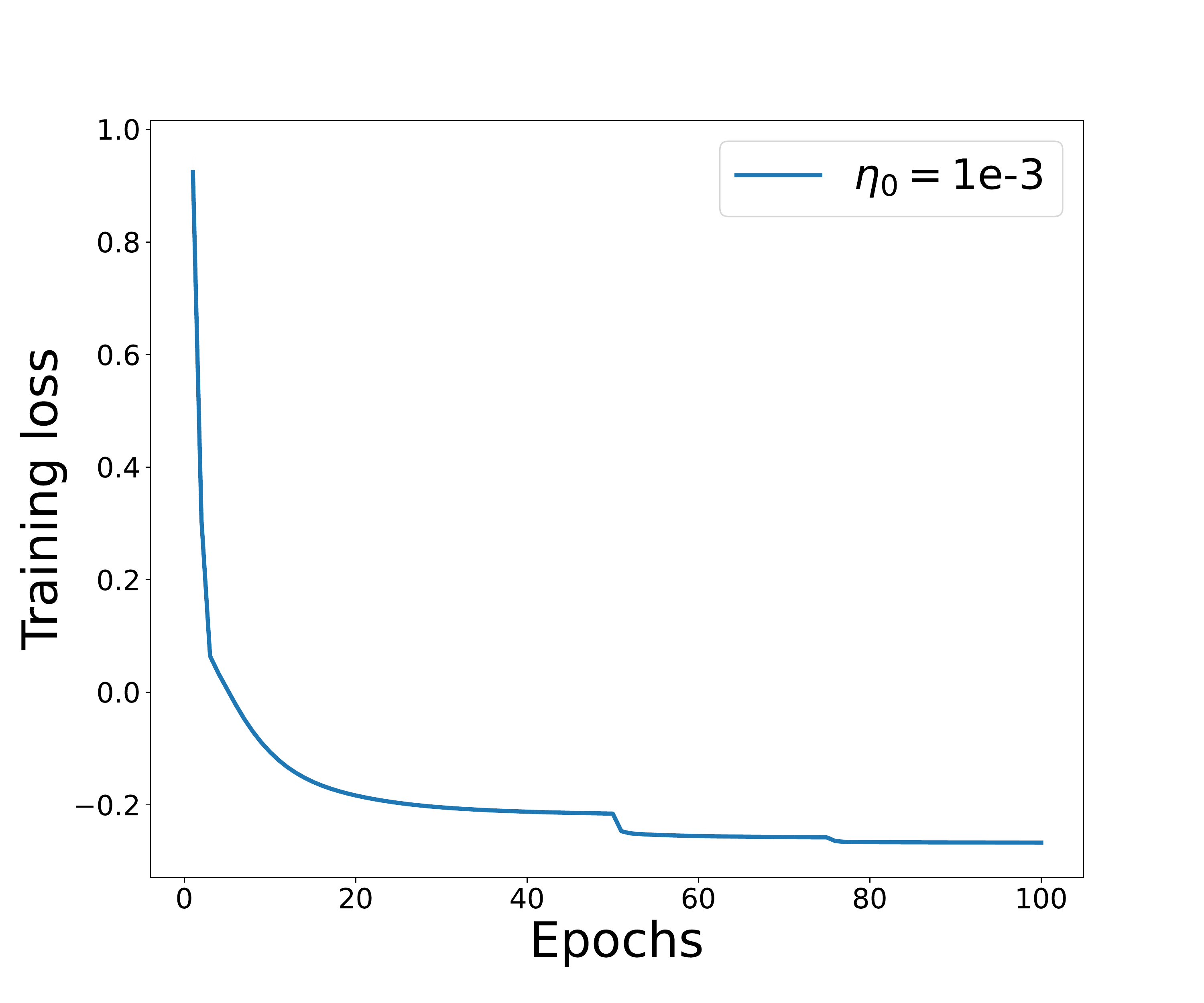}}
\subfigure[U(0.3), $c=1$]{\includegraphics[scale=0.1]{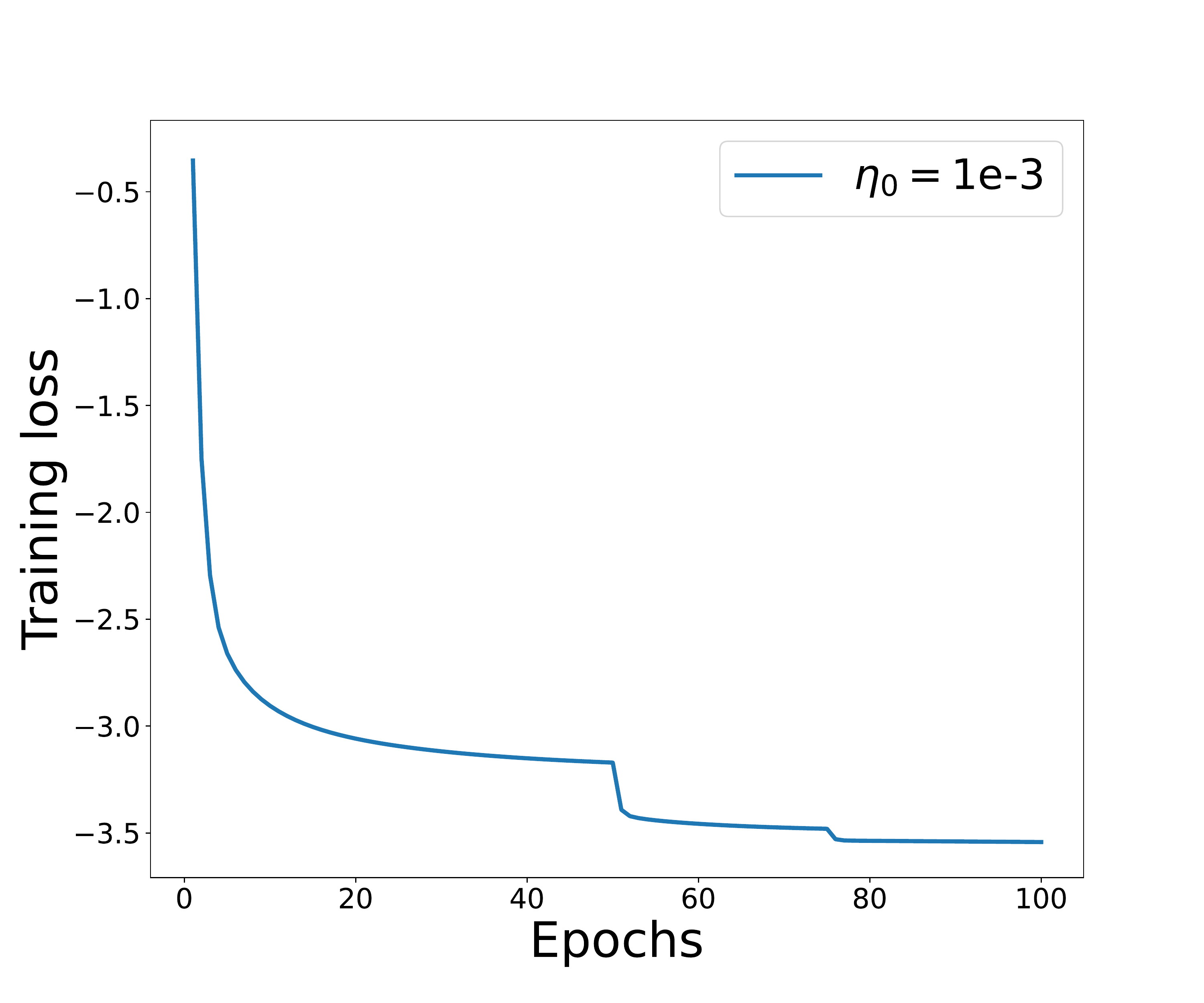}}
\subfigure[U(0.3), $c=10$]{\includegraphics[scale=0.1]{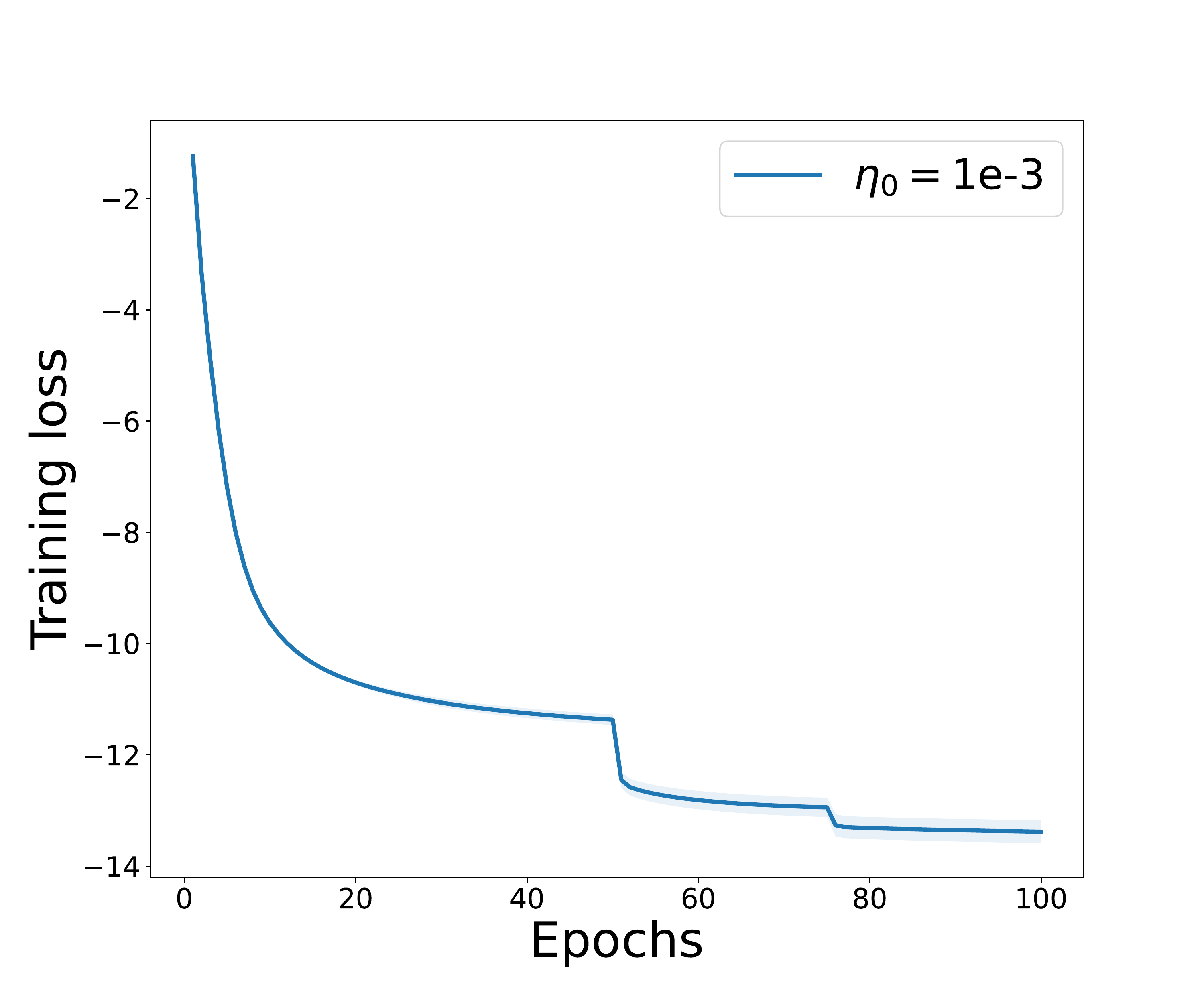}}

\subfigure[U(0.6), $c=0.1$]{\includegraphics[scale=0.1]{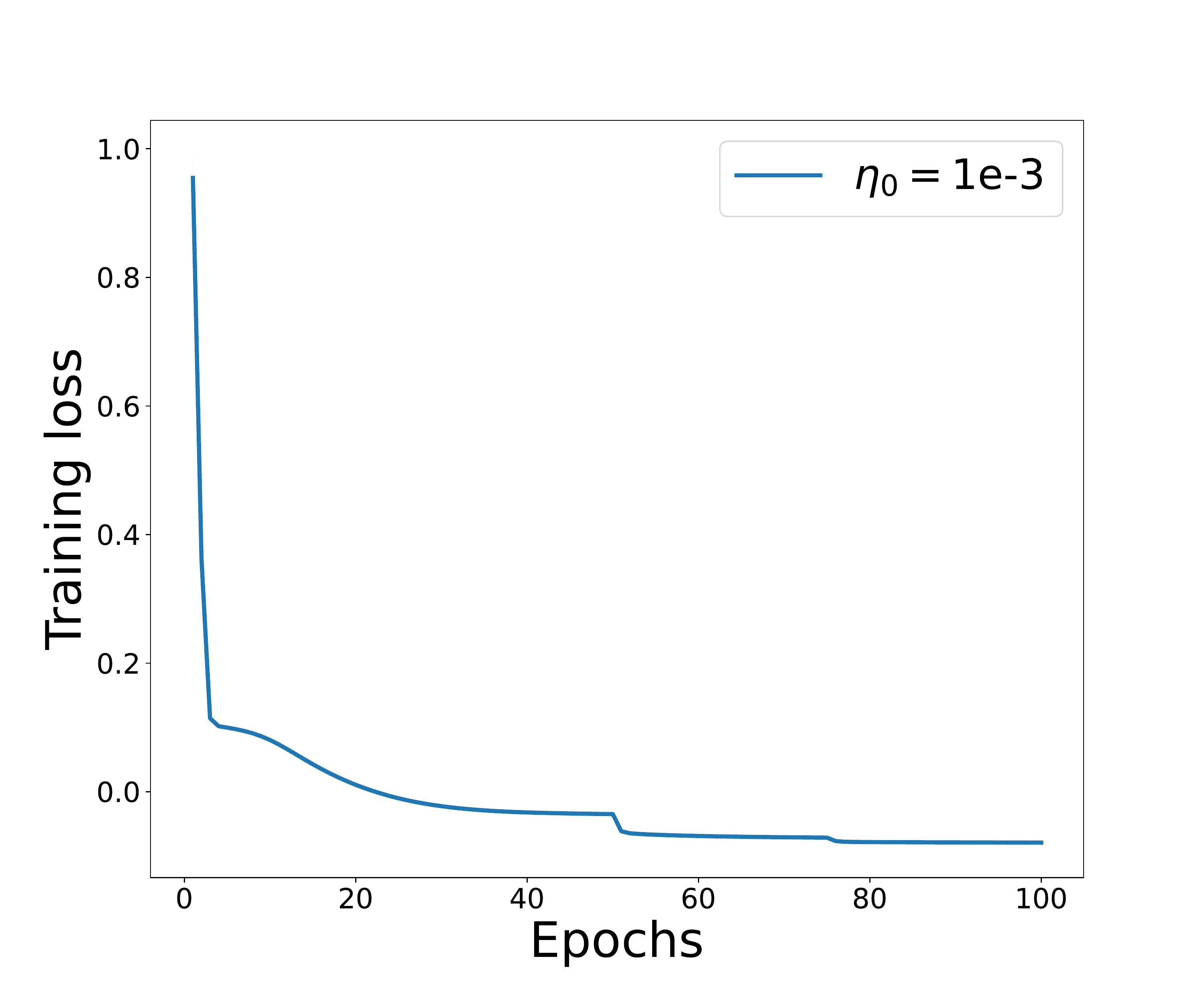}}
\subfigure[U(0.6), $c=1$]{\includegraphics[scale=0.1]{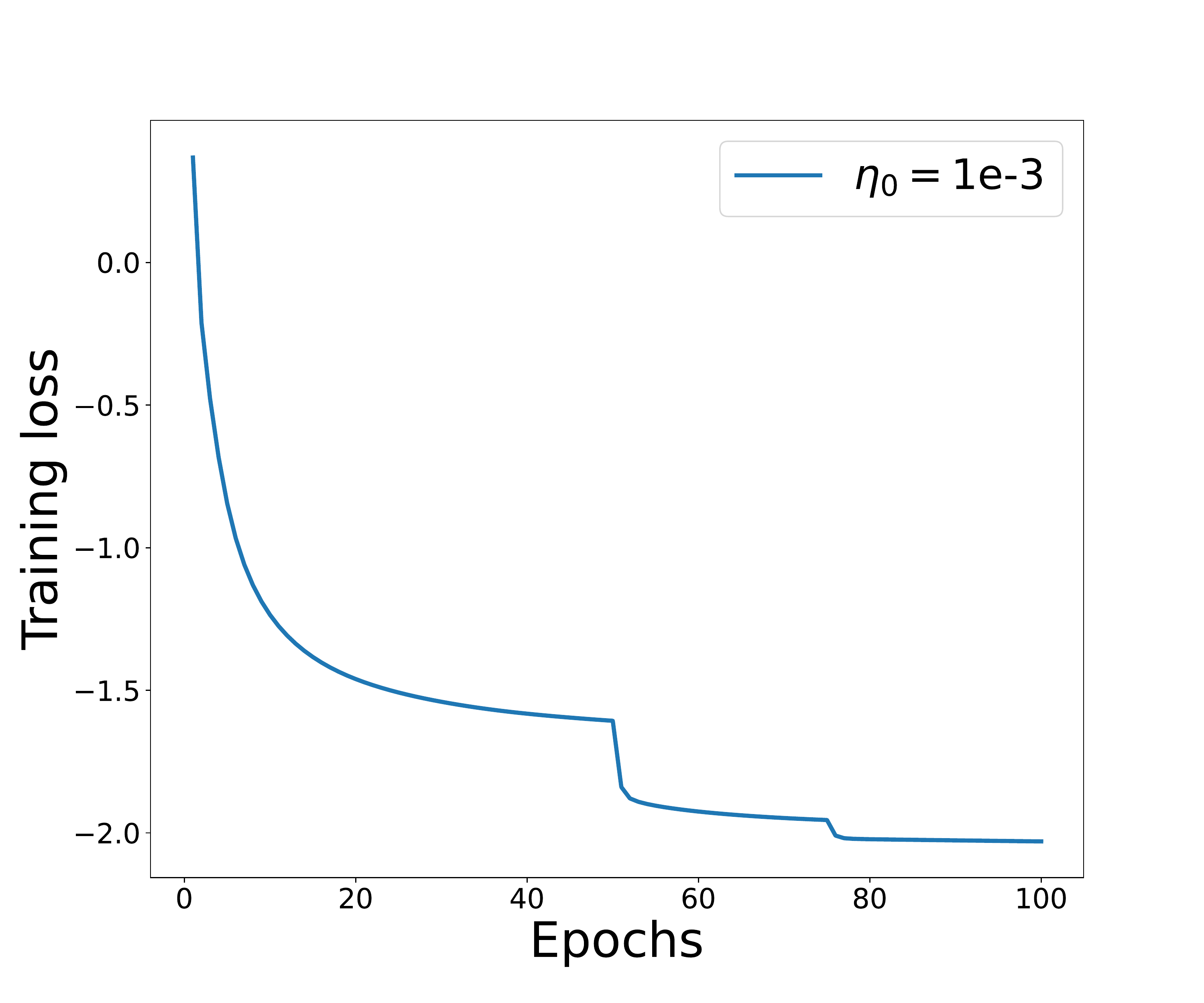}}
\subfigure[U(0.6), $c=10$]{\includegraphics[scale=0.1]{figures/aldr-aloi-uni06-cvg-top-m1.pdf}}

\subfigure[U(0.9), $c=0.1$]{\includegraphics[scale=0.1]{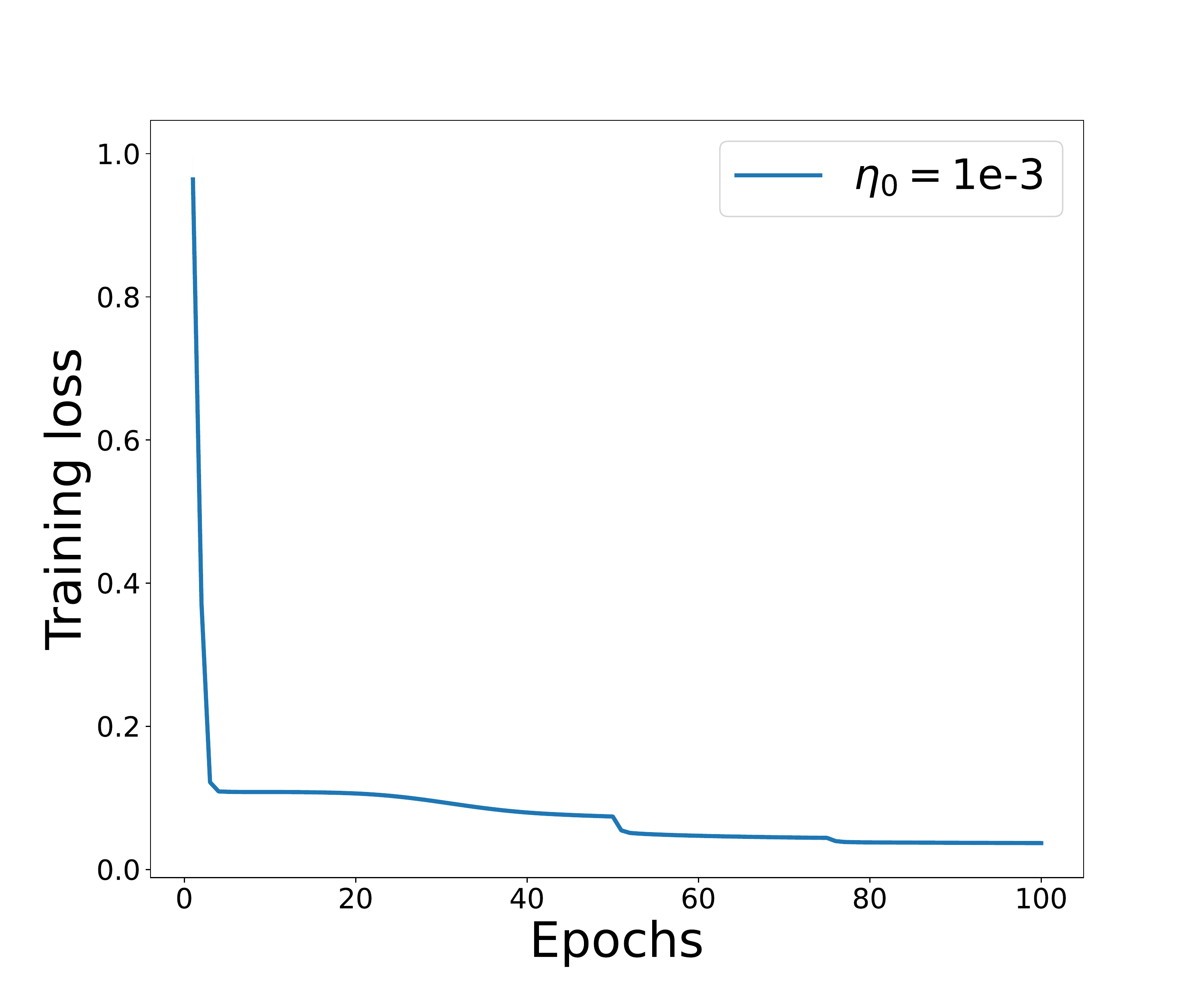}}
\subfigure[U(0.9), $c=1$]{\includegraphics[scale=0.1]{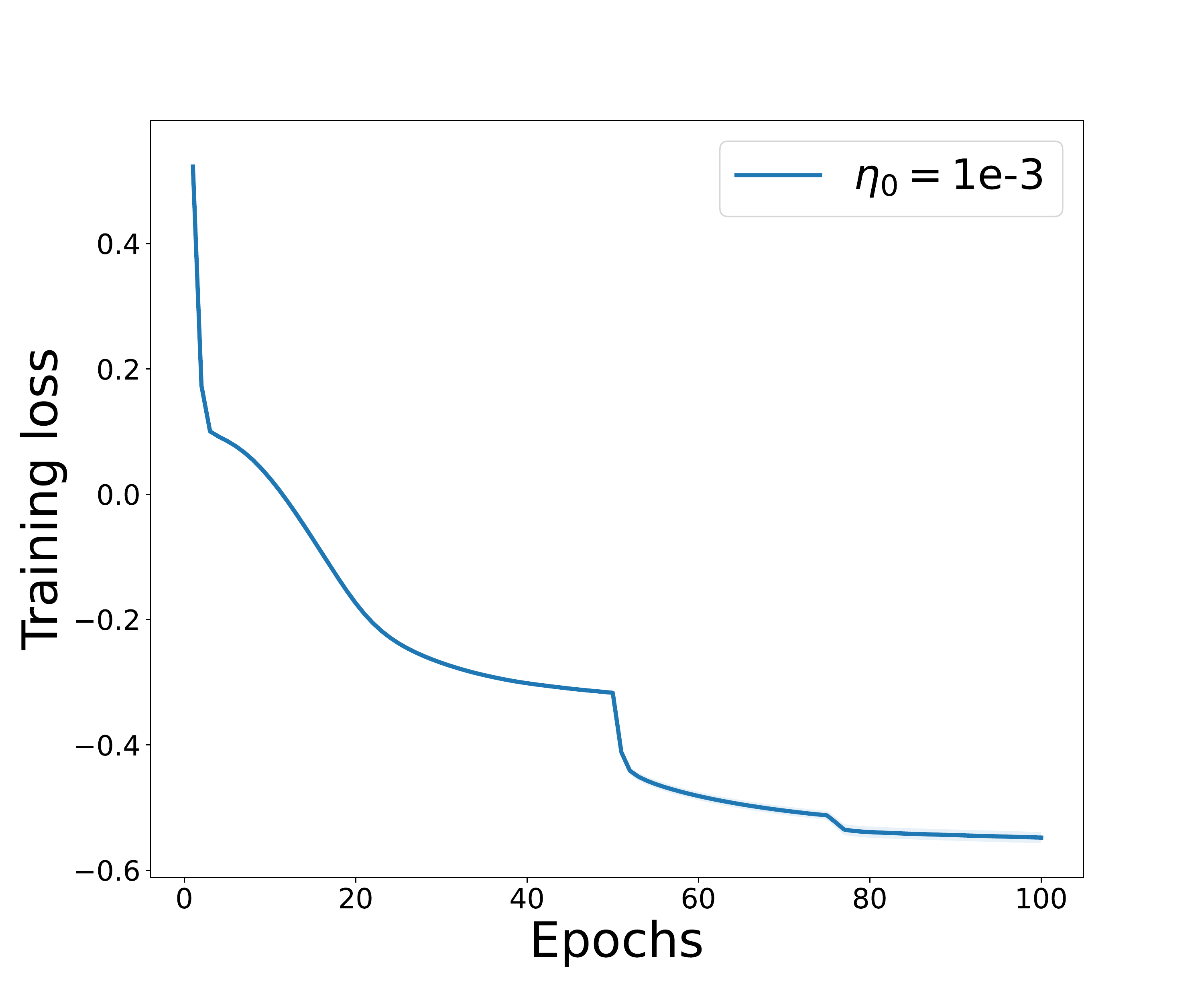}}
\subfigure[U(0.9), $c=10$]{\includegraphics[scale=0.1]{figures/aldr-aloi-uni09-cvg-top-m1.pdf}}
\caption{Training loss convergence for ALDR-KL (Algorithm 1) on ALOI dataset }\label{fig:aldr-aloi-cvg}
\end{figure}

\begin{figure}[p]
\centering

\subfigure[CD(0.1), $c=0.1$]{\includegraphics[scale=0.1]{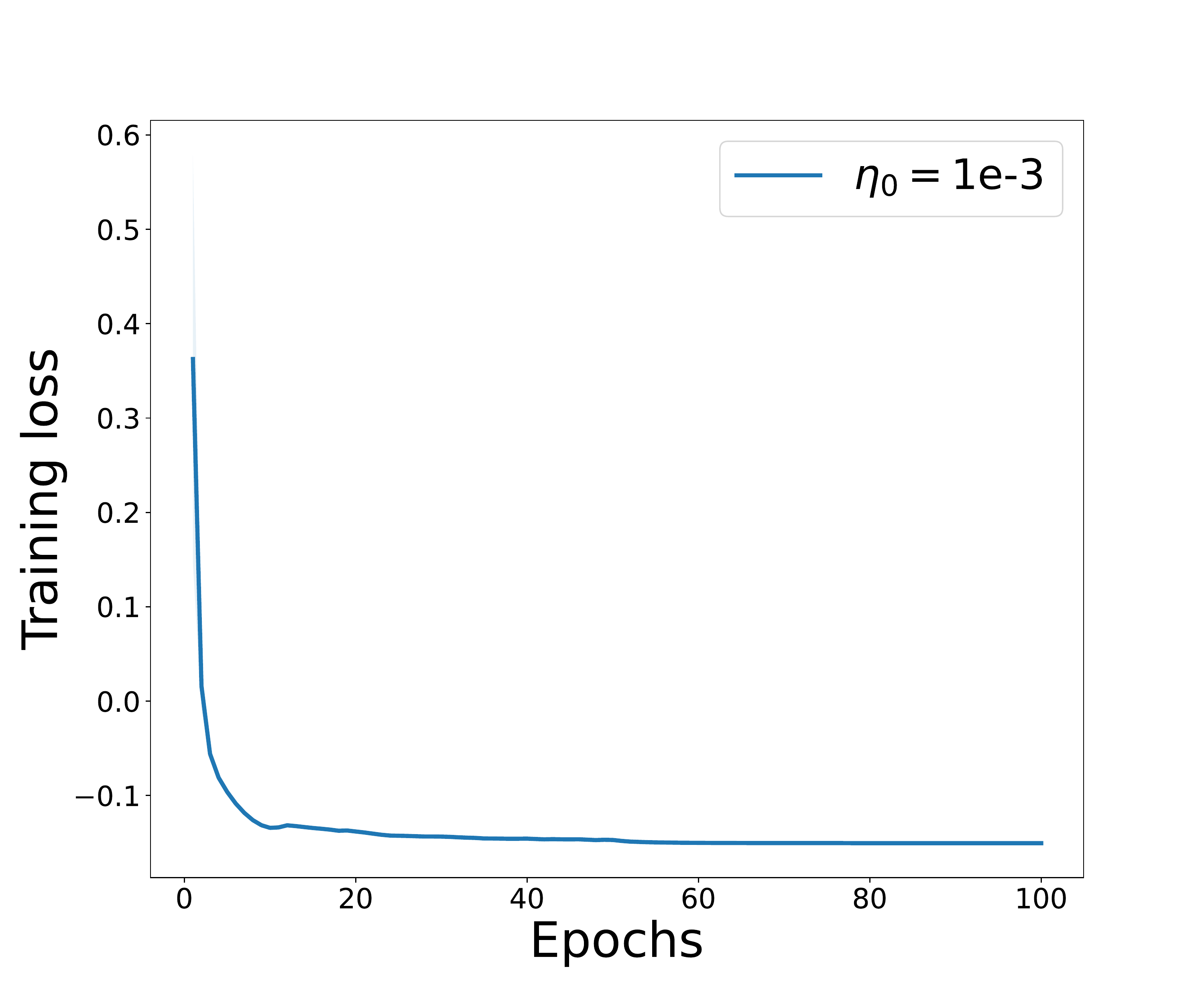}}
\subfigure[CD(0.1), $c=1$]{\includegraphics[scale=0.1]{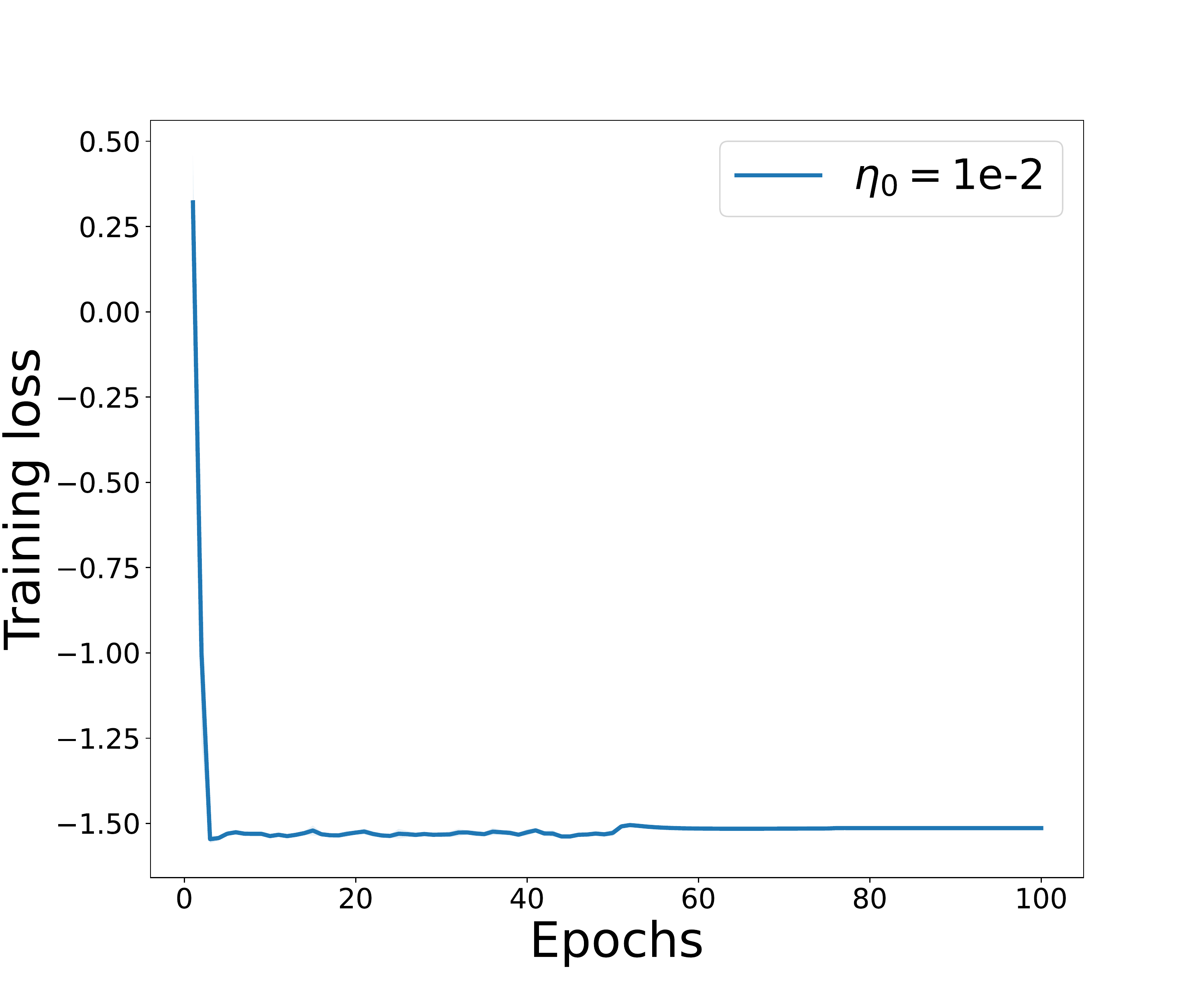}}
\subfigure[CD(0.1), $c=10$]{\includegraphics[scale=0.1]{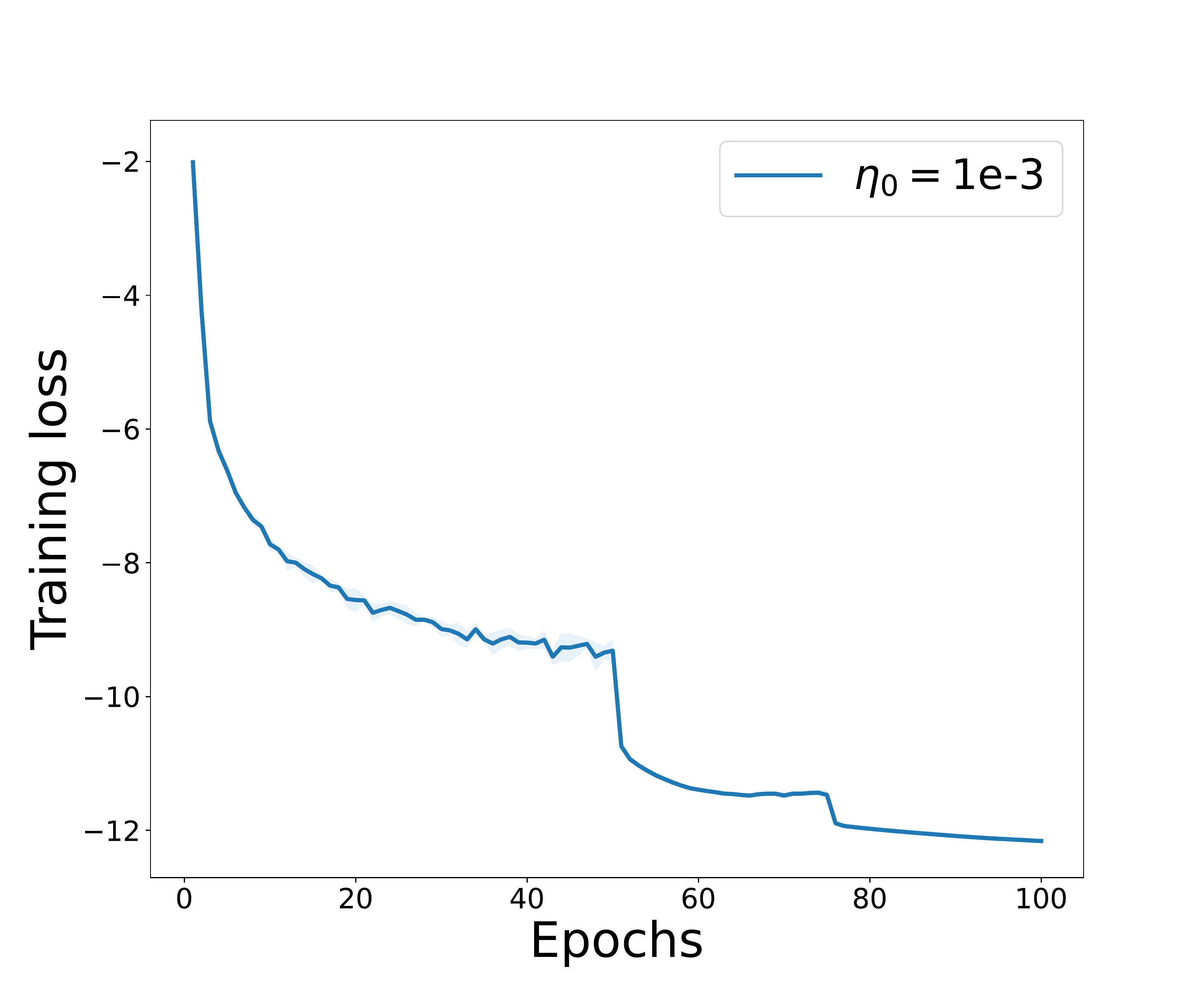}}

\subfigure[CD(0.3), $c=0.1$]{\includegraphics[scale=0.1]{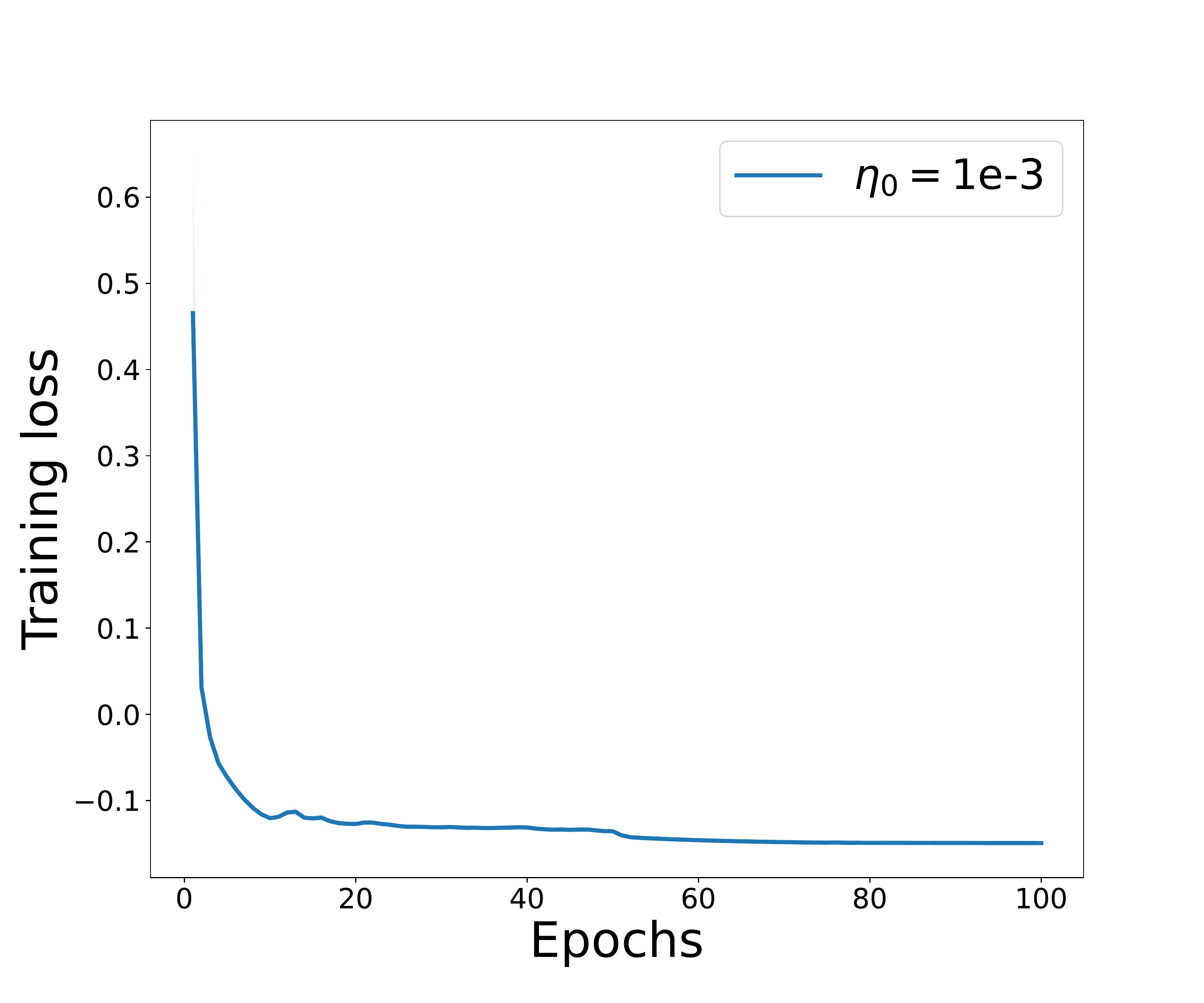}}
\subfigure[CD(0.3), $c=1$]{\includegraphics[scale=0.1]{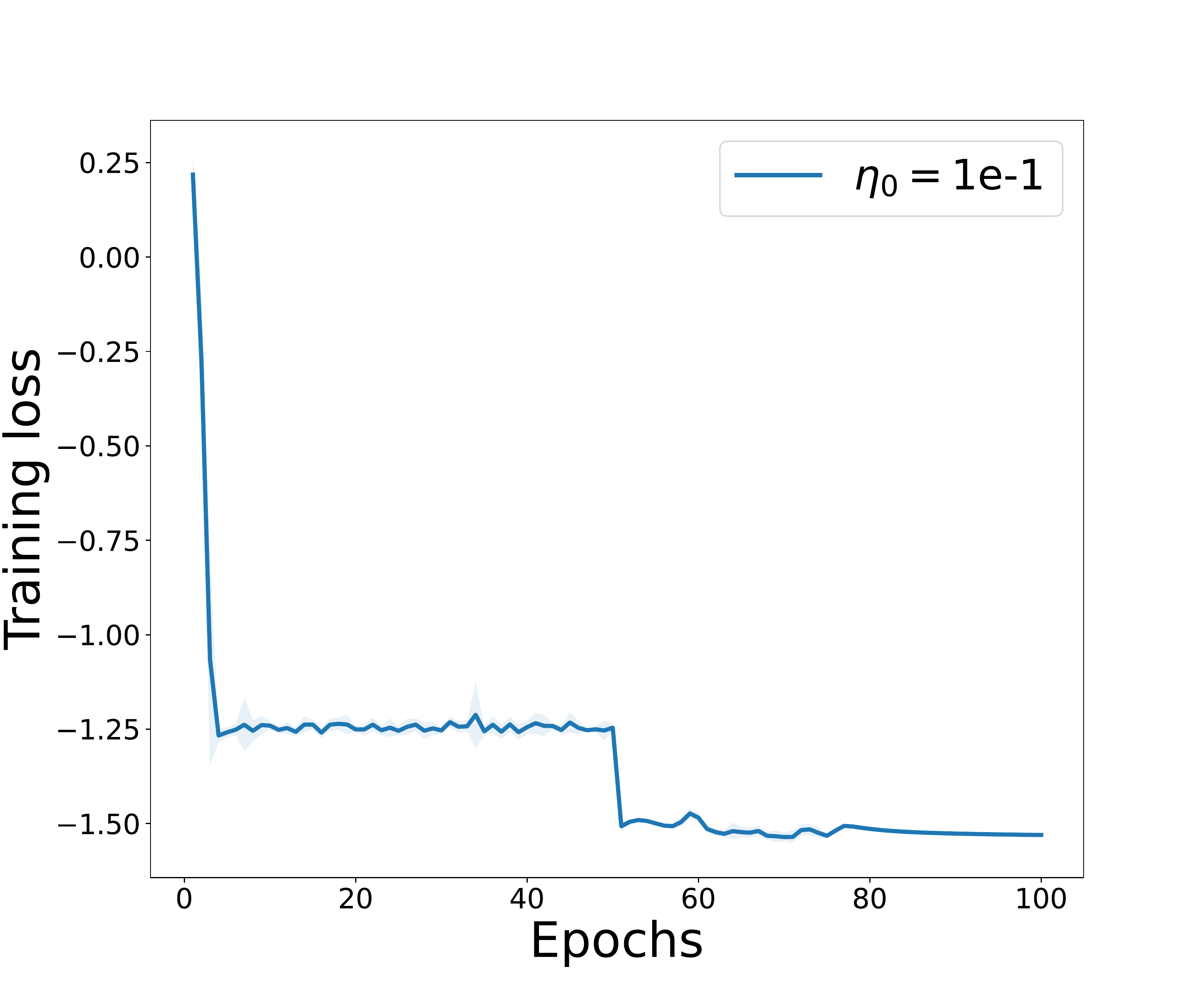}}
\subfigure[CD(0.3), $c=10$]{\includegraphics[scale=0.1]{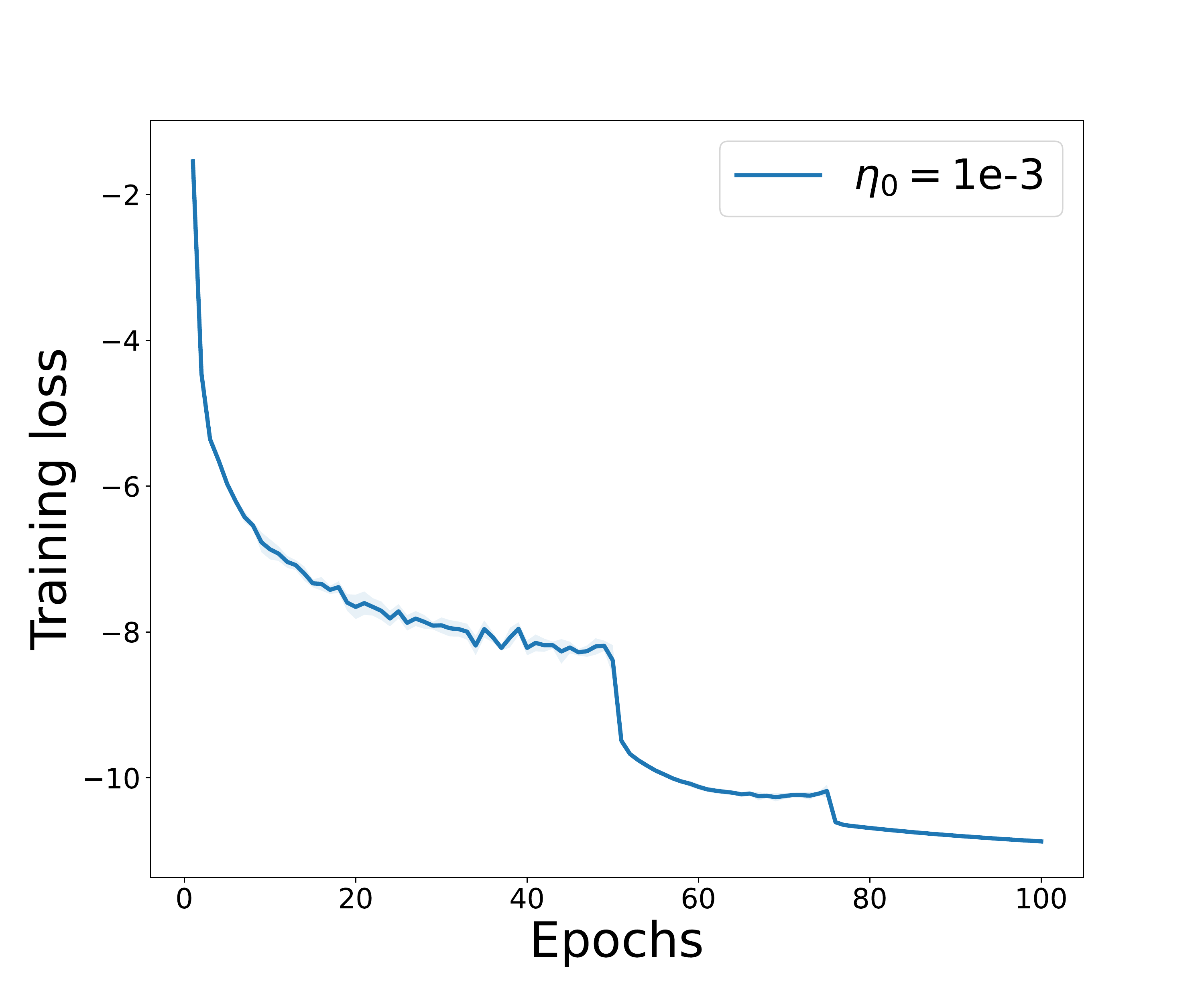}}

\subfigure[CD(0.5), $c=0.1$]{\includegraphics[scale=0.1]{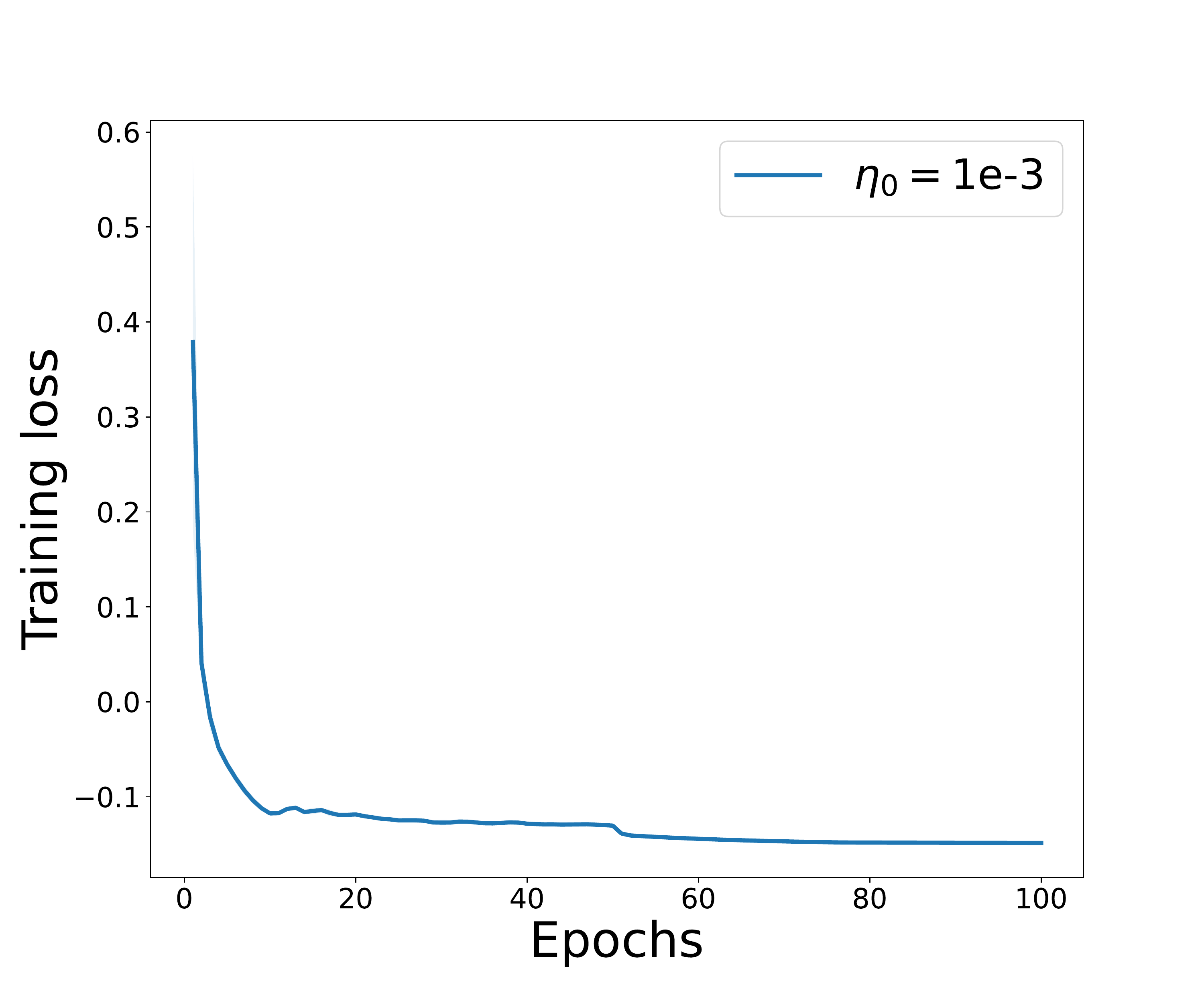}}
\subfigure[CD(0.5), $c=1$]{\includegraphics[scale=0.1]{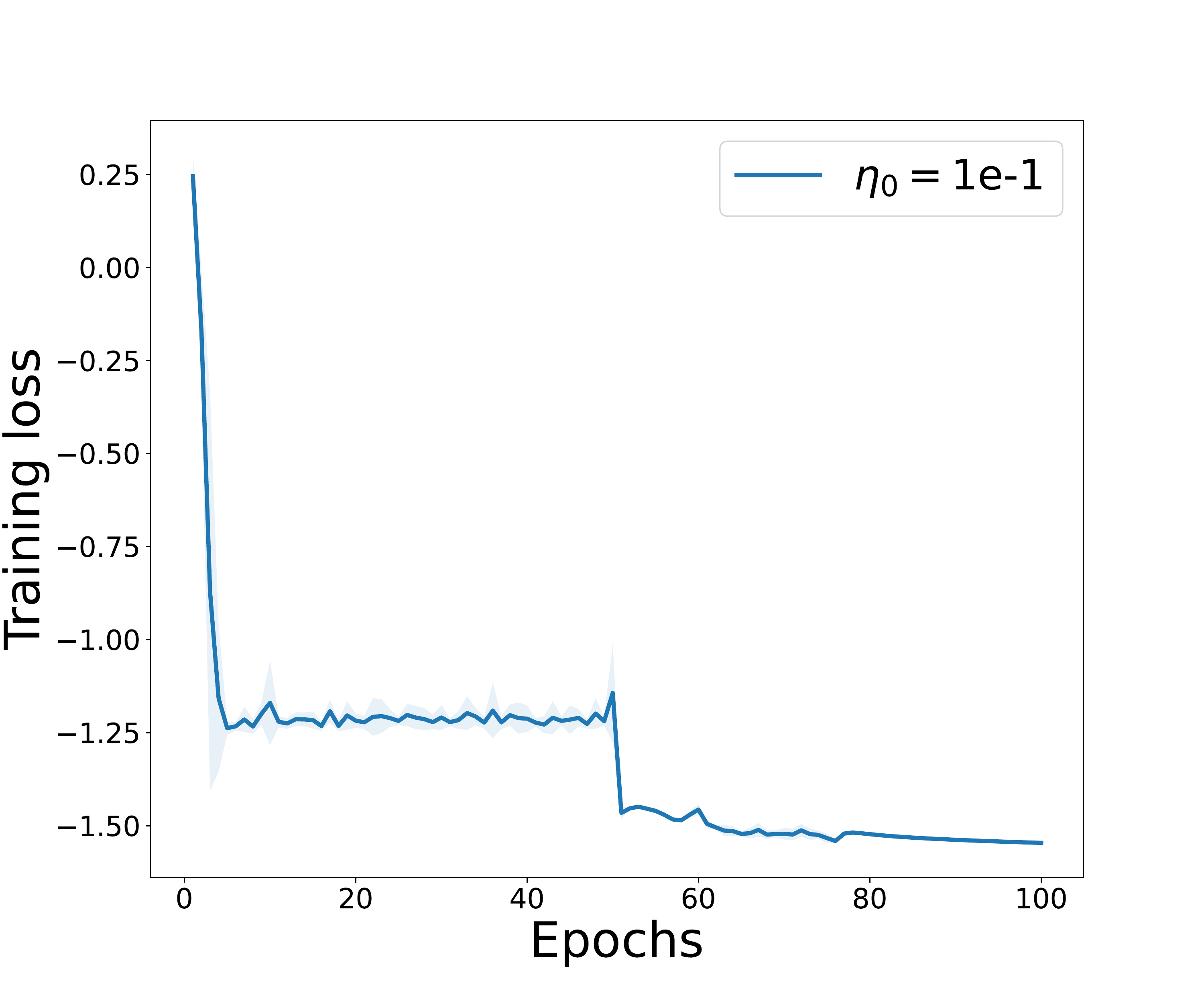}}
\subfigure[CD(0.5), $c=10$]{\includegraphics[scale=0.1]{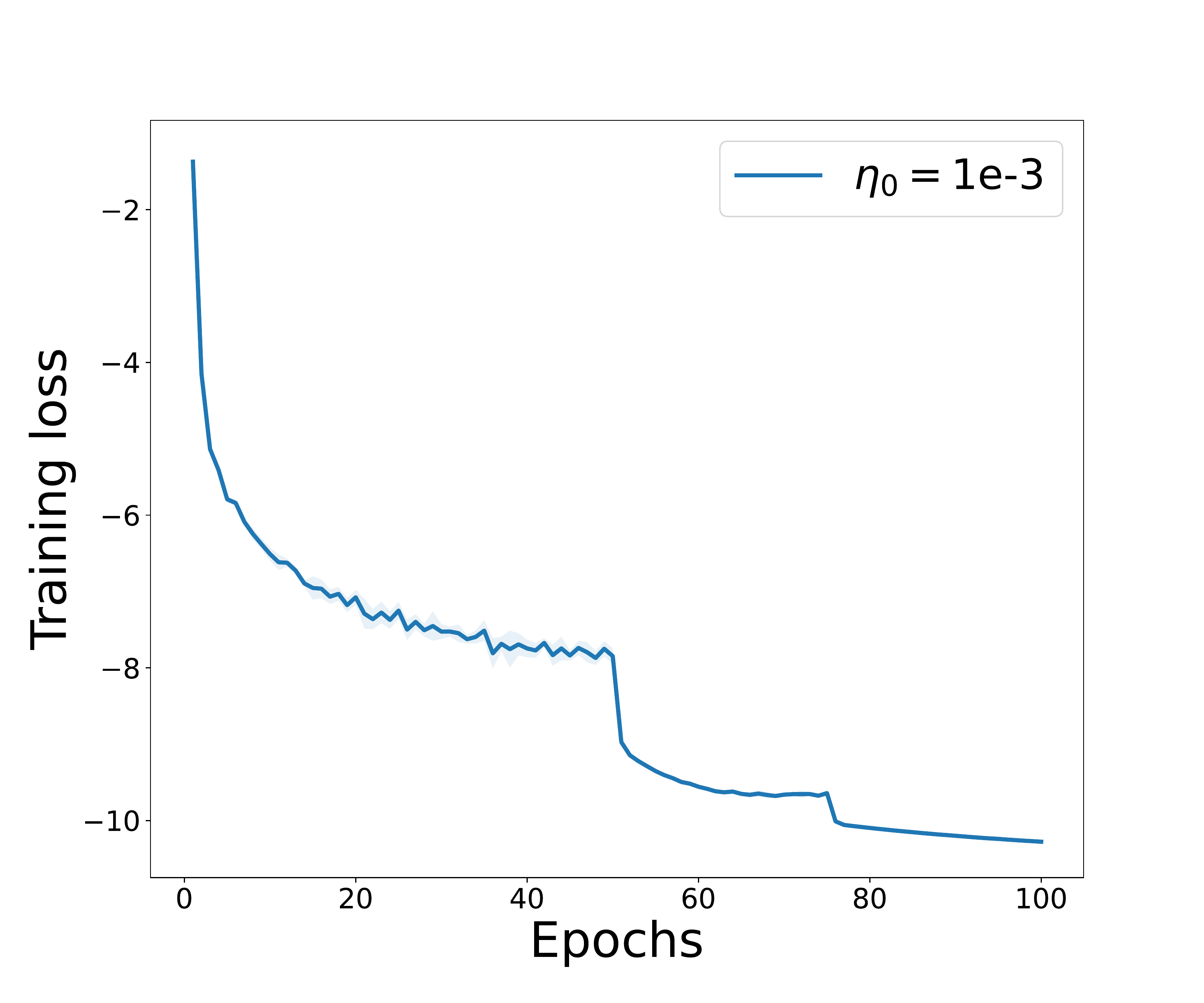}}

\subfigure[U(0.3), $c=0.1$]{\includegraphics[scale=0.1]{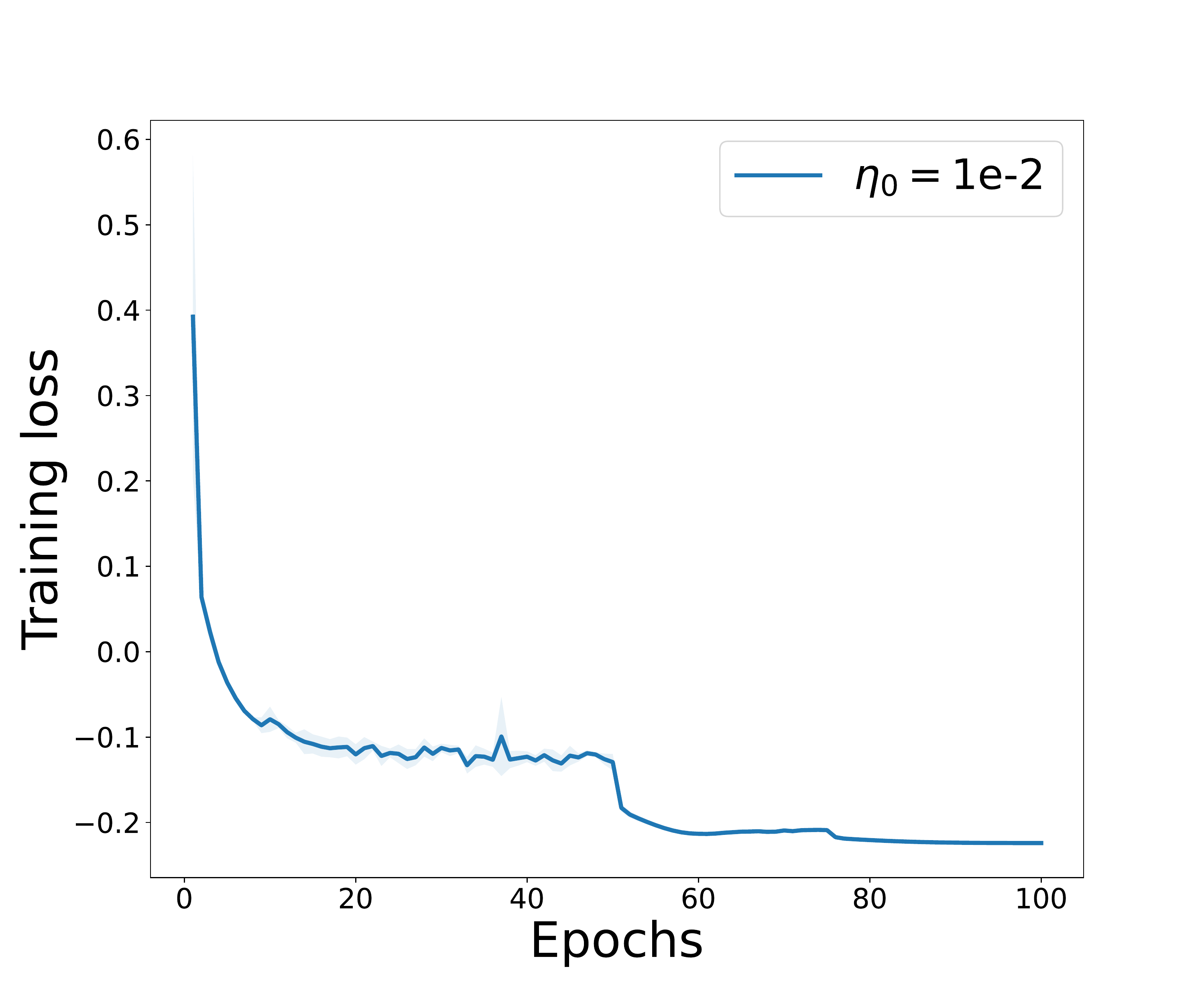}}
\subfigure[U(0.3), $c=1$]{\includegraphics[scale=0.1]{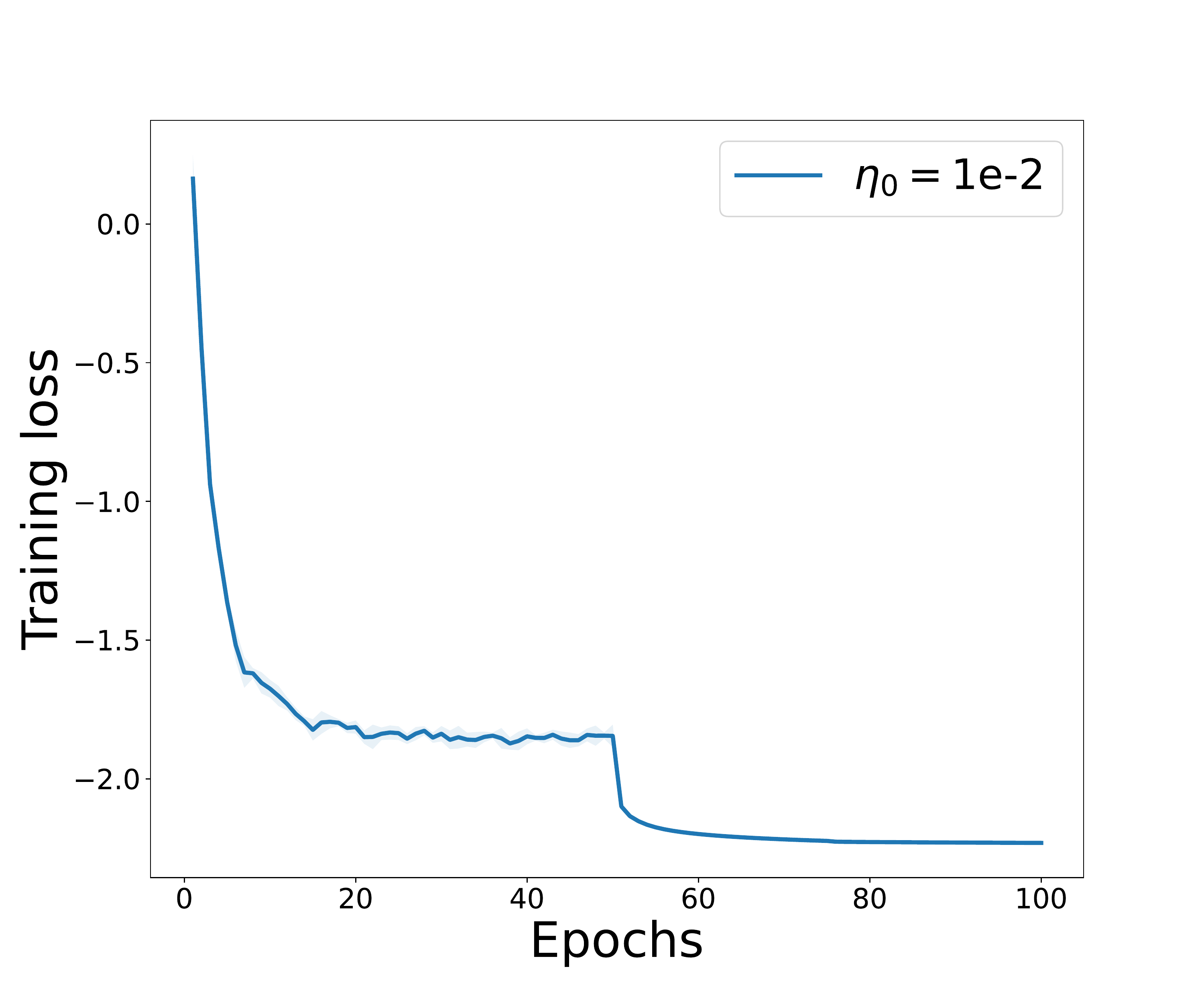}}
\subfigure[U(0.3), $c=10$]{\includegraphics[scale=0.1]{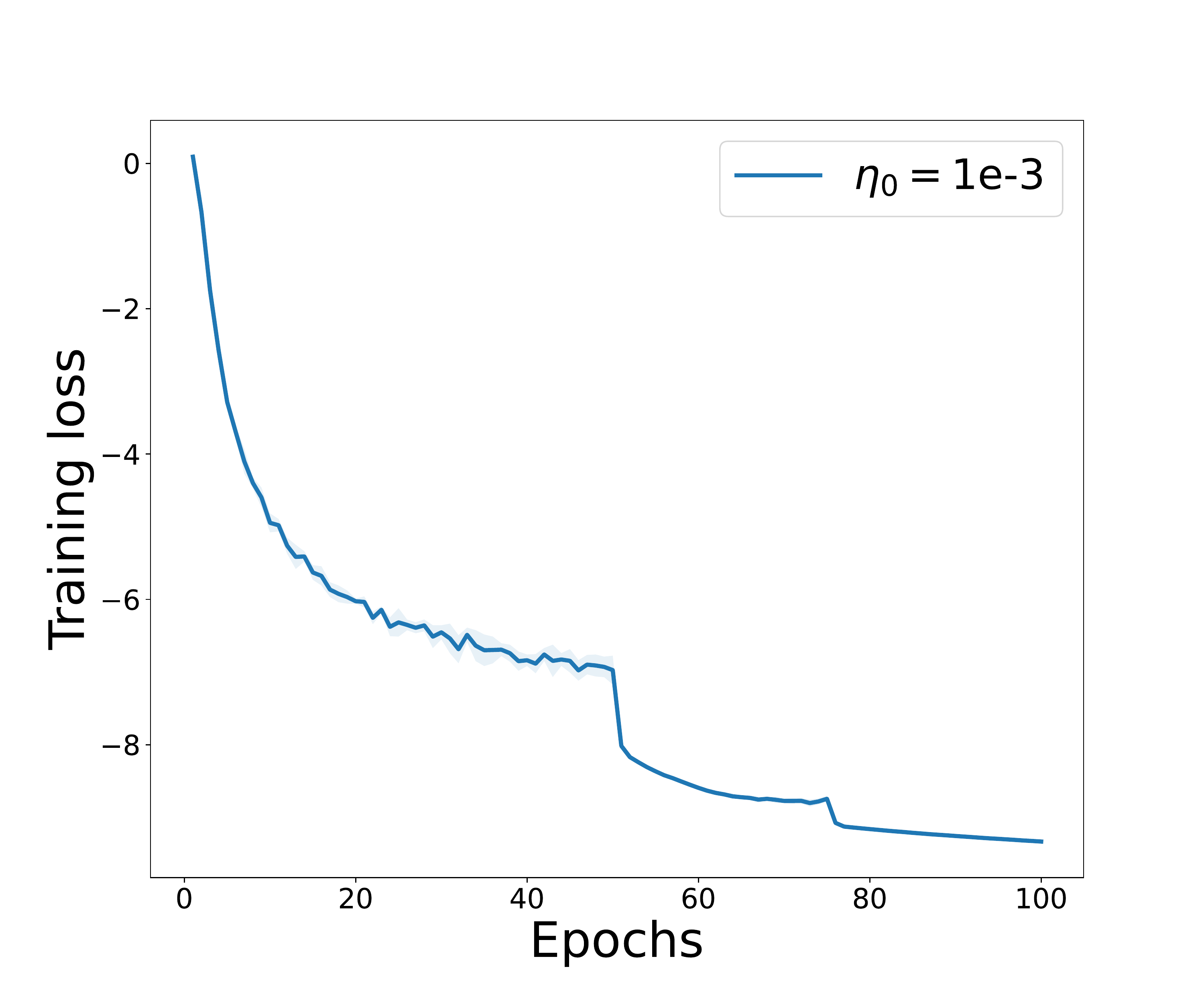}}

\subfigure[U(0.6), $c=0.1$]{\includegraphics[scale=0.1]{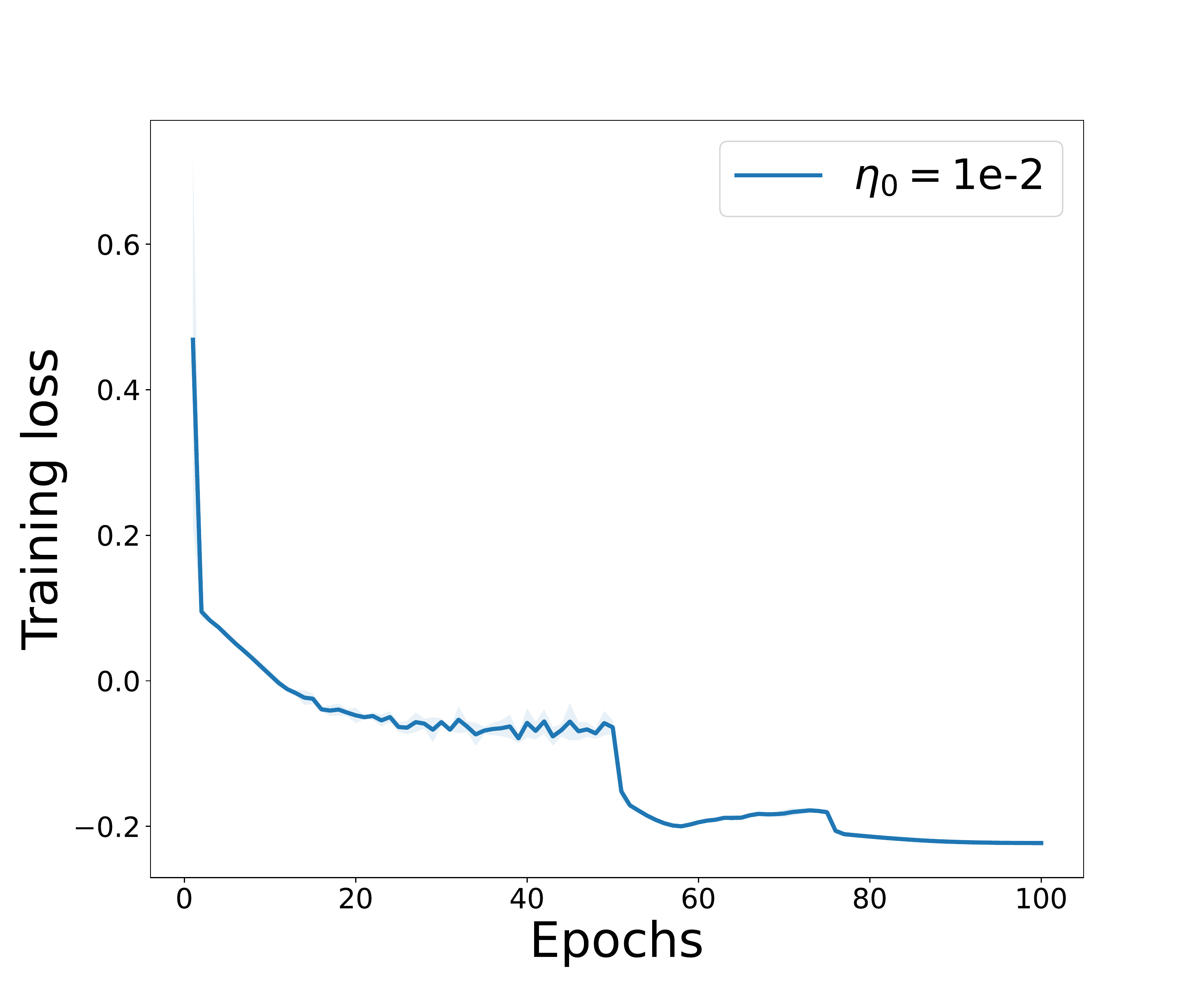}}
\subfigure[U(0.6), $c=1$]{\includegraphics[scale=0.1]{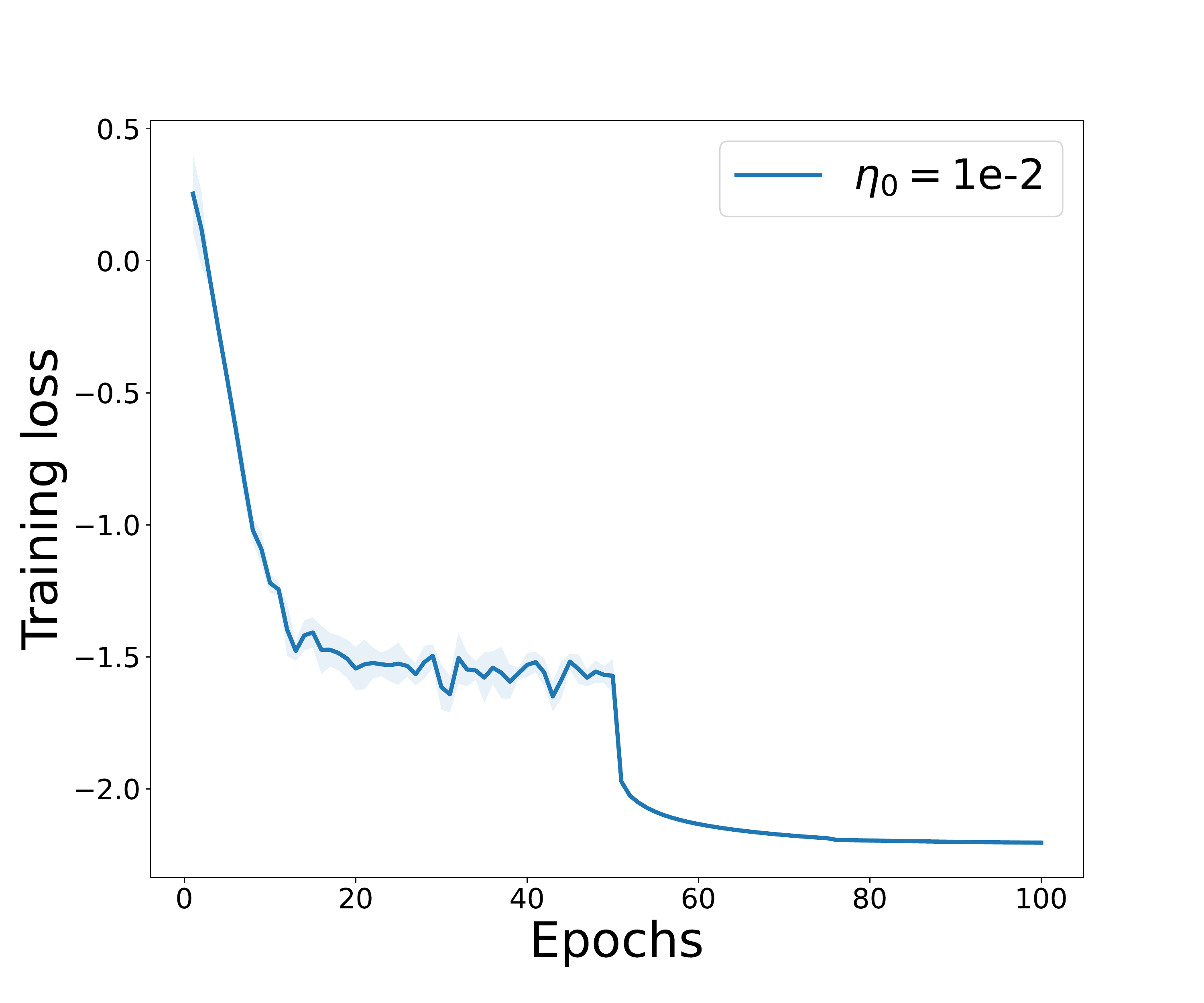}}
\subfigure[U(0.6), $c=10$]{\includegraphics[scale=0.1]{figures/aldr-news20-uni06-cvg-top-m1.pdf}}

\subfigure[U(0.9), $c=0.1$]{\includegraphics[scale=0.1]{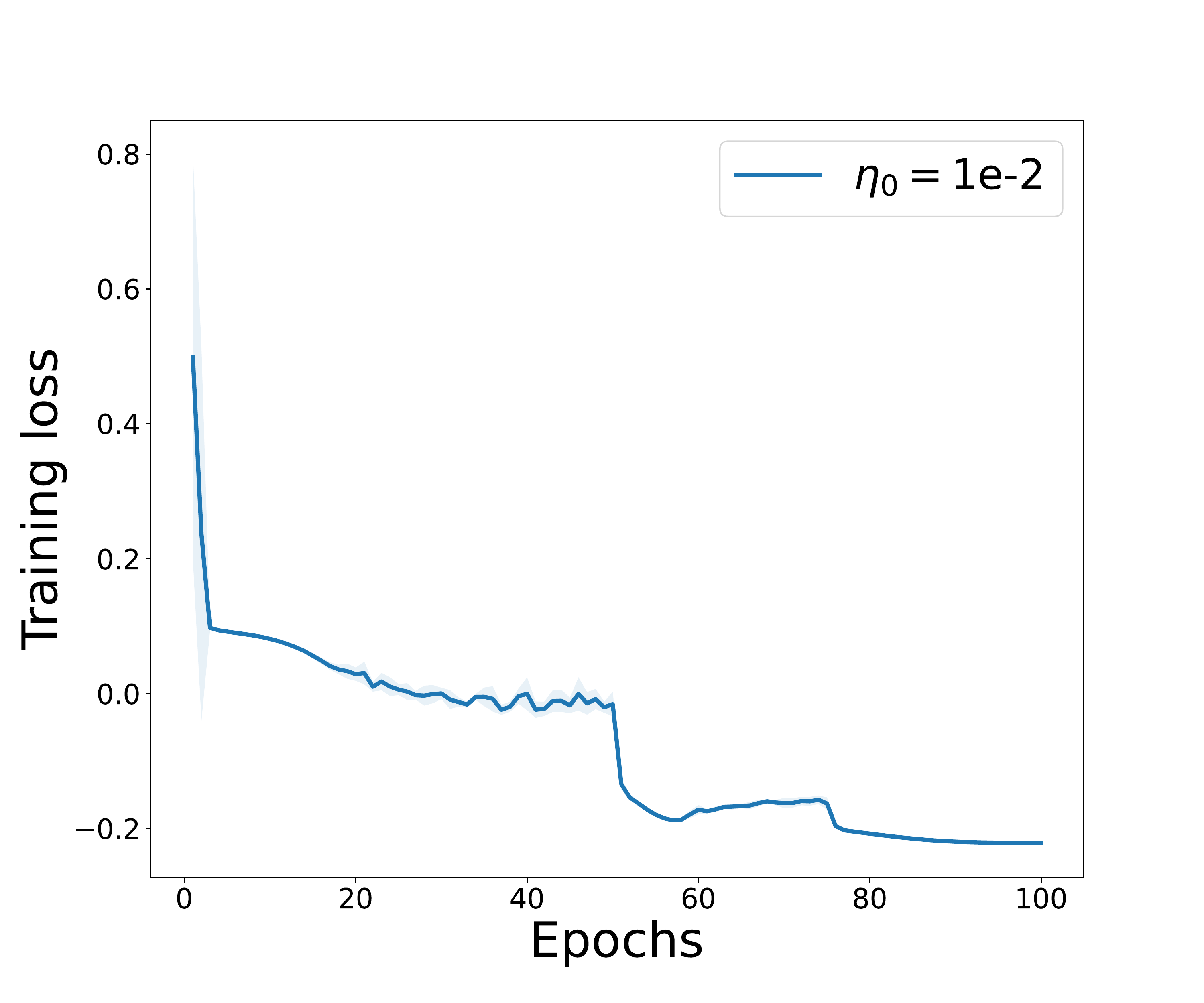}}
\subfigure[U(0.9), $c=1$]{\includegraphics[scale=0.1]{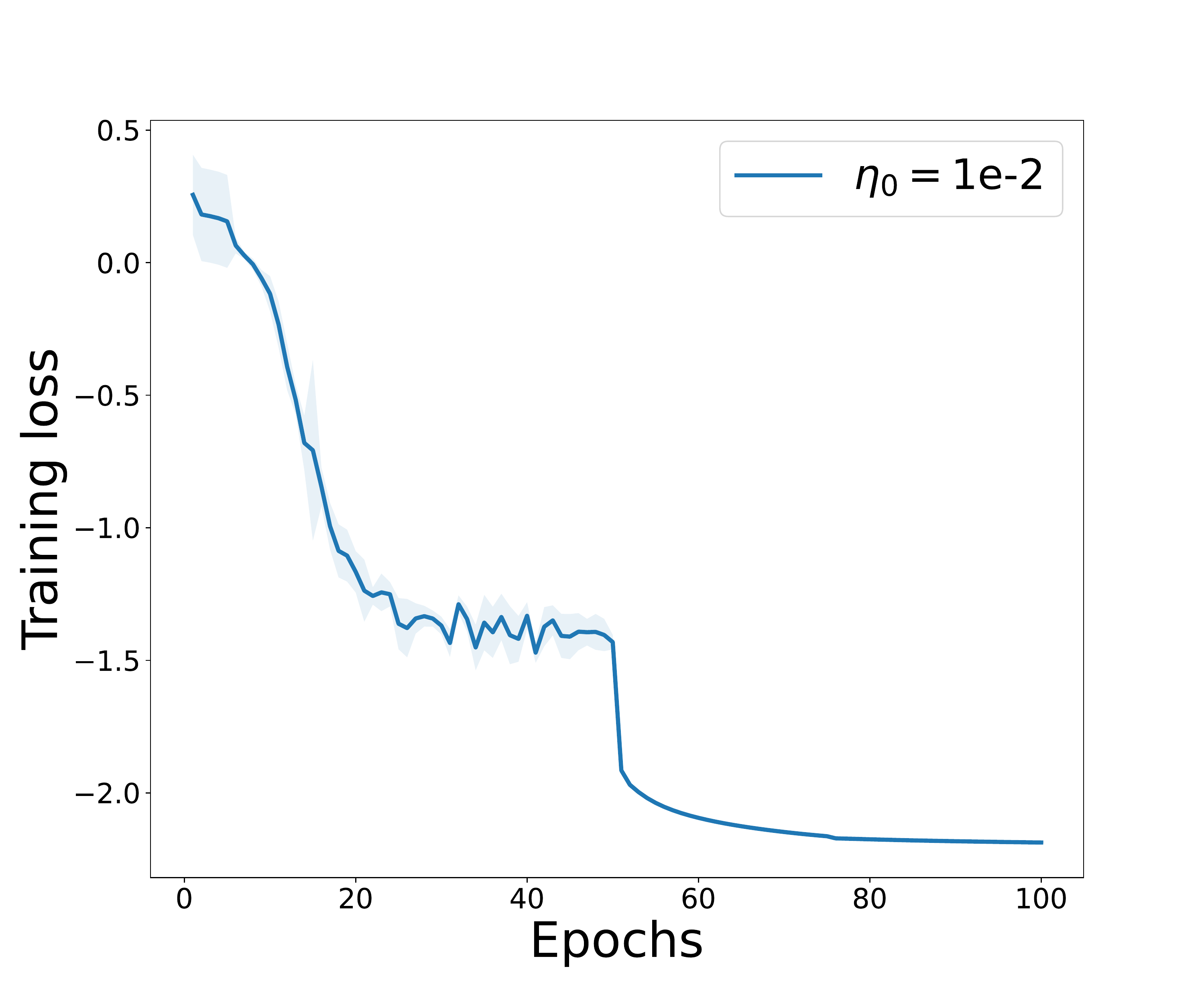}}
\subfigure[U(0.9), $c=10$]{\includegraphics[scale=0.1]{figures/aldr-news20-uni09-cvg-top-m1.pdf}}
\caption{Training loss convergence for ALDR-KL (Algorithm 1) on News20 dataset }\label{fig:aldr-news20-cvg}
\end{figure}

\begin{figure}[p]
\centering

\subfigure[CD(0.1), $c=0.1$]{\includegraphics[scale=0.1]{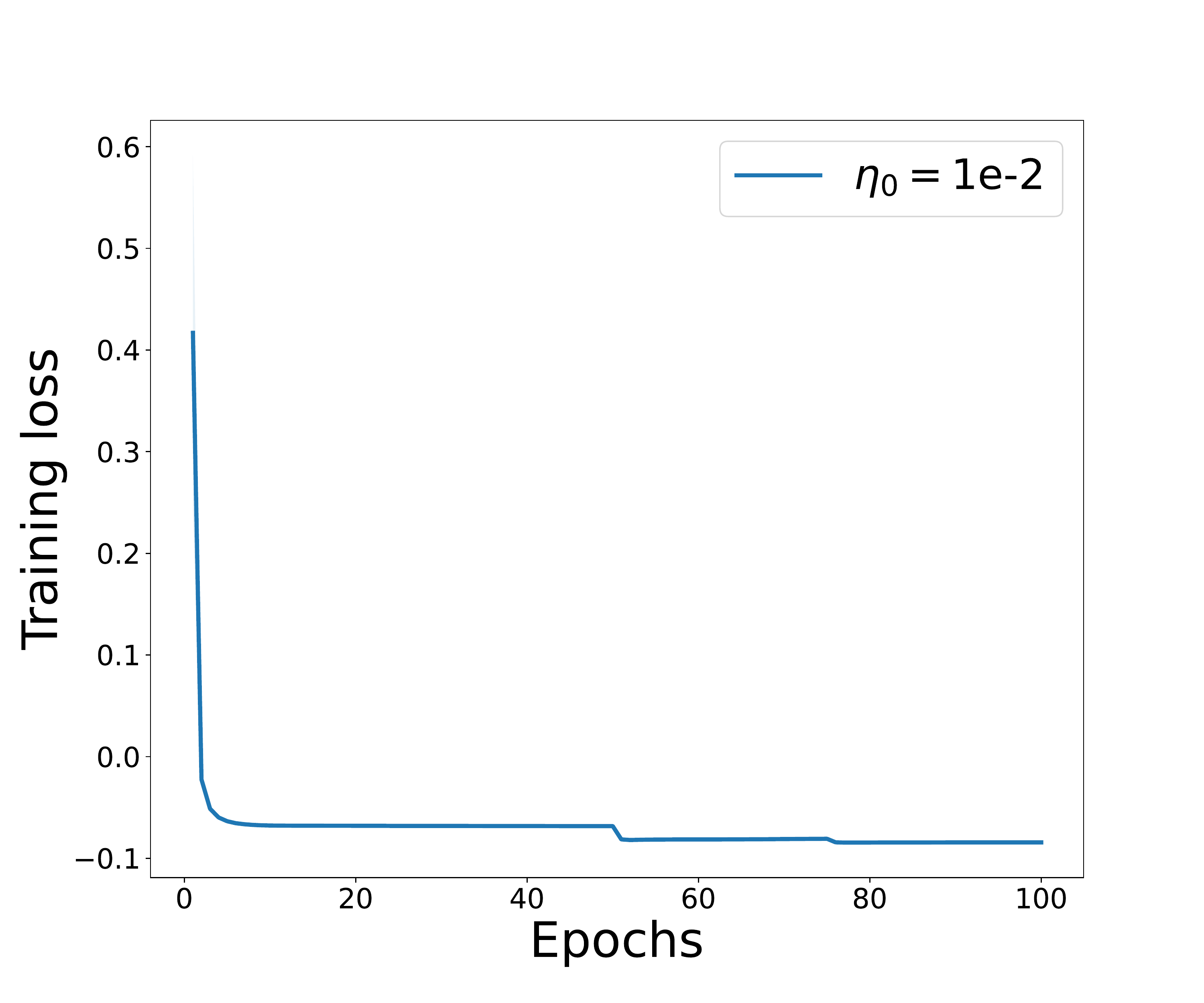}}
\subfigure[CD(0.1), $c=1$]{\includegraphics[scale=0.1]{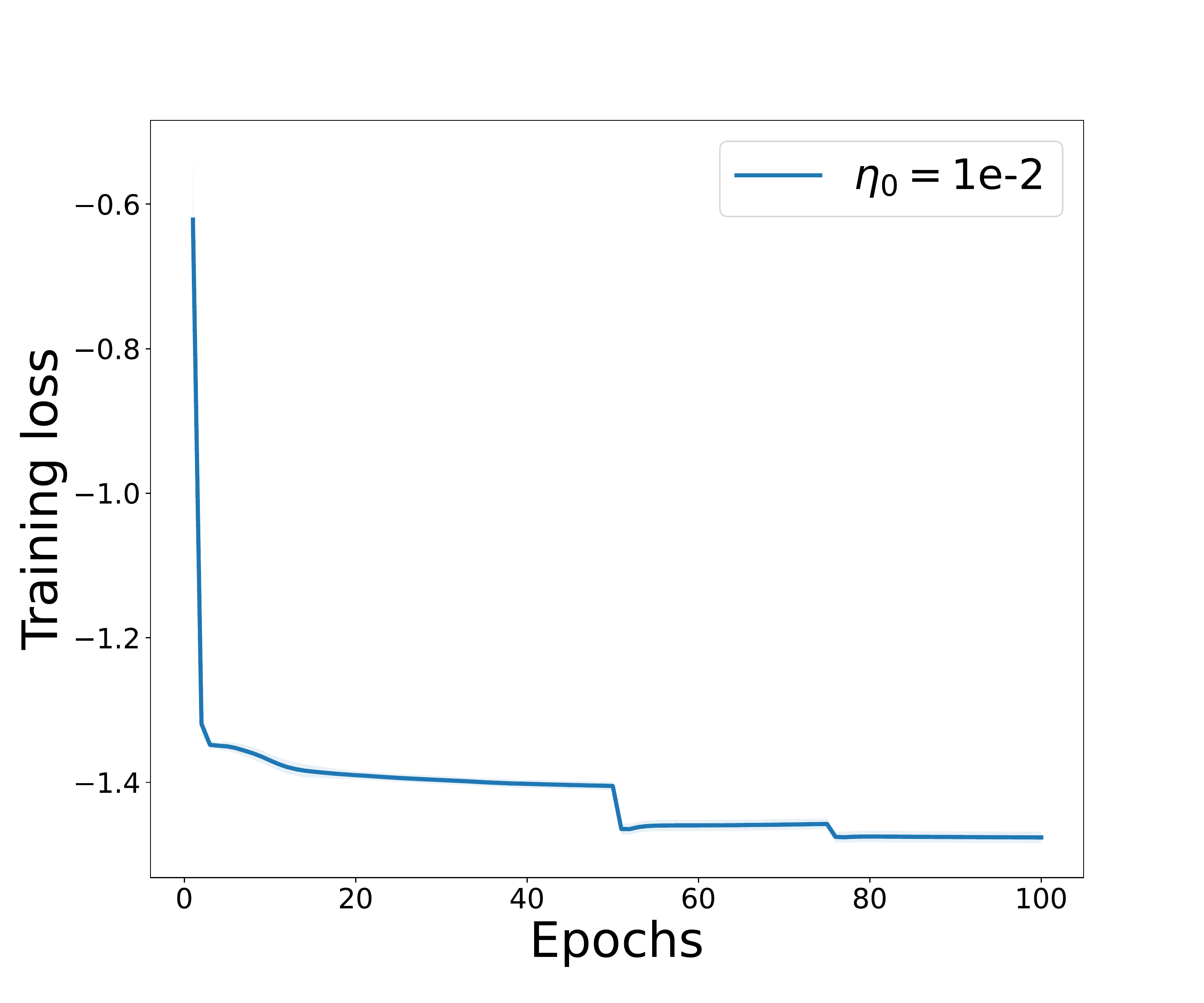}}
\subfigure[CD(0.1), $c=10$]{\includegraphics[scale=0.1]{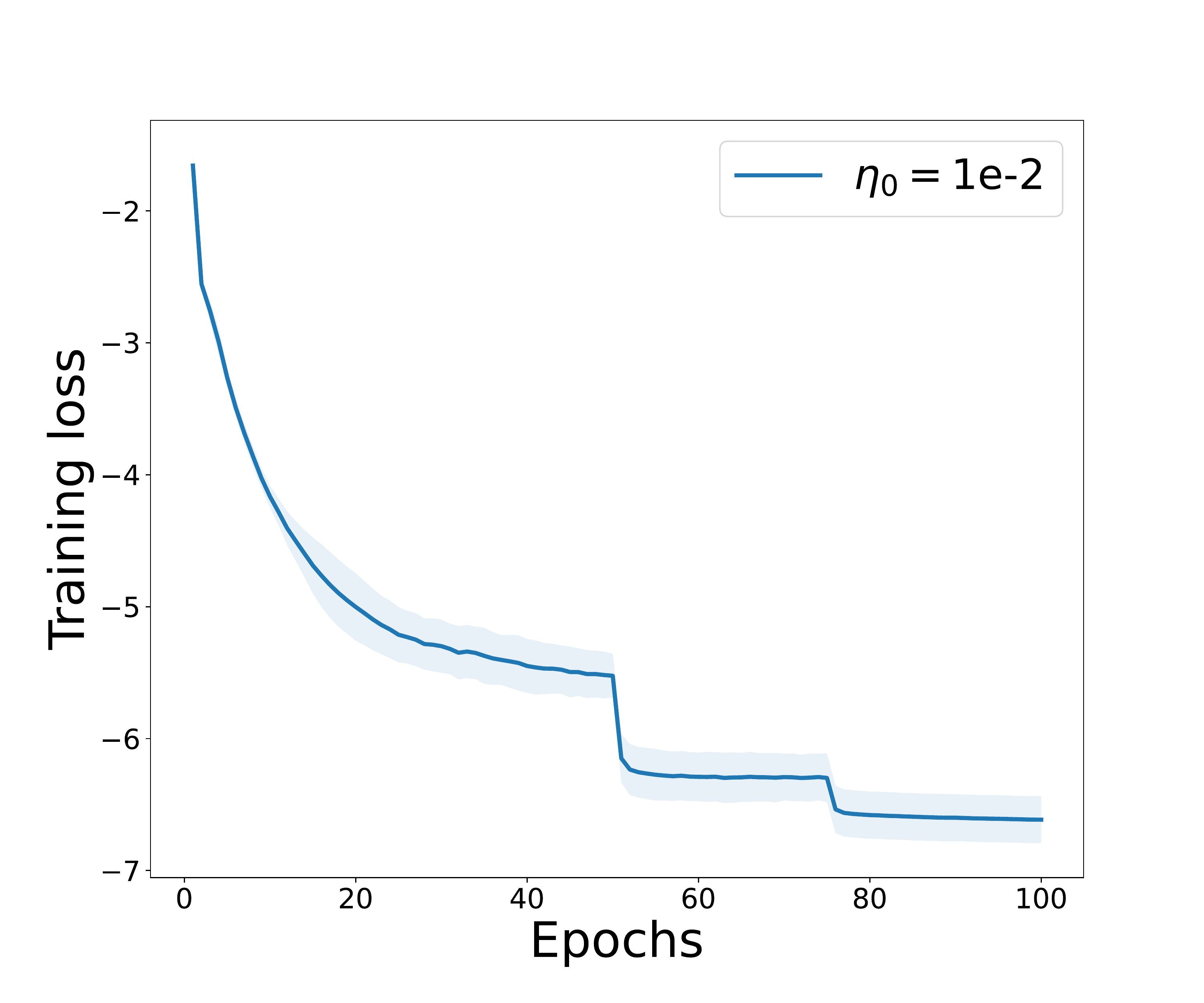}}

\subfigure[CD(0.3), $c=0.1$]{\includegraphics[scale=0.1]{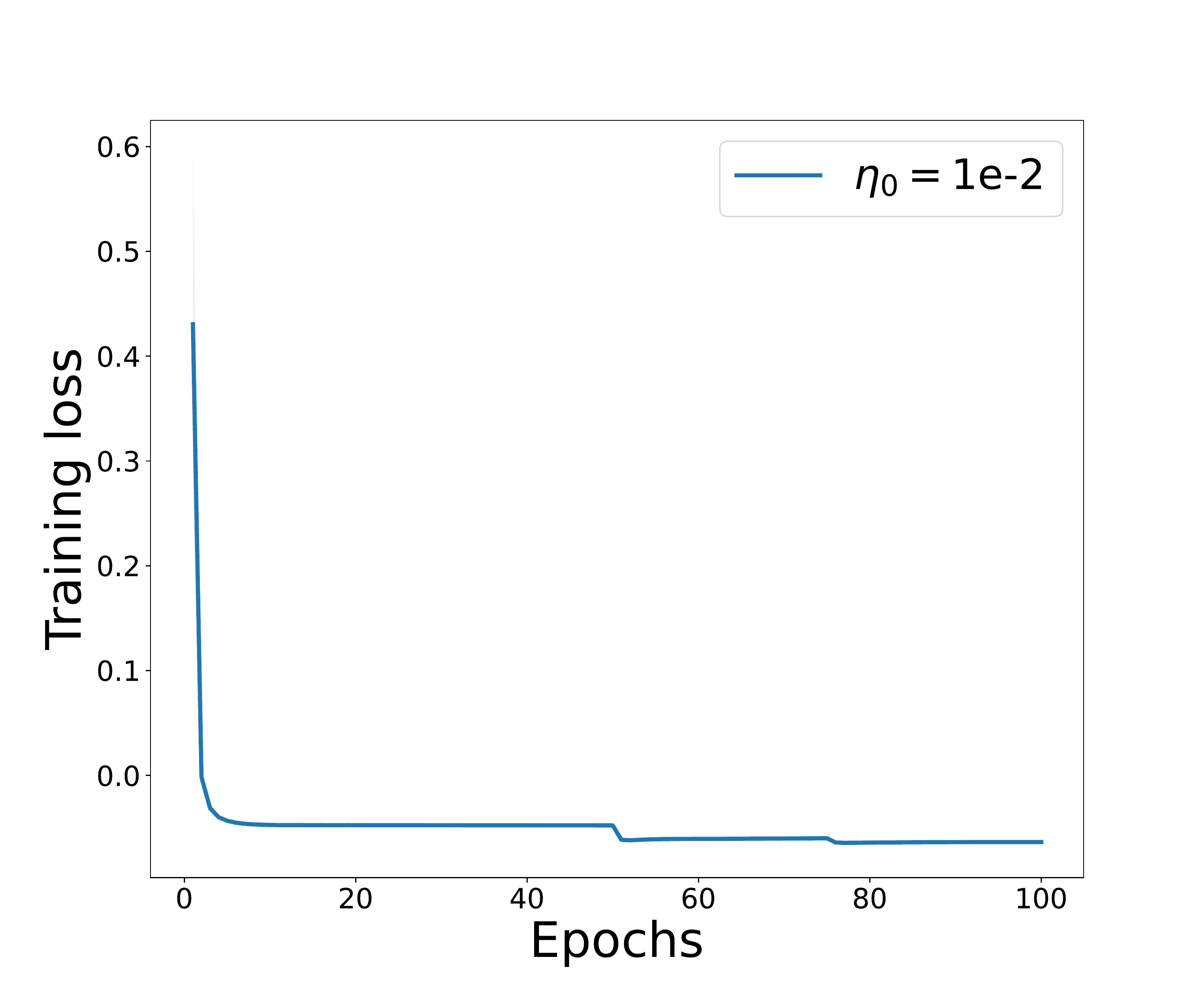}}
\subfigure[CD(0.3), $c=1$]{\includegraphics[scale=0.1]{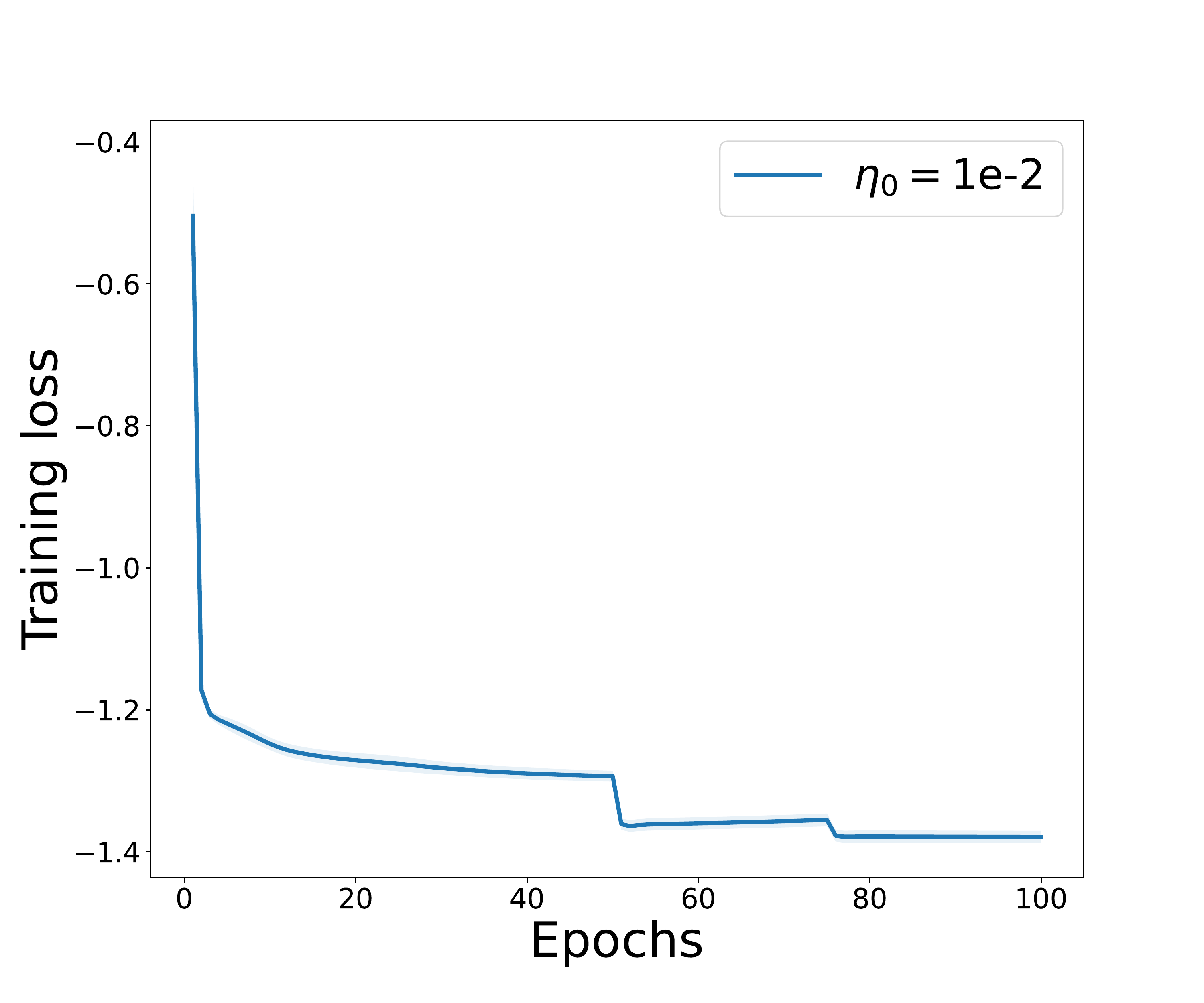}}
\subfigure[CD(0.3), $c=10$]{\includegraphics[scale=0.1]{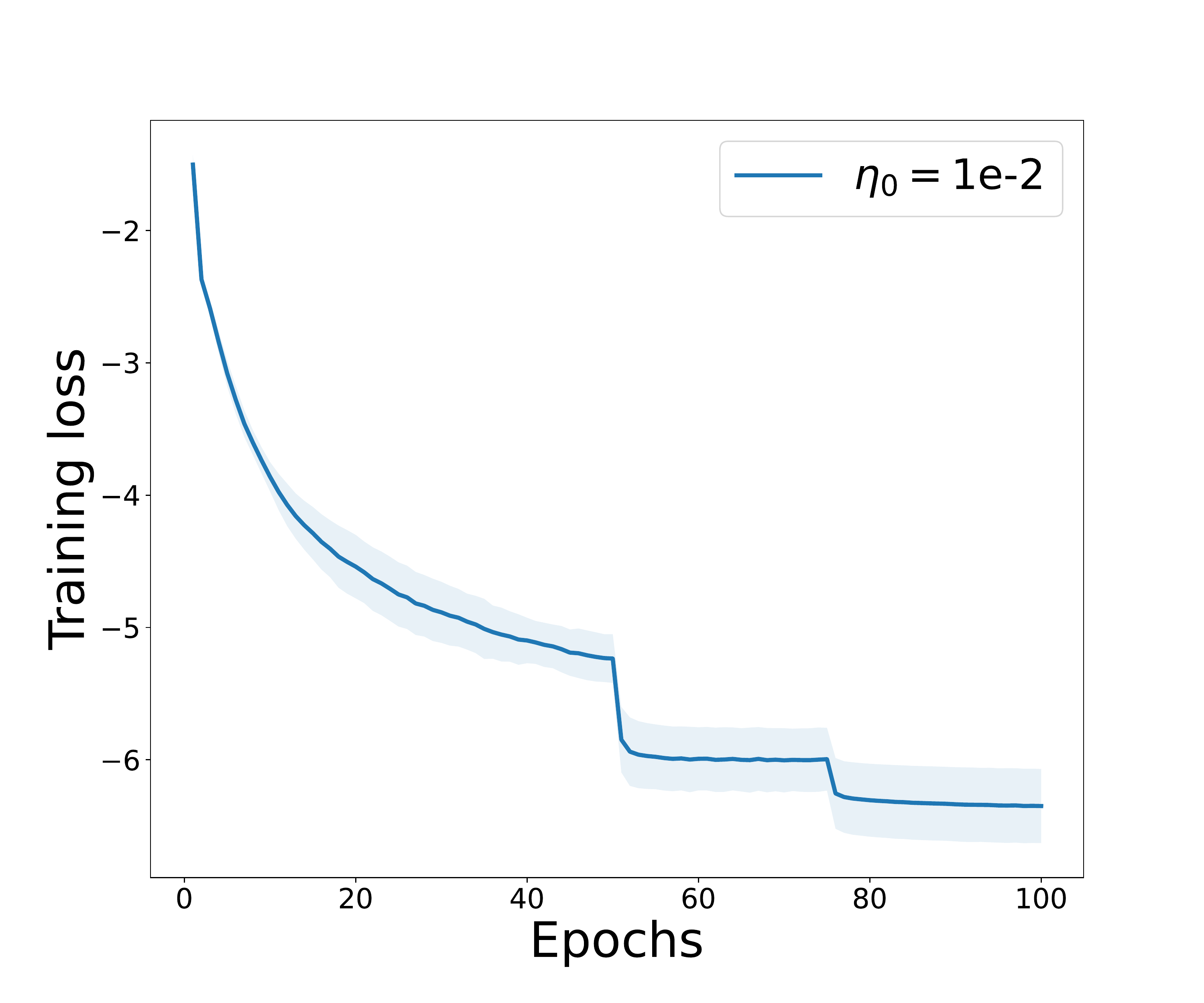}}

\subfigure[CD(0.5), $c=0.1$]{\includegraphics[scale=0.1]{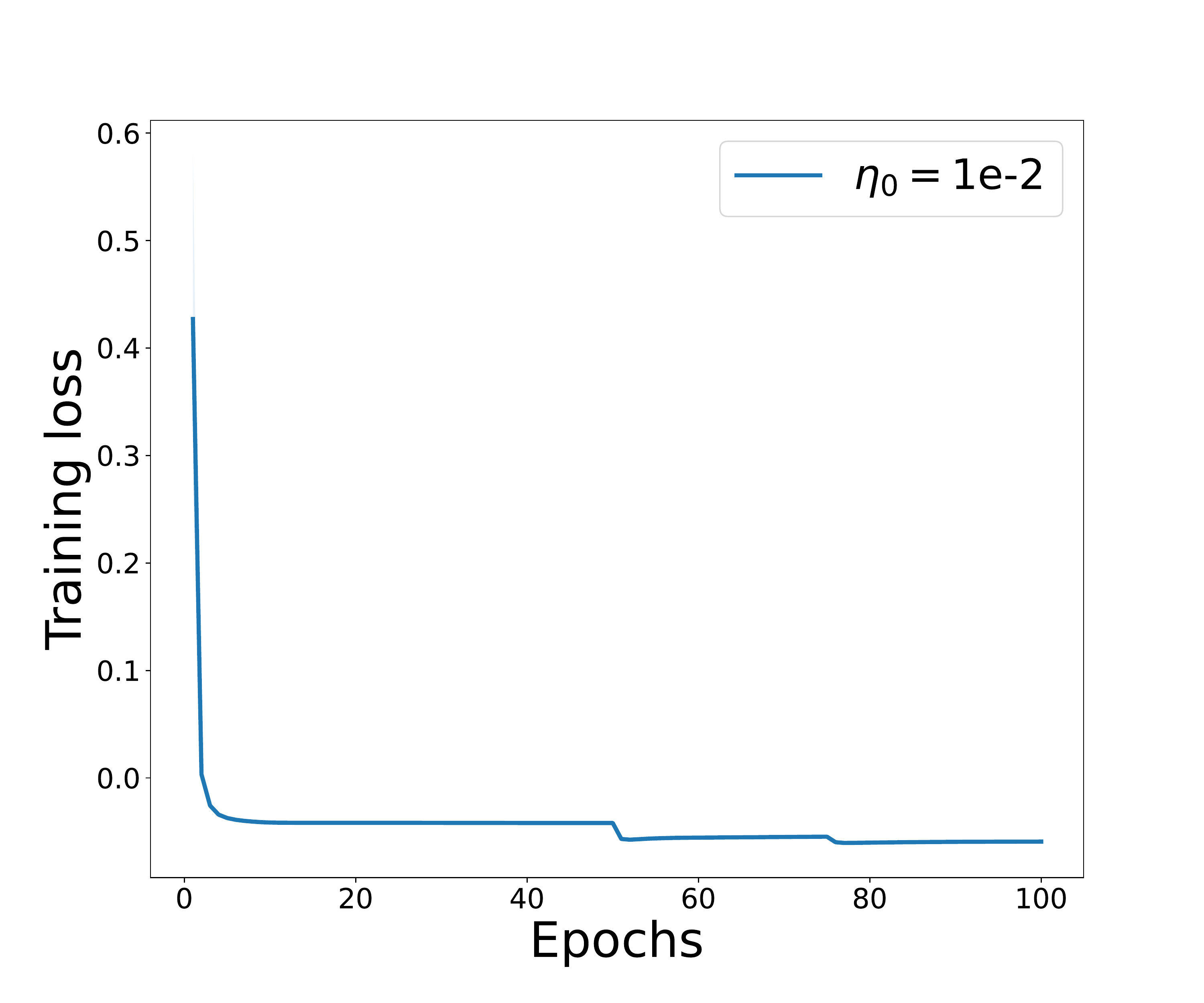}}
\subfigure[CD(0.5), $c=1$]{\includegraphics[scale=0.1]{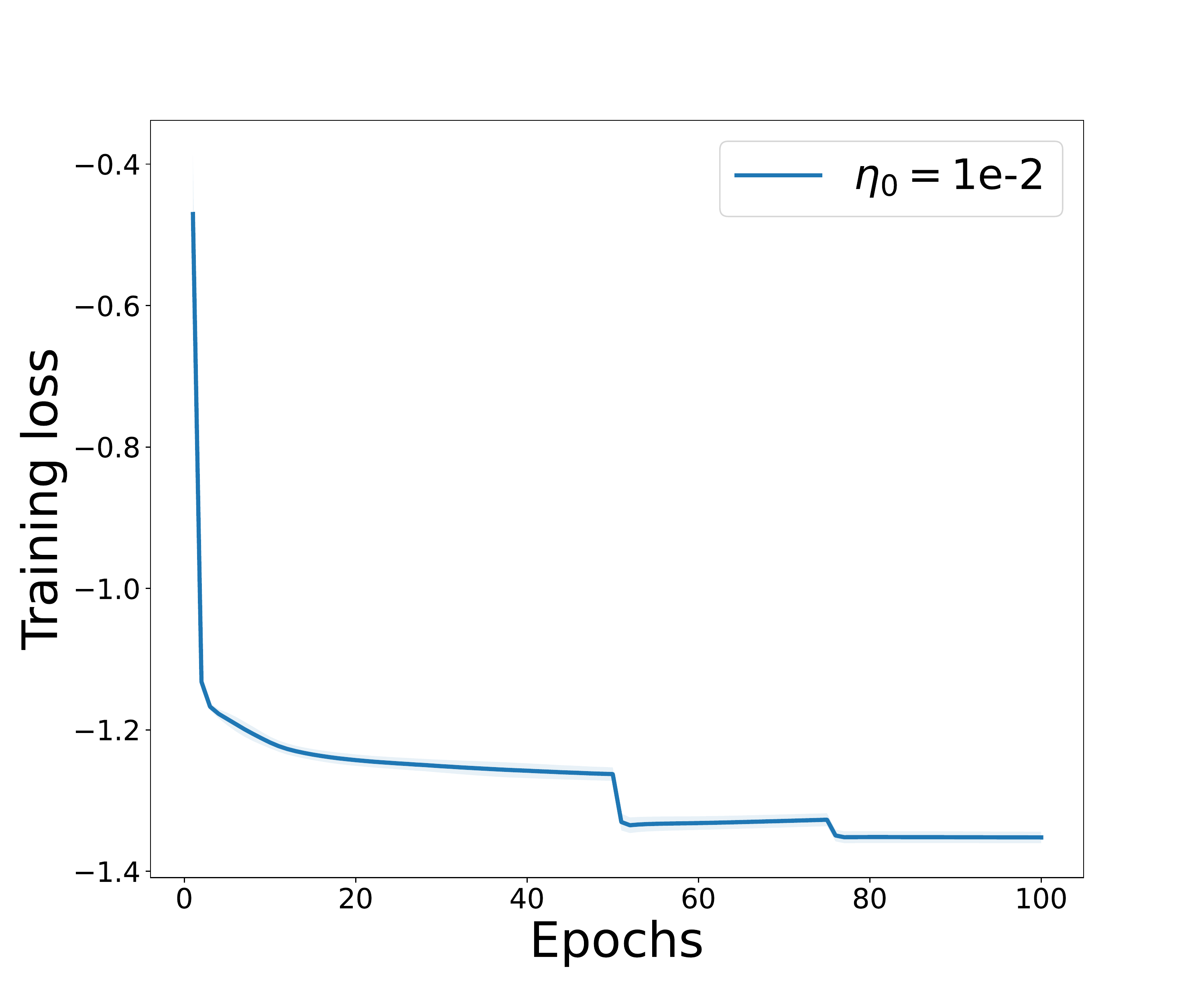}}
\subfigure[CD(0.5), $c=10$]{\includegraphics[scale=0.1]{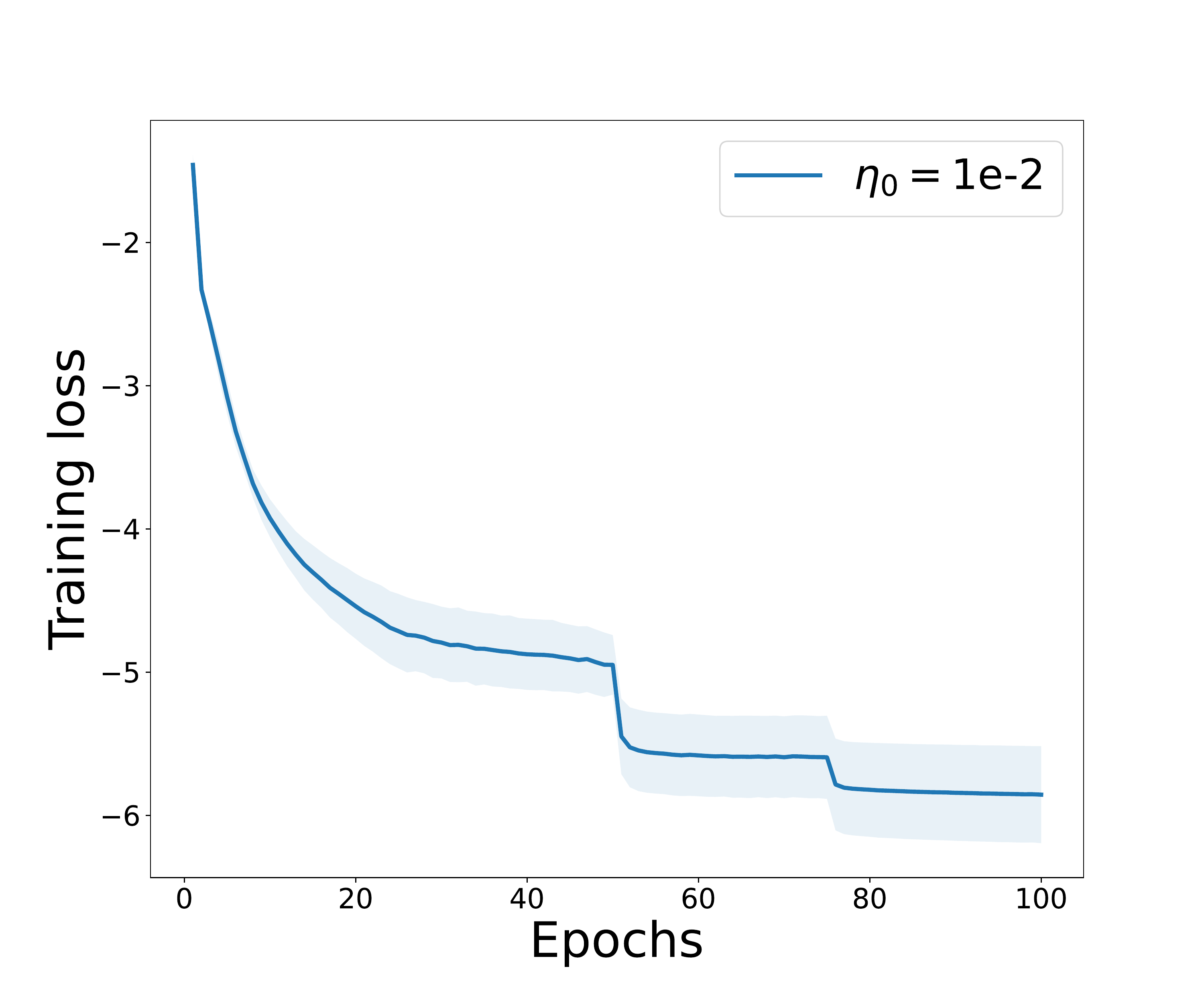}}

\subfigure[U(0.3), $c=0.1$]{\includegraphics[scale=0.1]{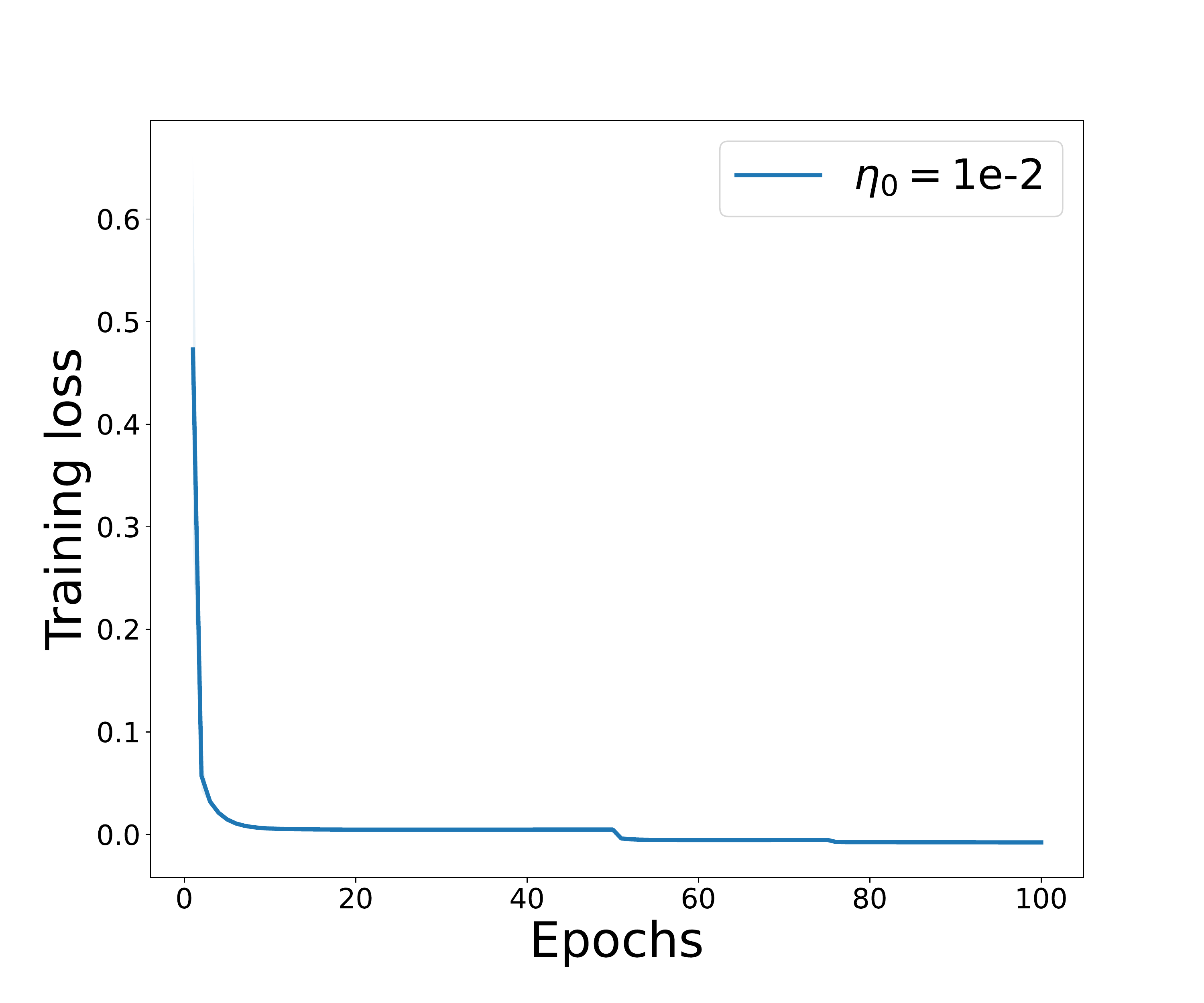}}
\subfigure[U(0.3), $c=1$]{\includegraphics[scale=0.1]{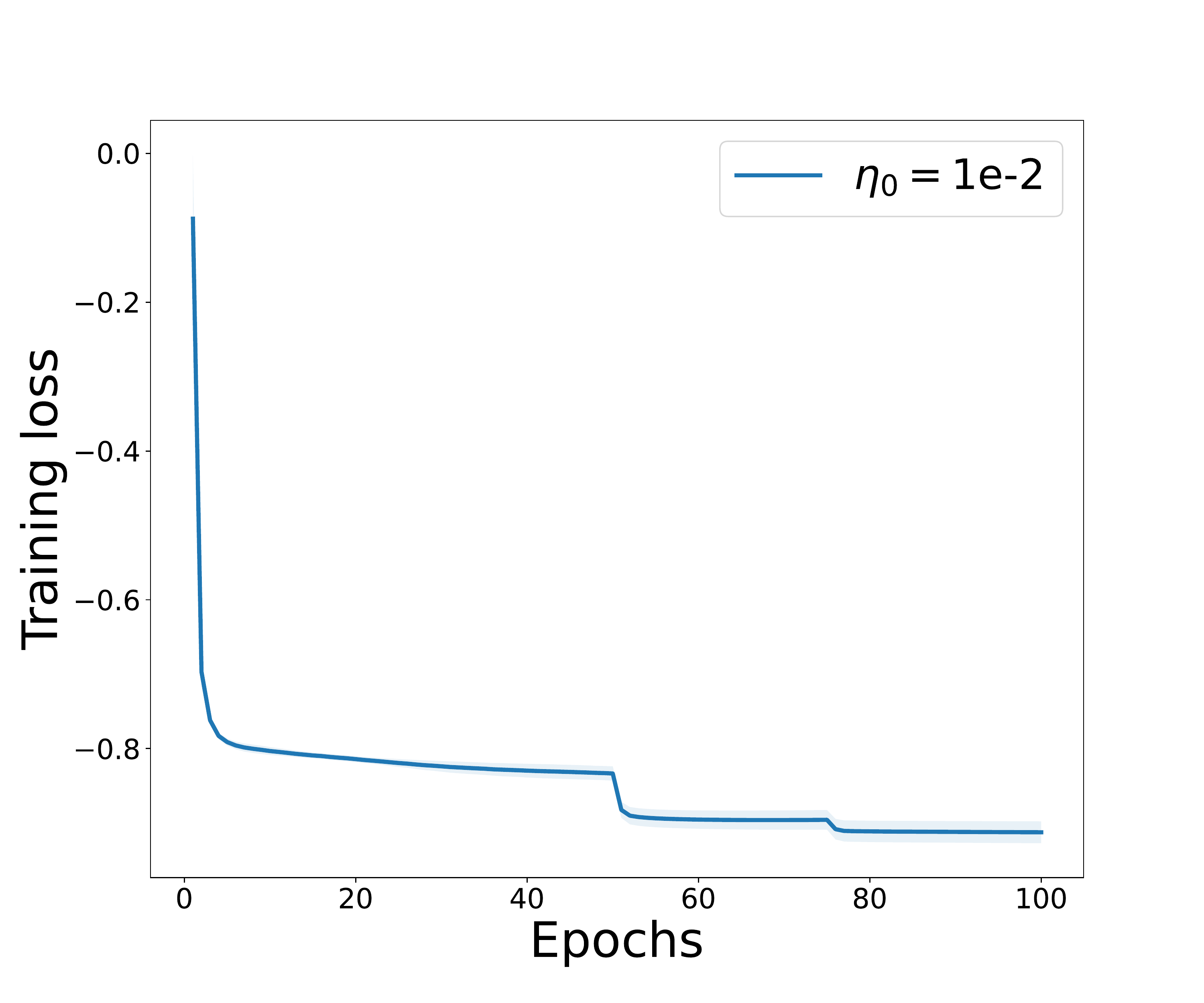}}
\subfigure[U(0.3), $c=10$]{\includegraphics[scale=0.1]{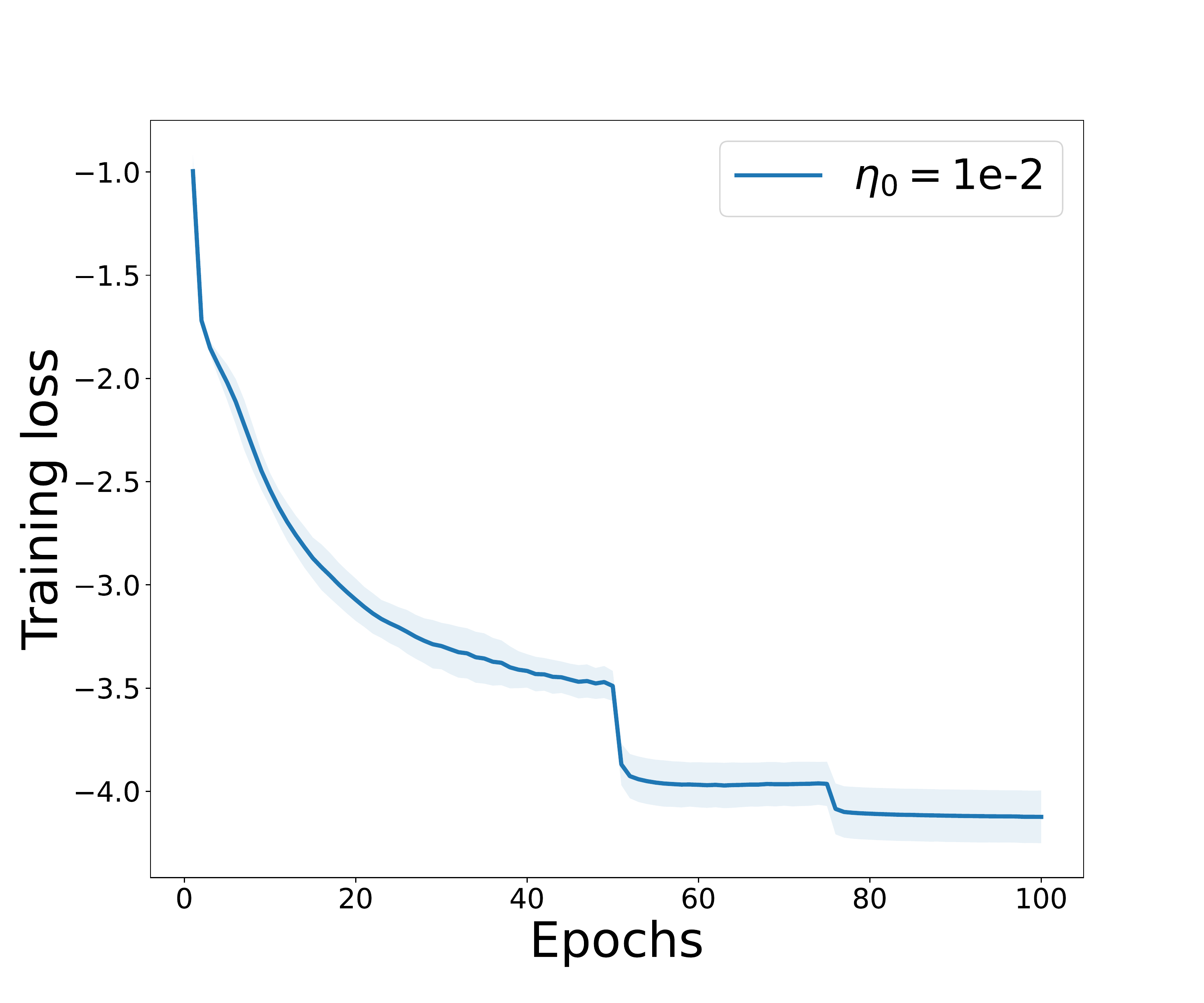}}

\subfigure[U(0.6), $c=0.1$]{\includegraphics[scale=0.1]{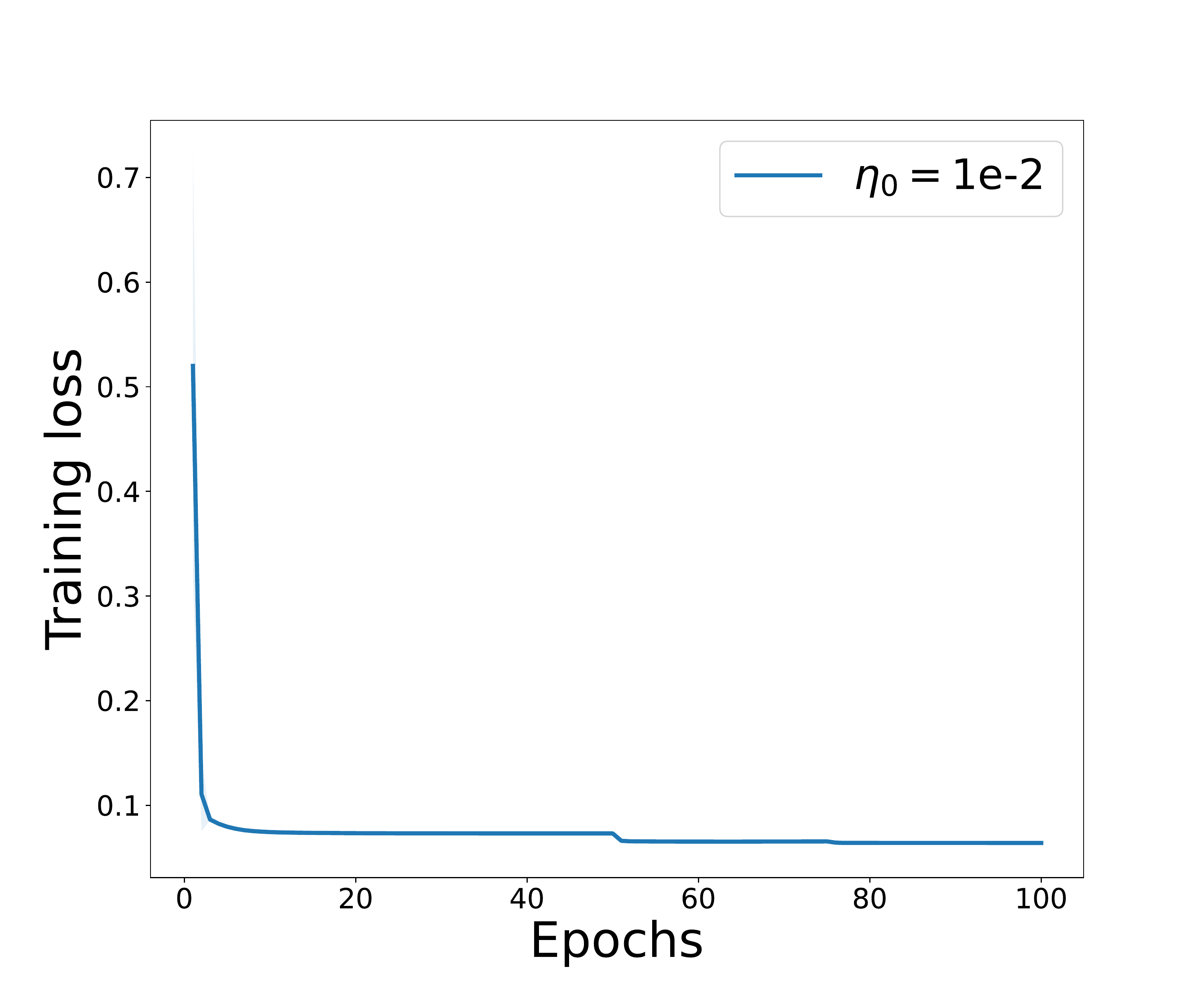}}
\subfigure[U(0.6), $c=1$]{\includegraphics[scale=0.1]{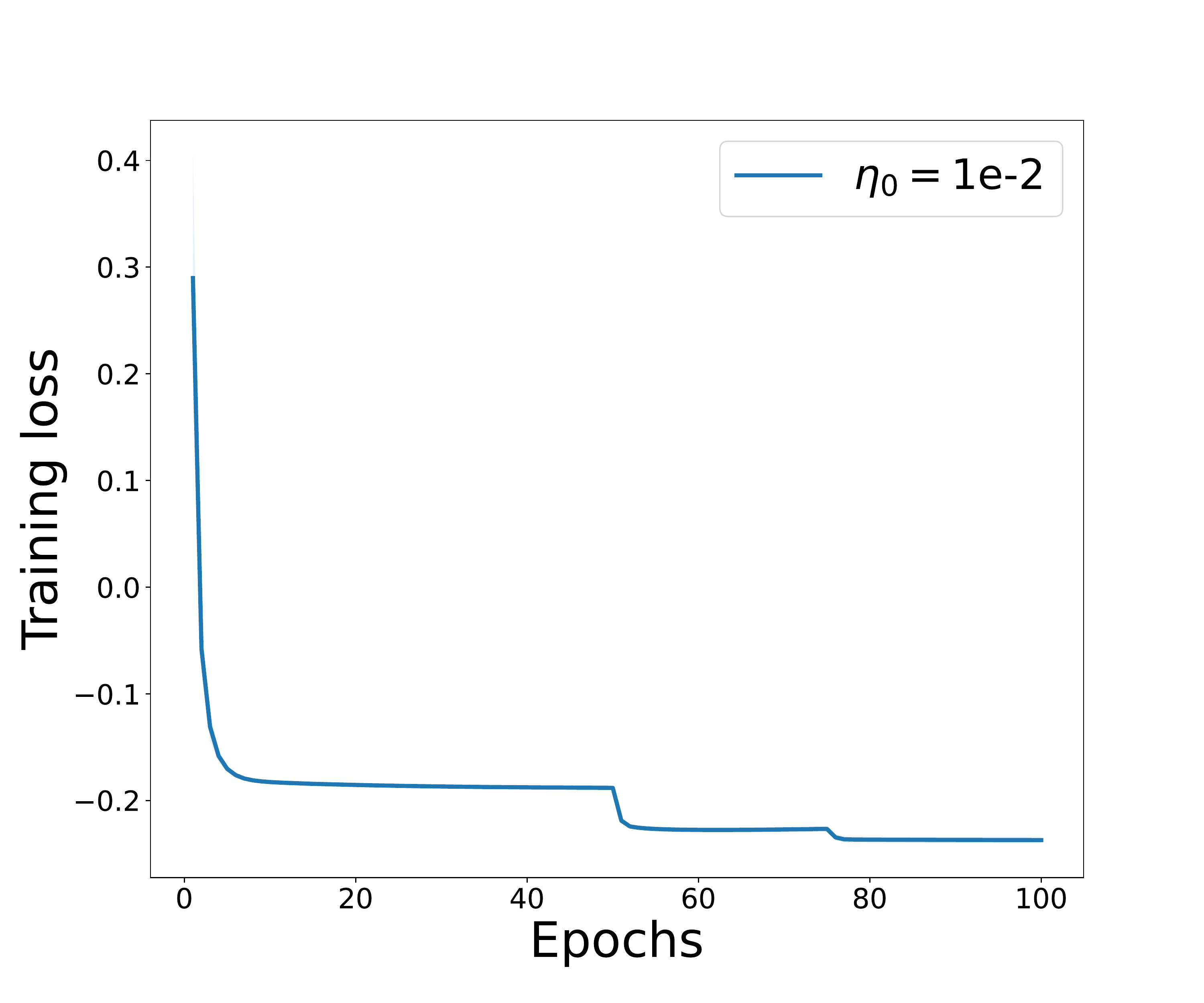}}
\subfigure[U(0.6), $c=10$]{\includegraphics[scale=0.1]{figures/aldr-letter-uni06-cvg-top-m1.pdf}}

\subfigure[U(0.9), $c=0.1$]{\includegraphics[scale=0.1]{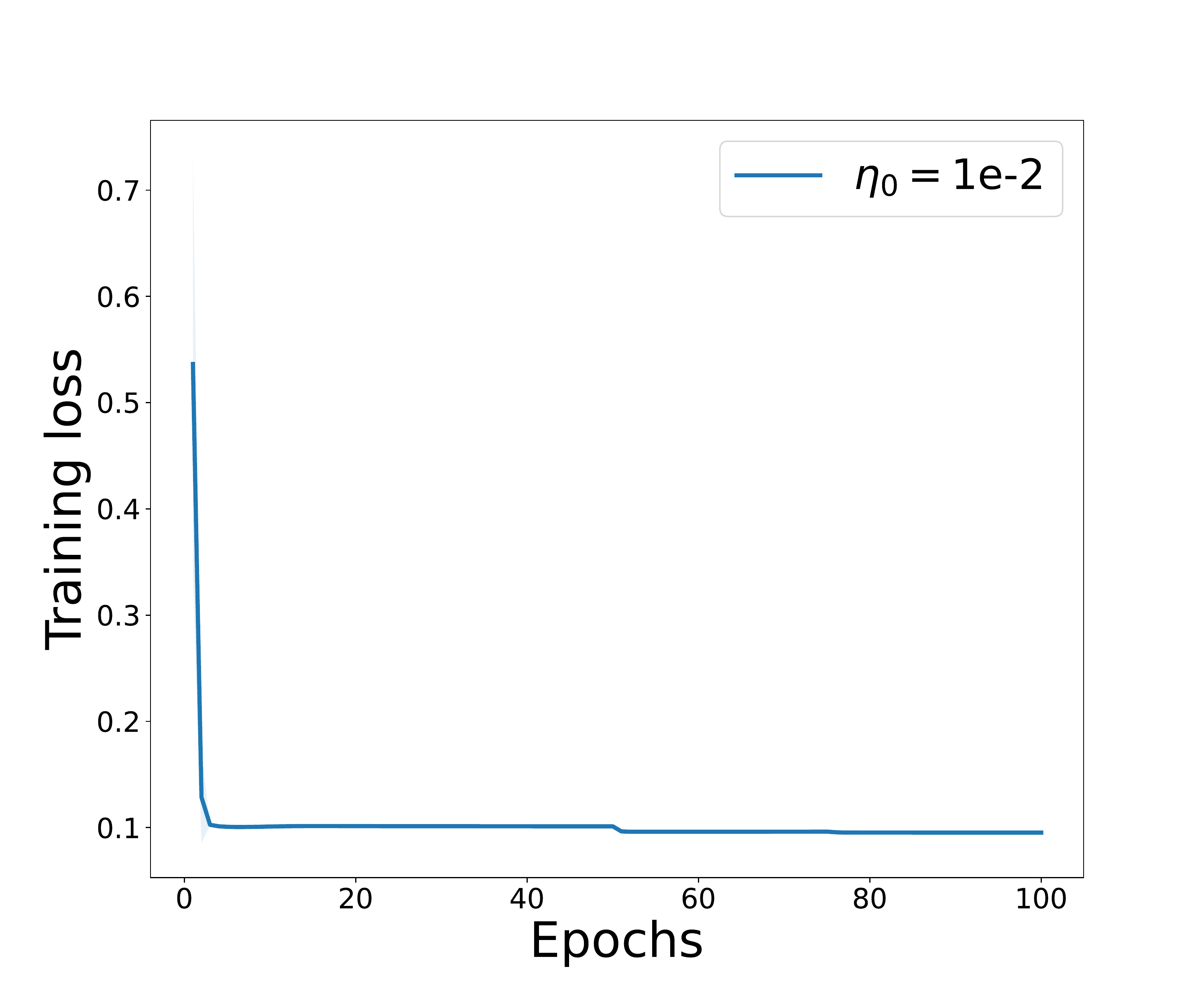}}
\subfigure[U(0.9), $c=1$]{\includegraphics[scale=0.1]{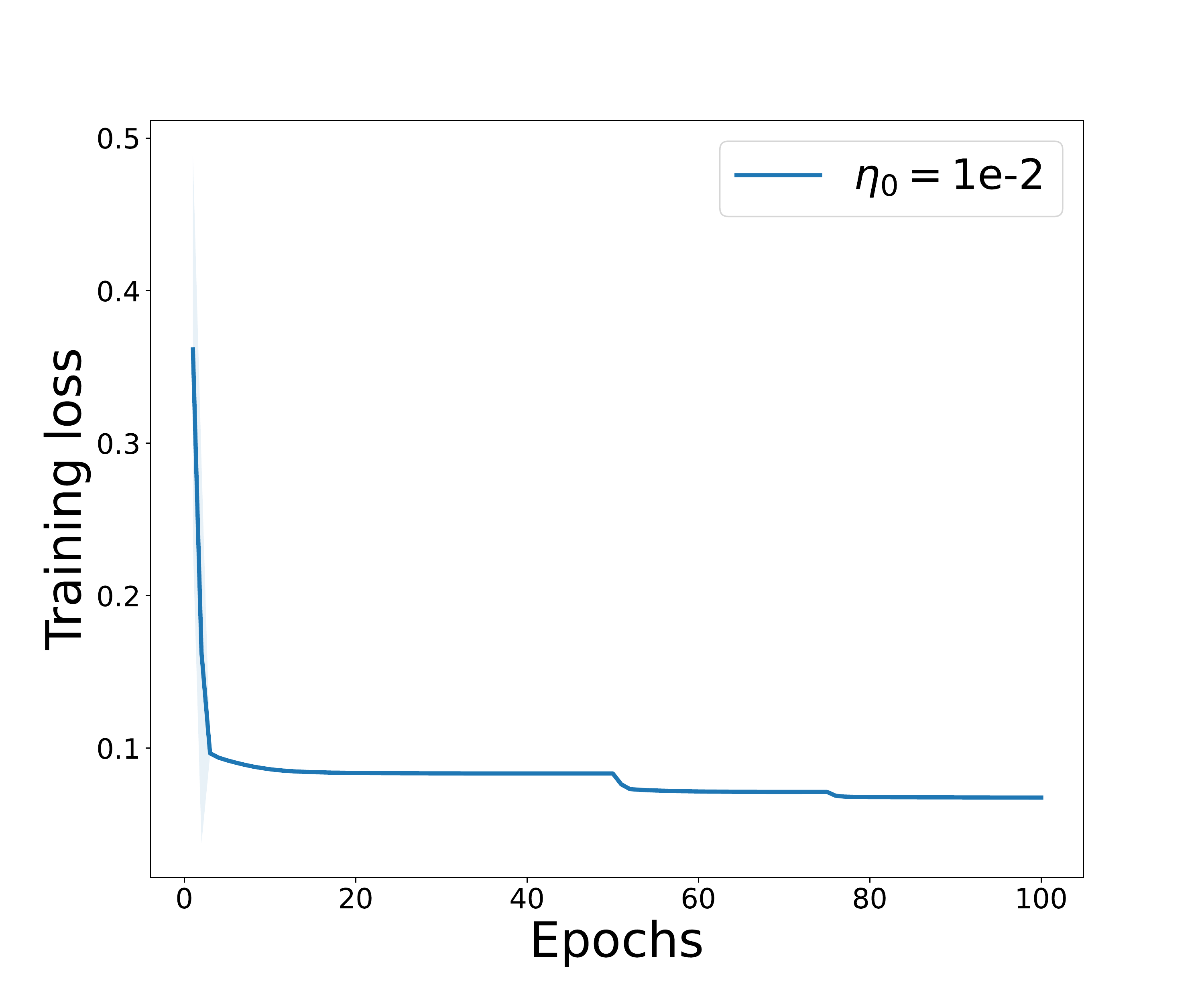}}
\subfigure[U(0.9), $c=10$]{\includegraphics[scale=0.1]{figures/aldr-letter-uni09-cvg-top-m1.pdf}}
\caption{Training loss convergence for ALDR-KL (Algorithm 1) on Letter dataset }\label{fig:aldr-letter-cvg}
\end{figure}

\begin{figure}[p]
\centering

\subfigure[CD(0.1), $c=0.1$]{\includegraphics[scale=0.1]{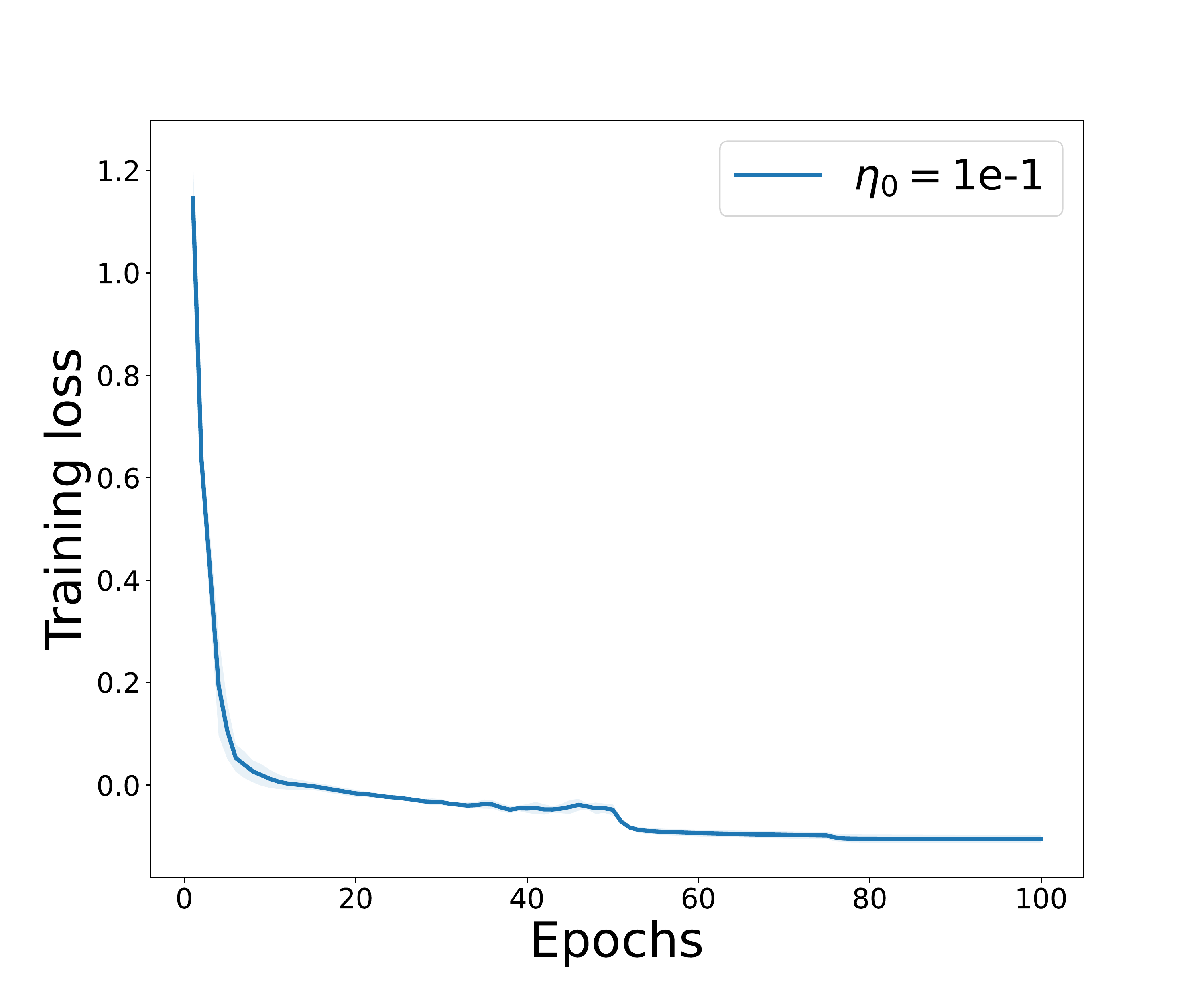}}
\subfigure[CD(0.1), $c=1$]{\includegraphics[scale=0.1]{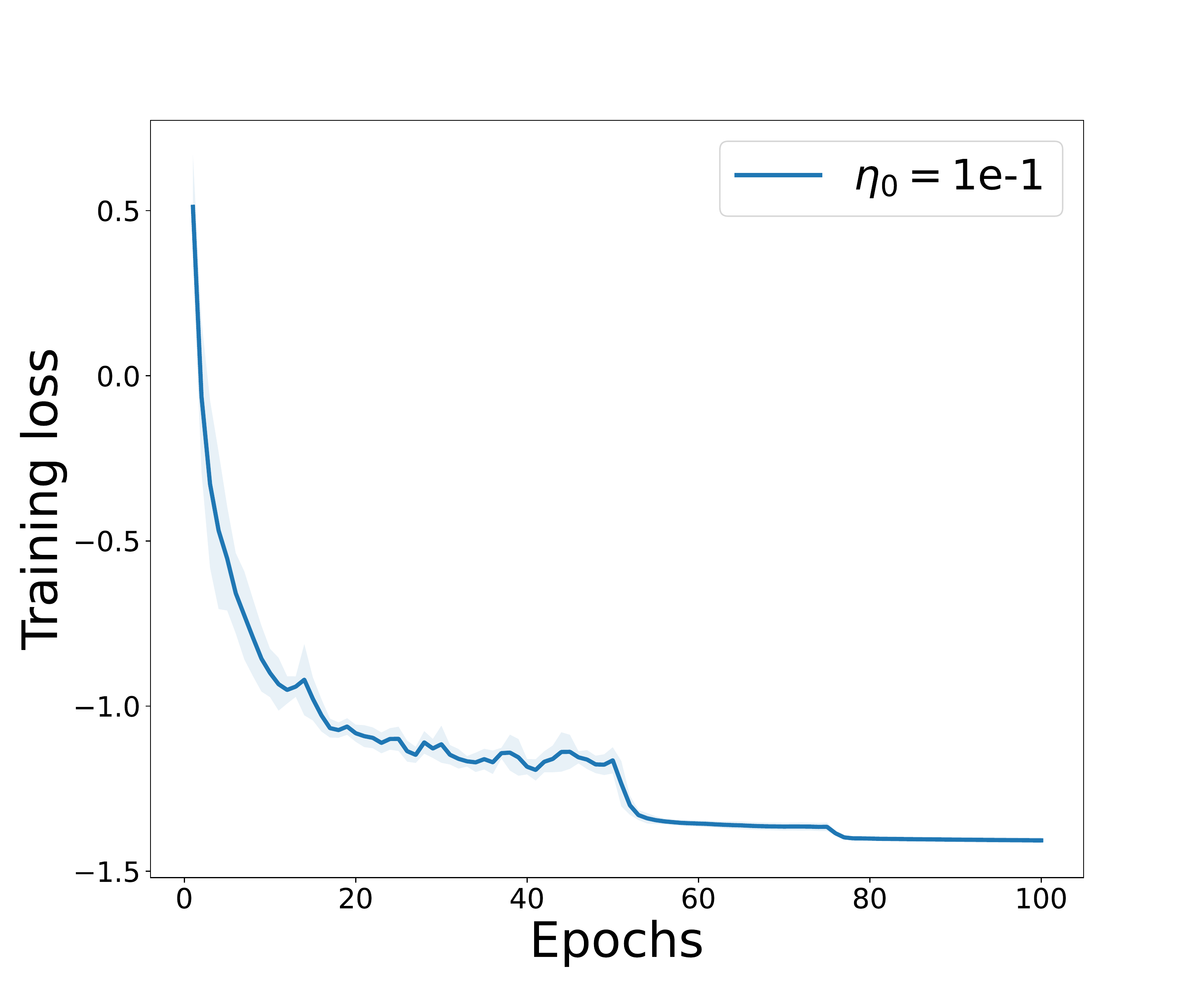}}
\subfigure[CD(0.1), $c=10$]{\includegraphics[scale=0.1]{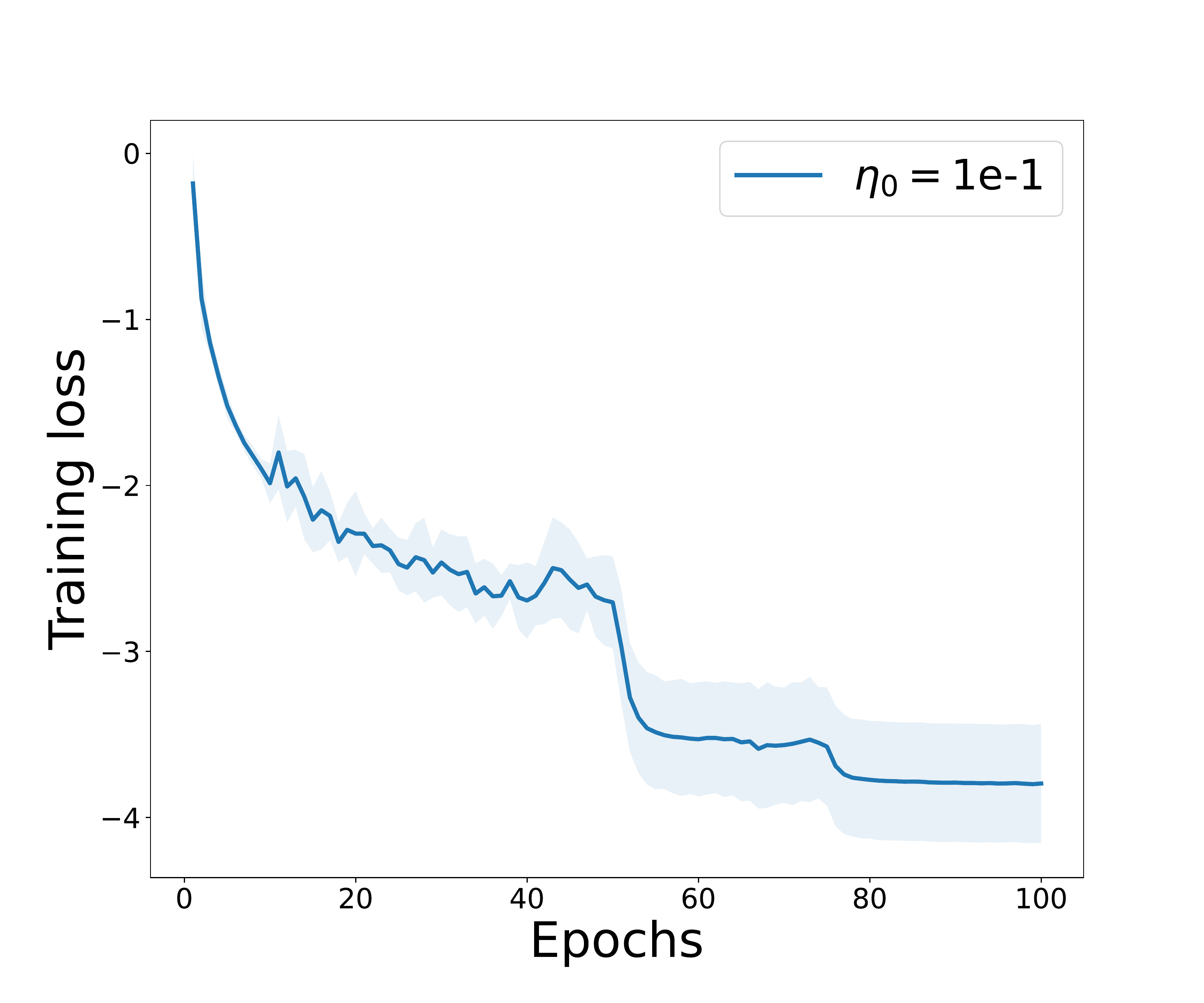}}

\subfigure[CD(0.3), $c=0.1$]{\includegraphics[scale=0.1]{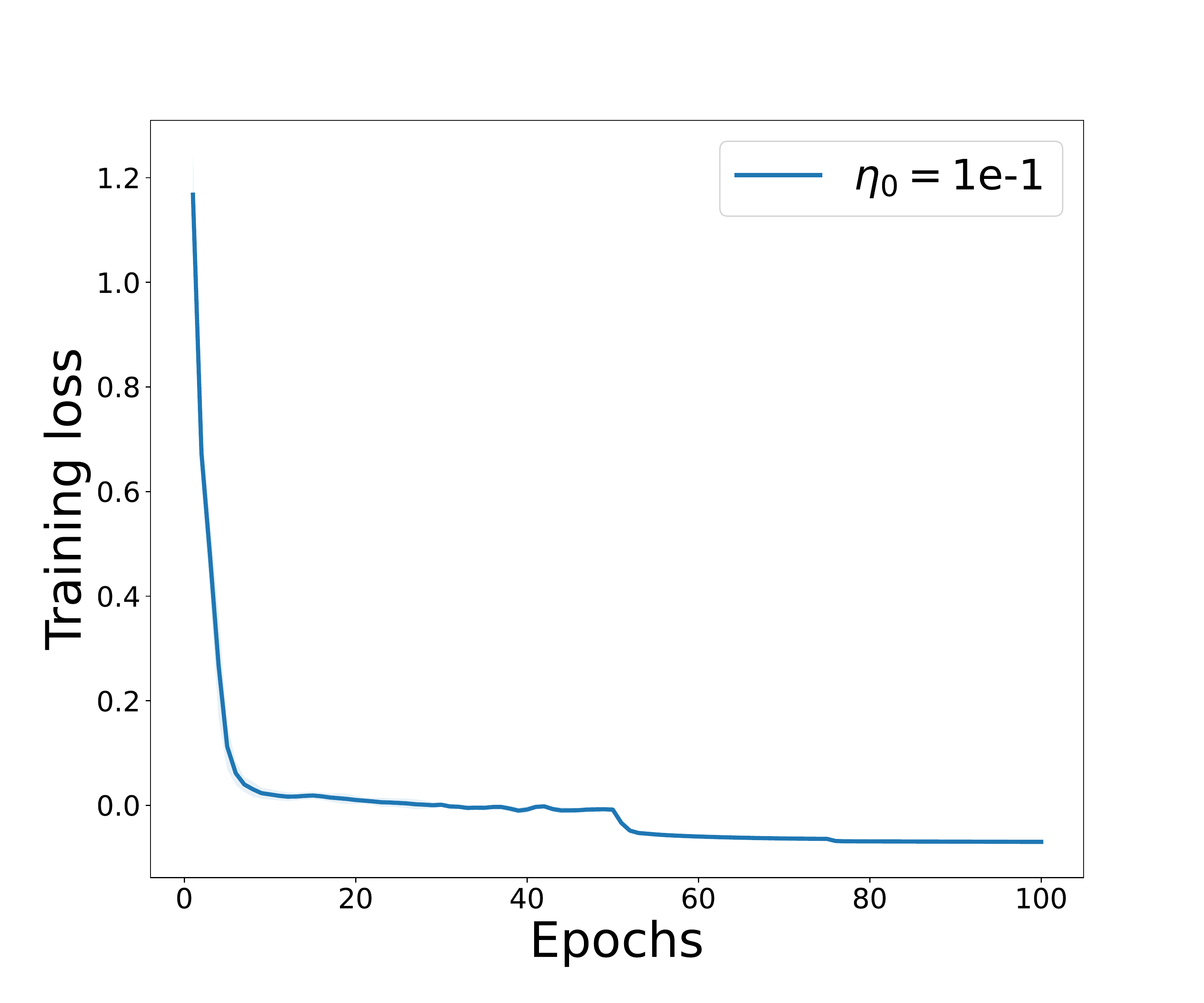}}
\subfigure[CD(0.3), $c=1$]{\includegraphics[scale=0.1]{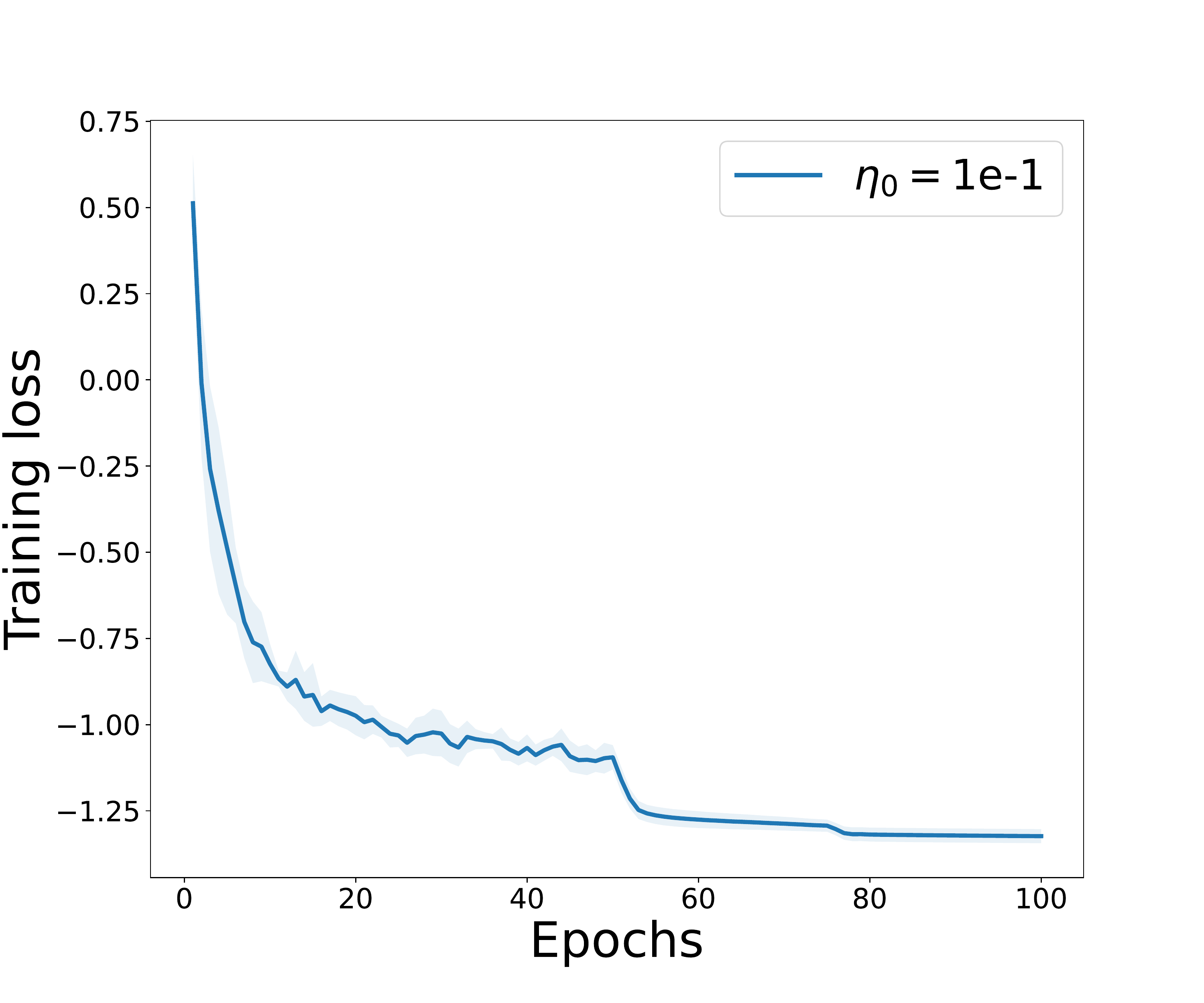}}
\subfigure[CD(0.3), $c=10$]{\includegraphics[scale=0.1]{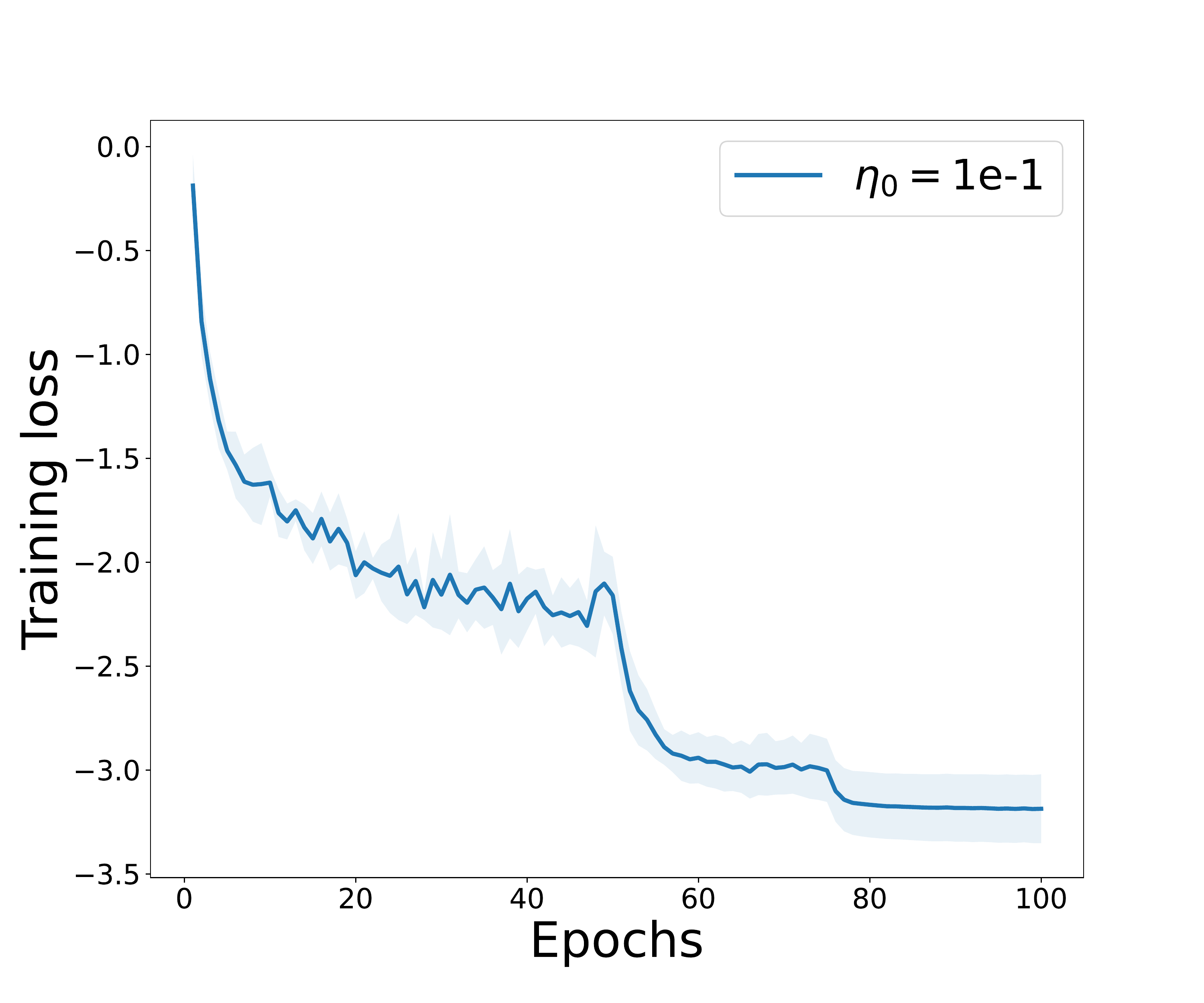}}

\subfigure[CD(0.5), $c=0.1$]{\includegraphics[scale=0.1]{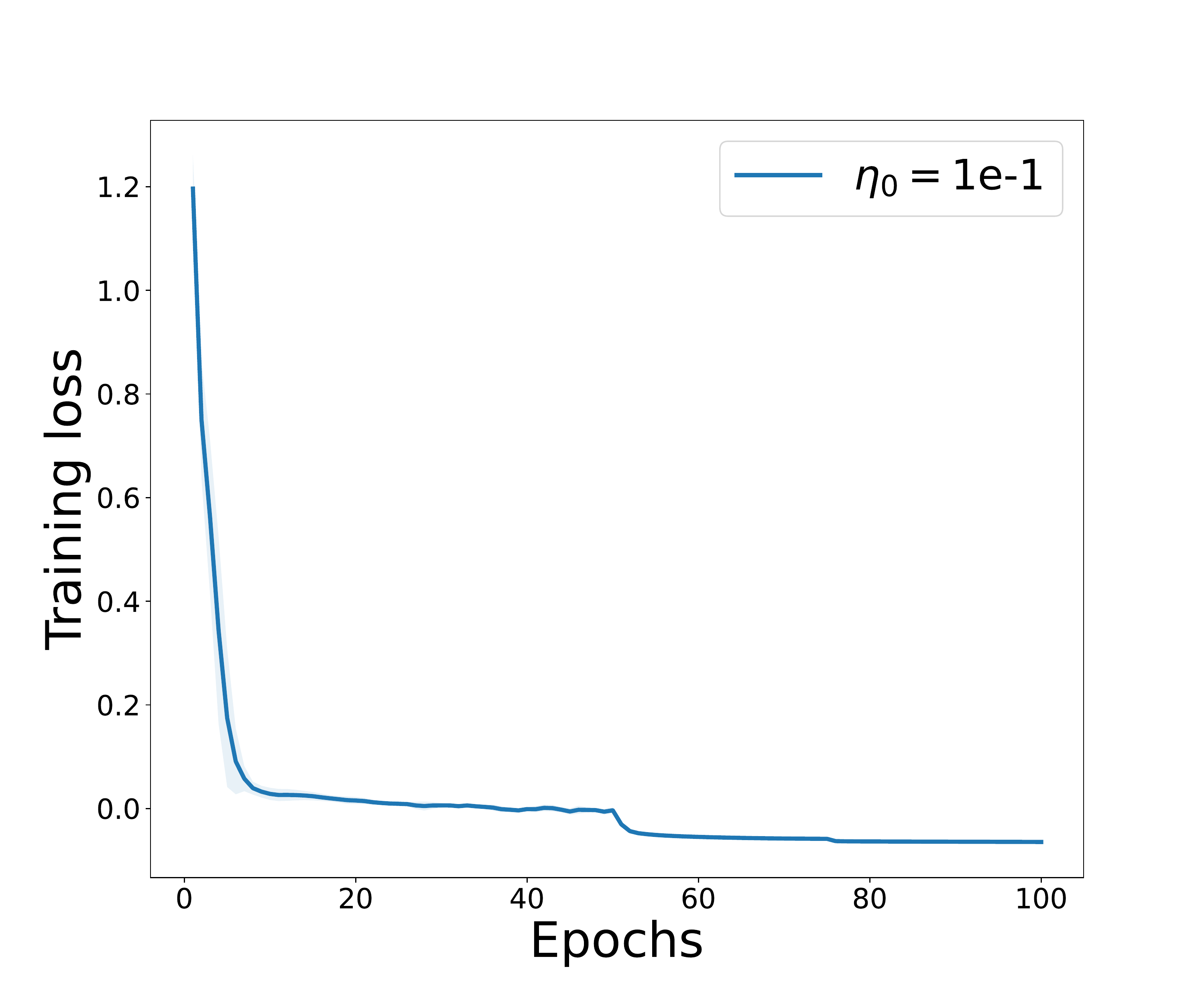}}
\subfigure[CD(0.5), $c=1$]{\includegraphics[scale=0.1]{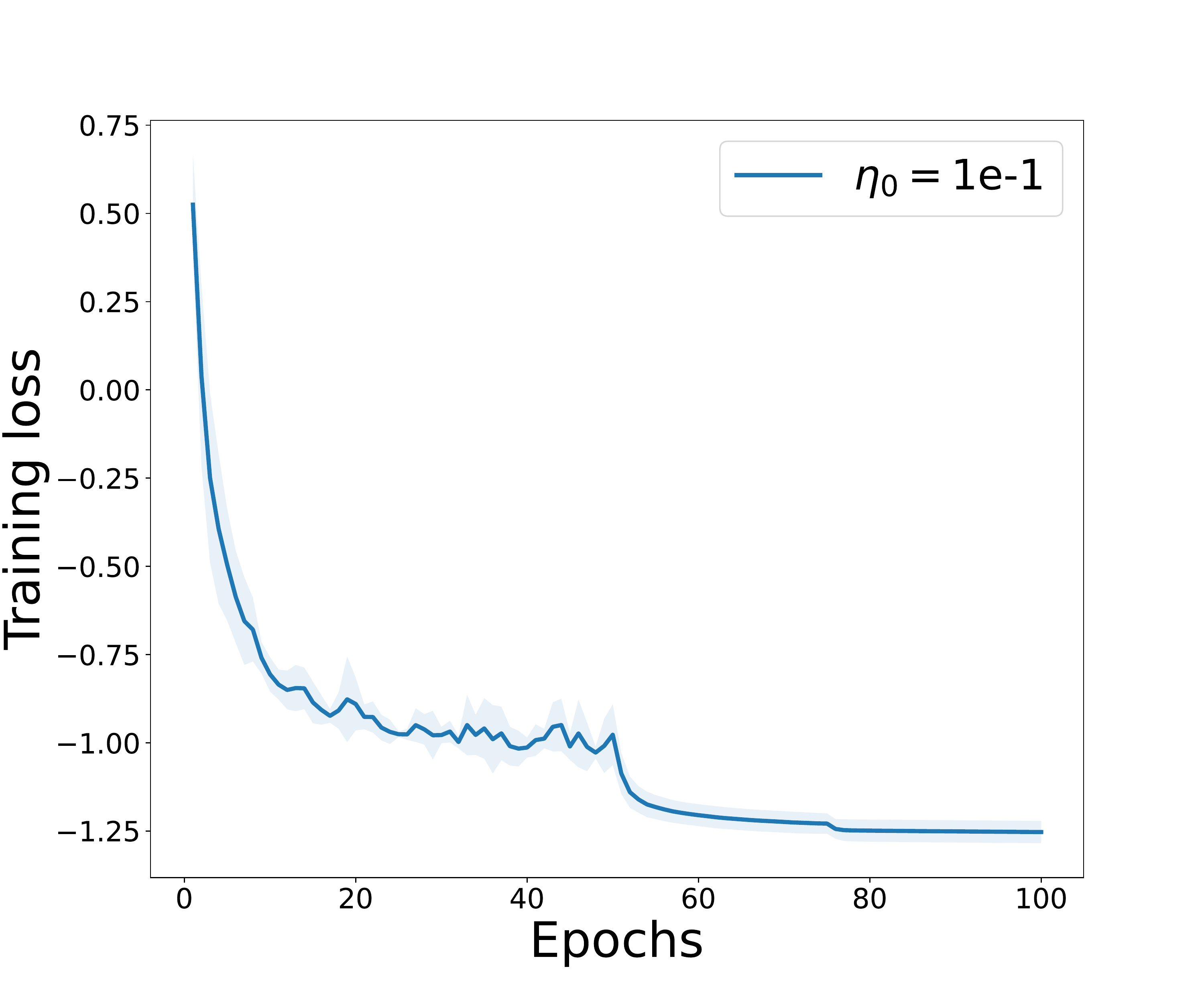}}
\subfigure[CD(0.5), $c=10$]{\includegraphics[scale=0.1]{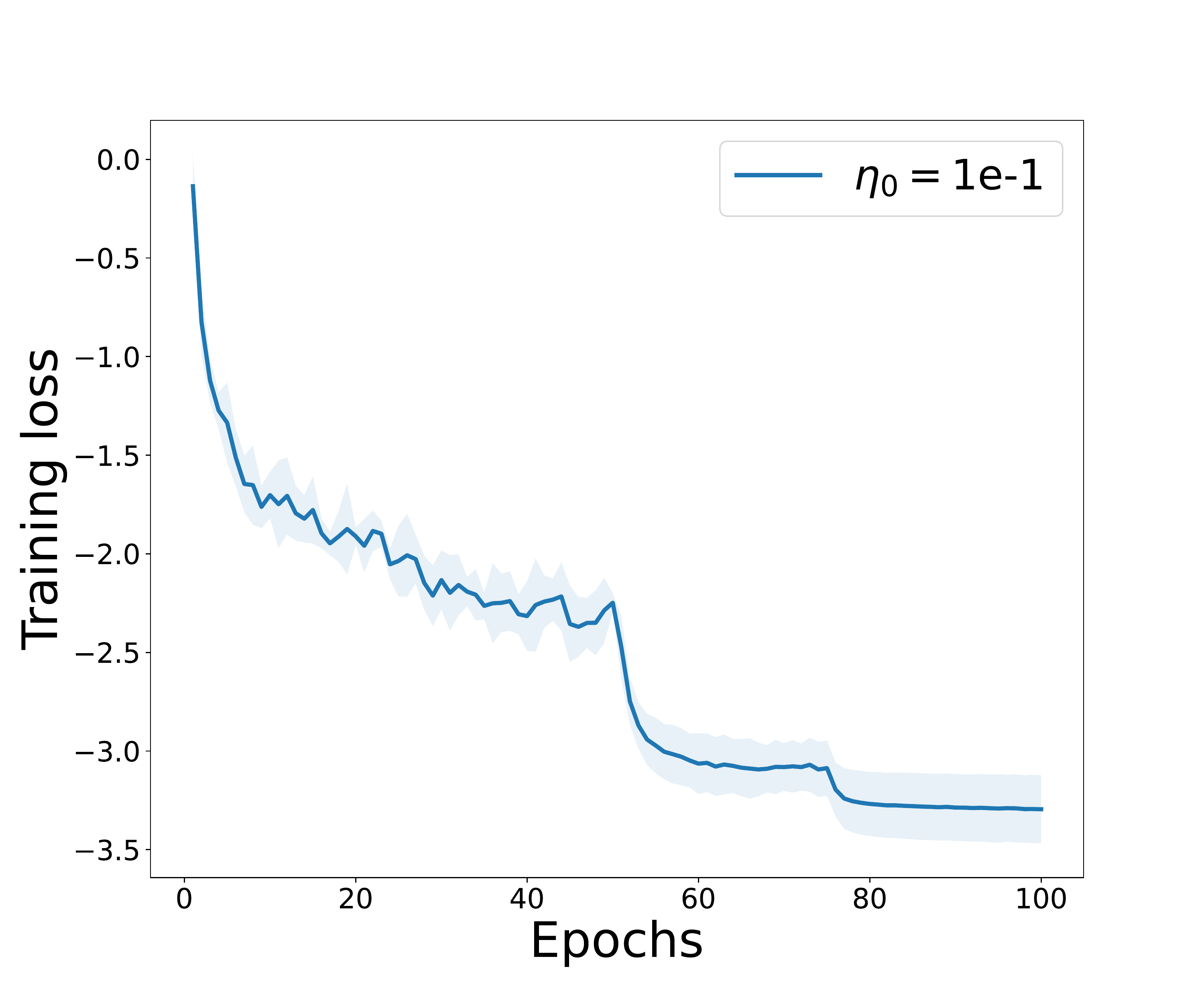}}

\subfigure[U(0.3), $c=0.1$]{\includegraphics[scale=0.1]{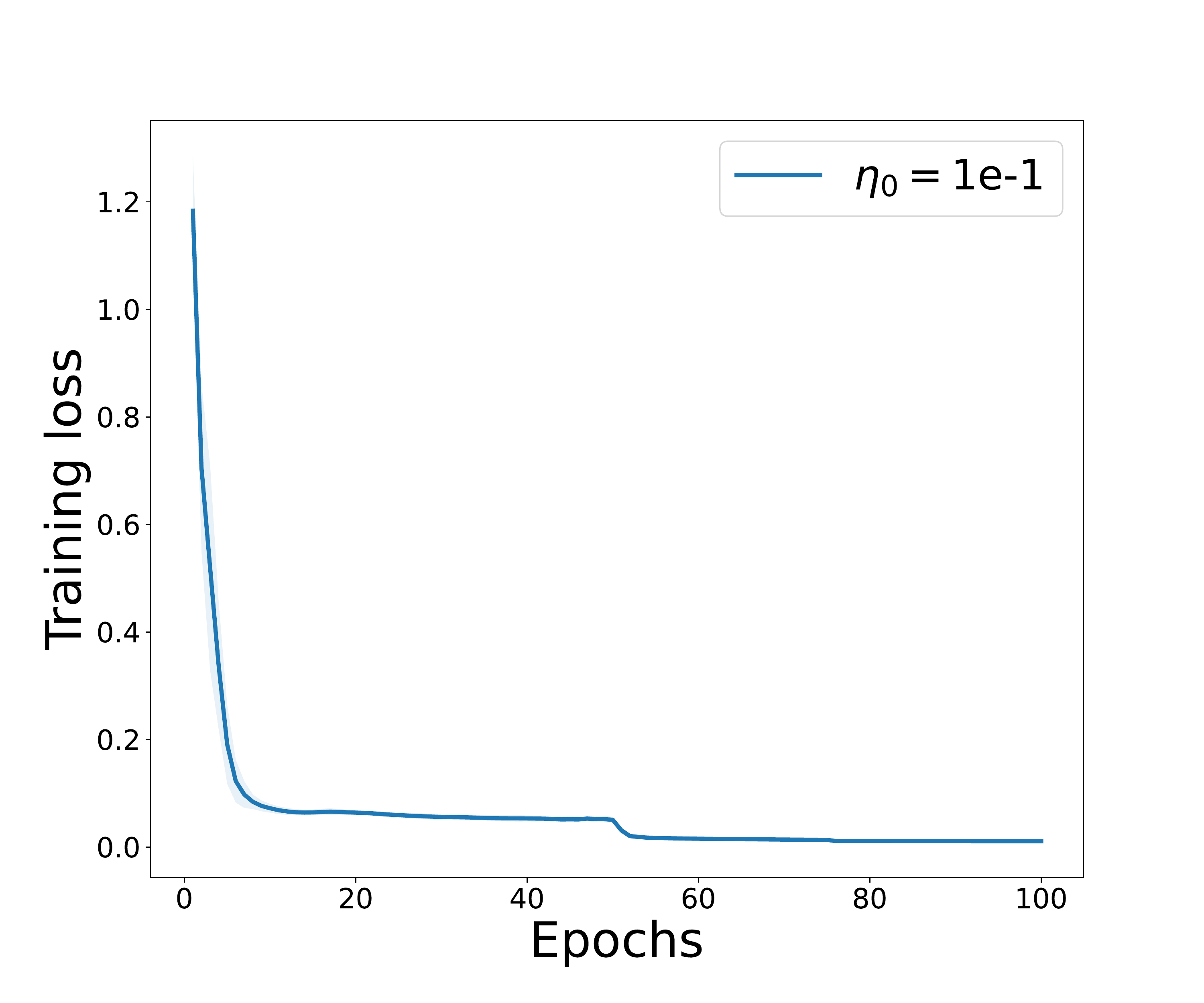}}
\subfigure[U(0.3), $c=1$]{\includegraphics[scale=0.1]{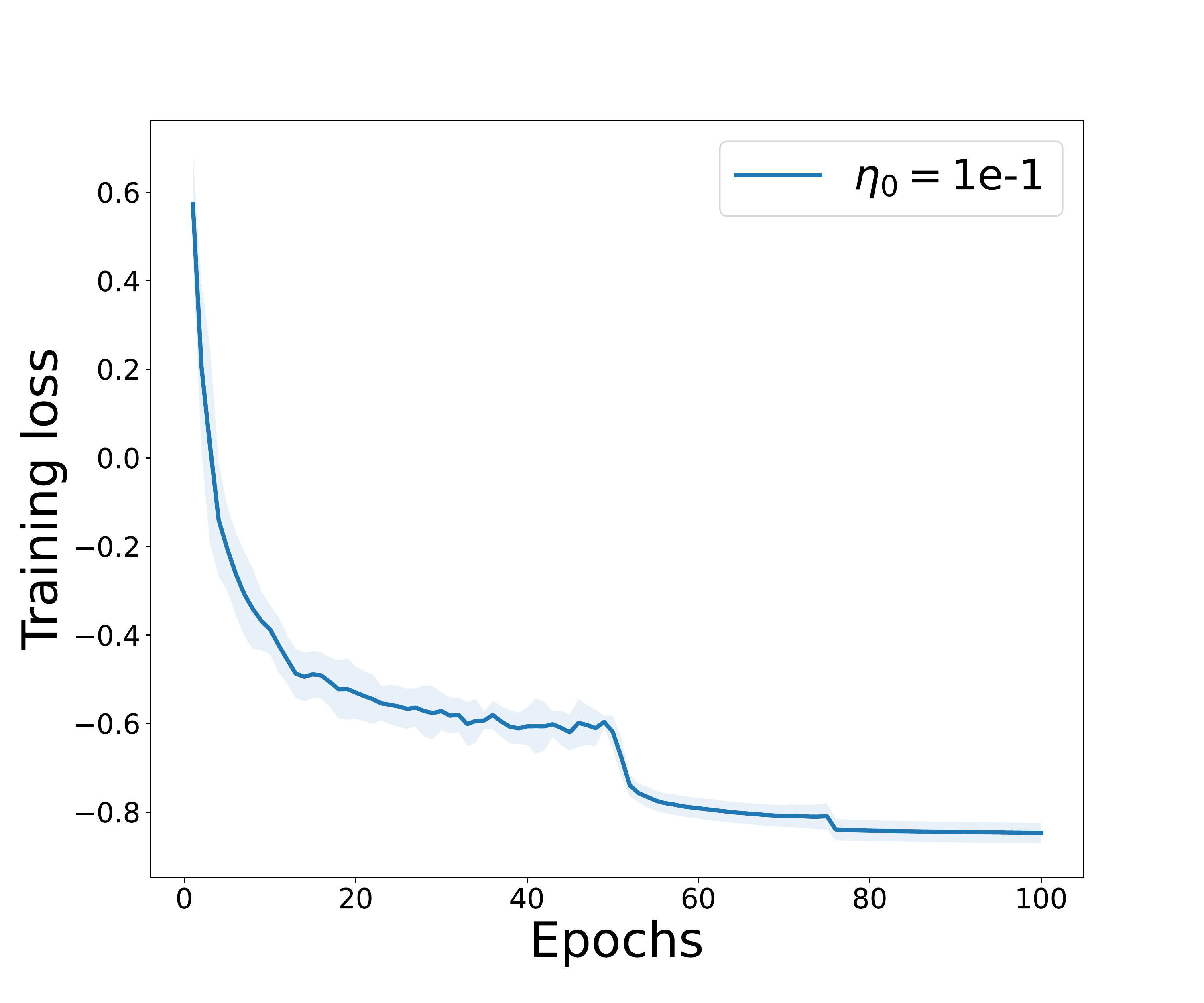}}
\subfigure[U(0.3), $c=10$]{\includegraphics[scale=0.1]{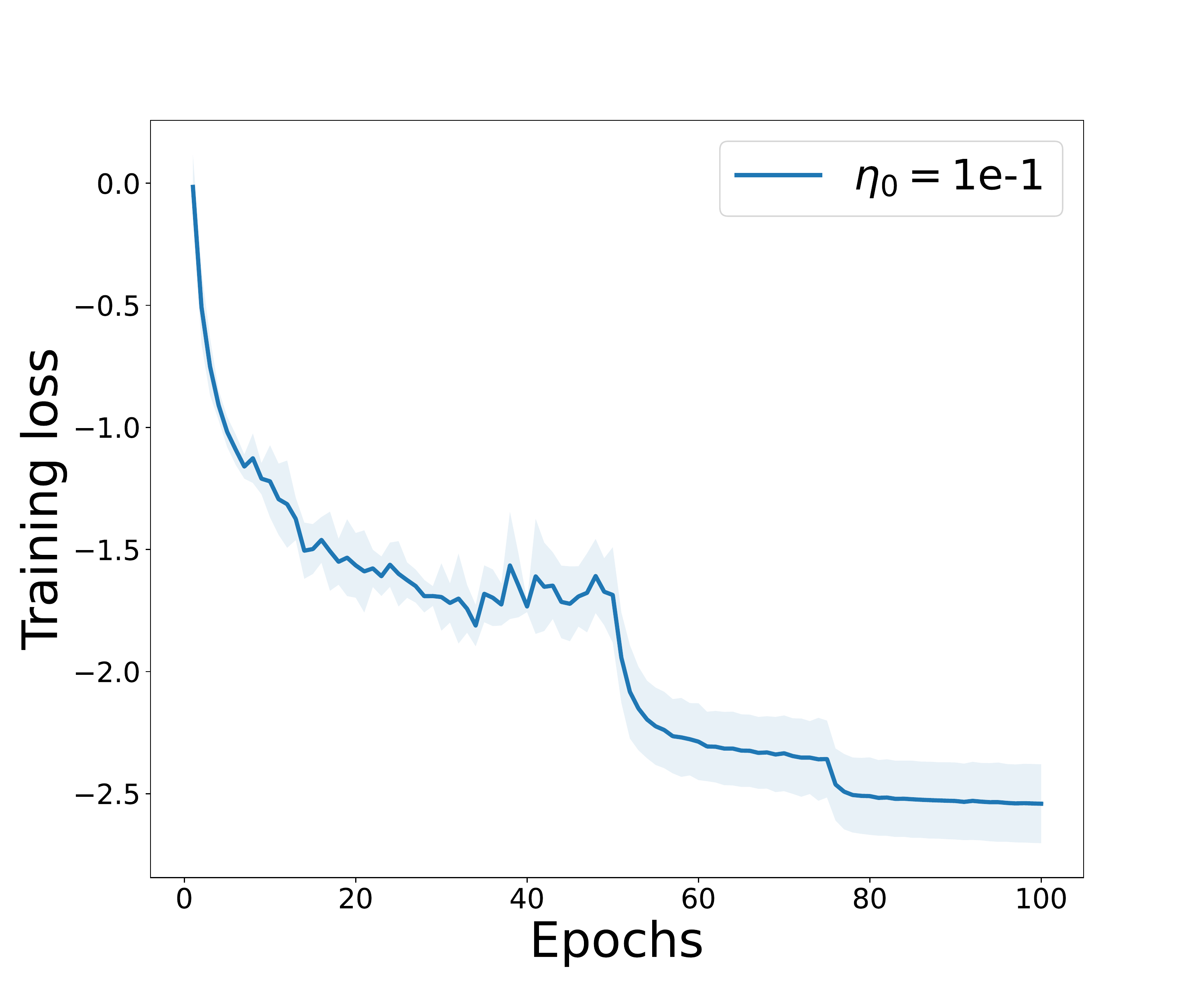}}

\subfigure[U(0.6), $c=0.1$]{\includegraphics[scale=0.1]{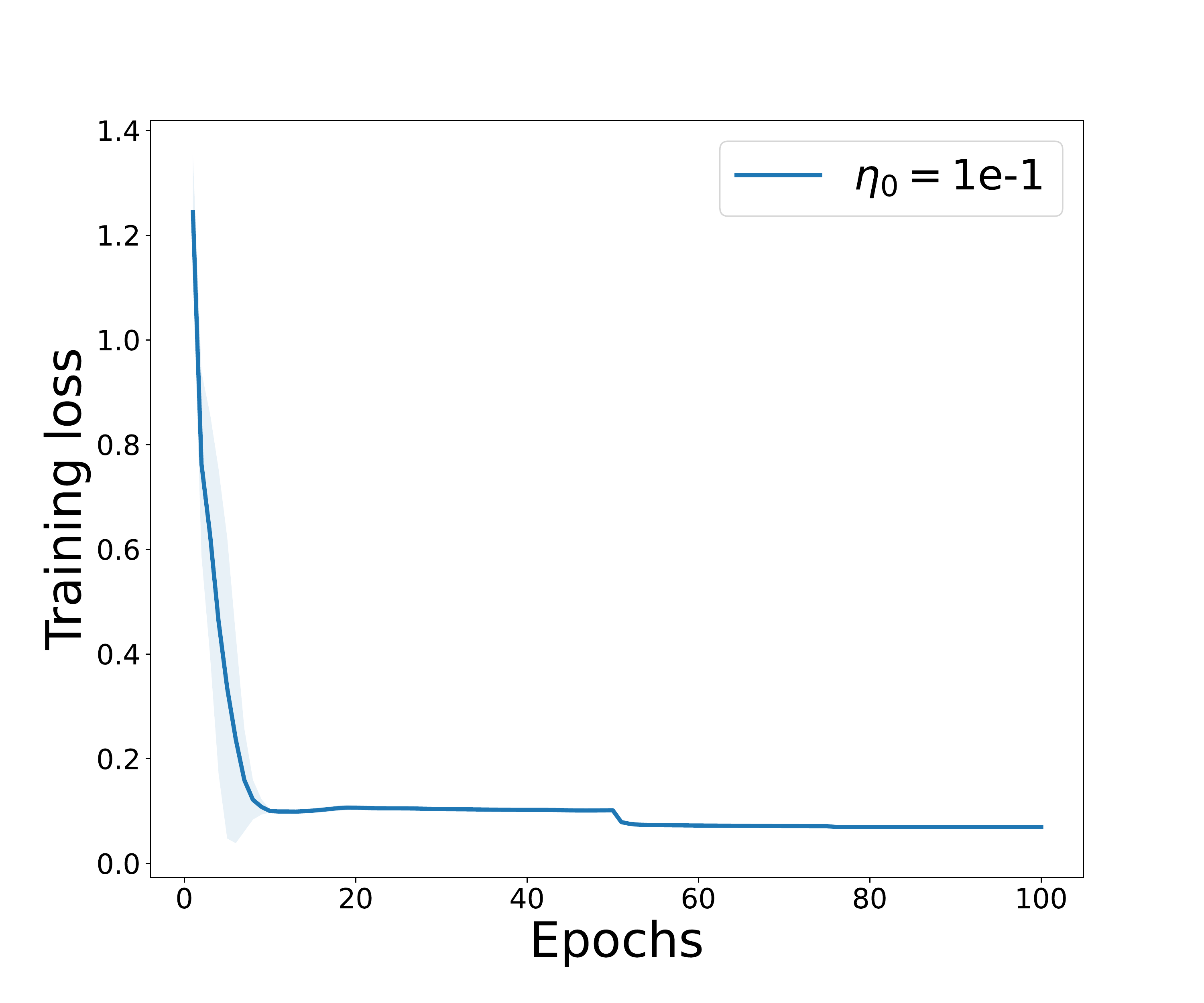}}
\subfigure[U(0.6), $c=1$]{\includegraphics[scale=0.1]{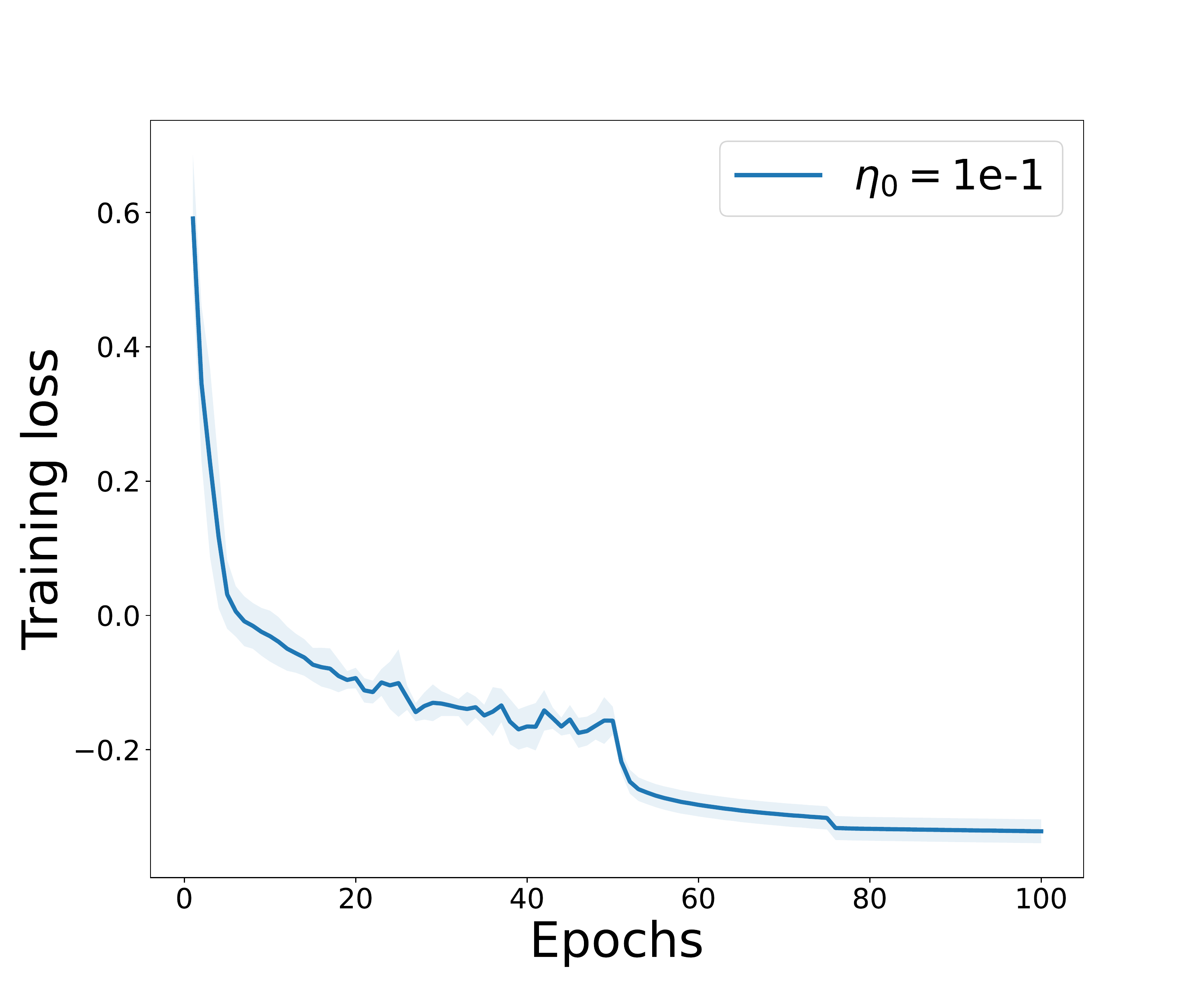}}
\subfigure[U(0.6), $c=10$]{\includegraphics[scale=0.1]{figures/aldr-vowel-uni06-cvg-top-m1.pdf}}

\subfigure[U(0.9), $c=0.1$]{\includegraphics[scale=0.1]{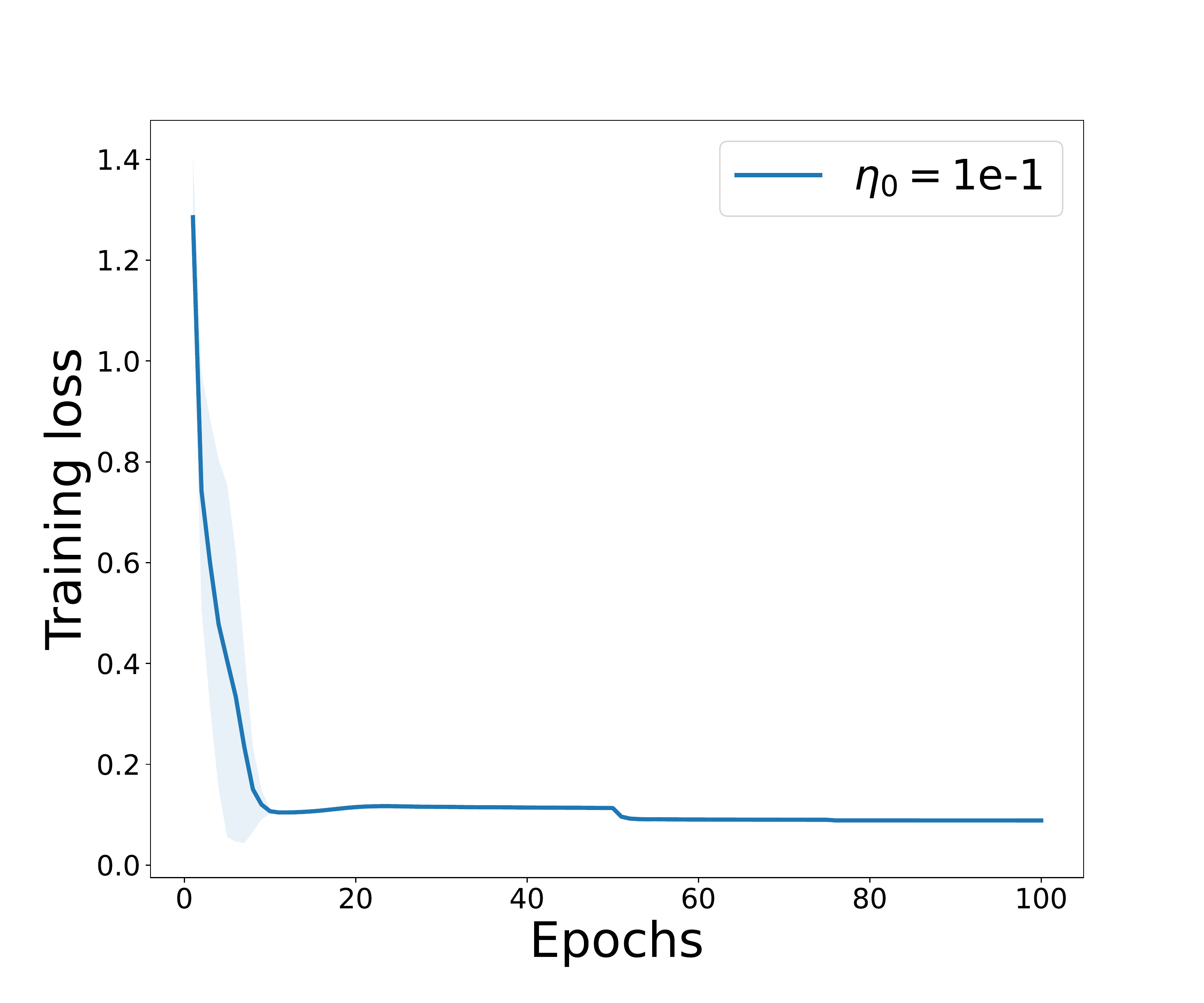}}
\subfigure[U(0.9), $c=1$]{\includegraphics[scale=0.1]{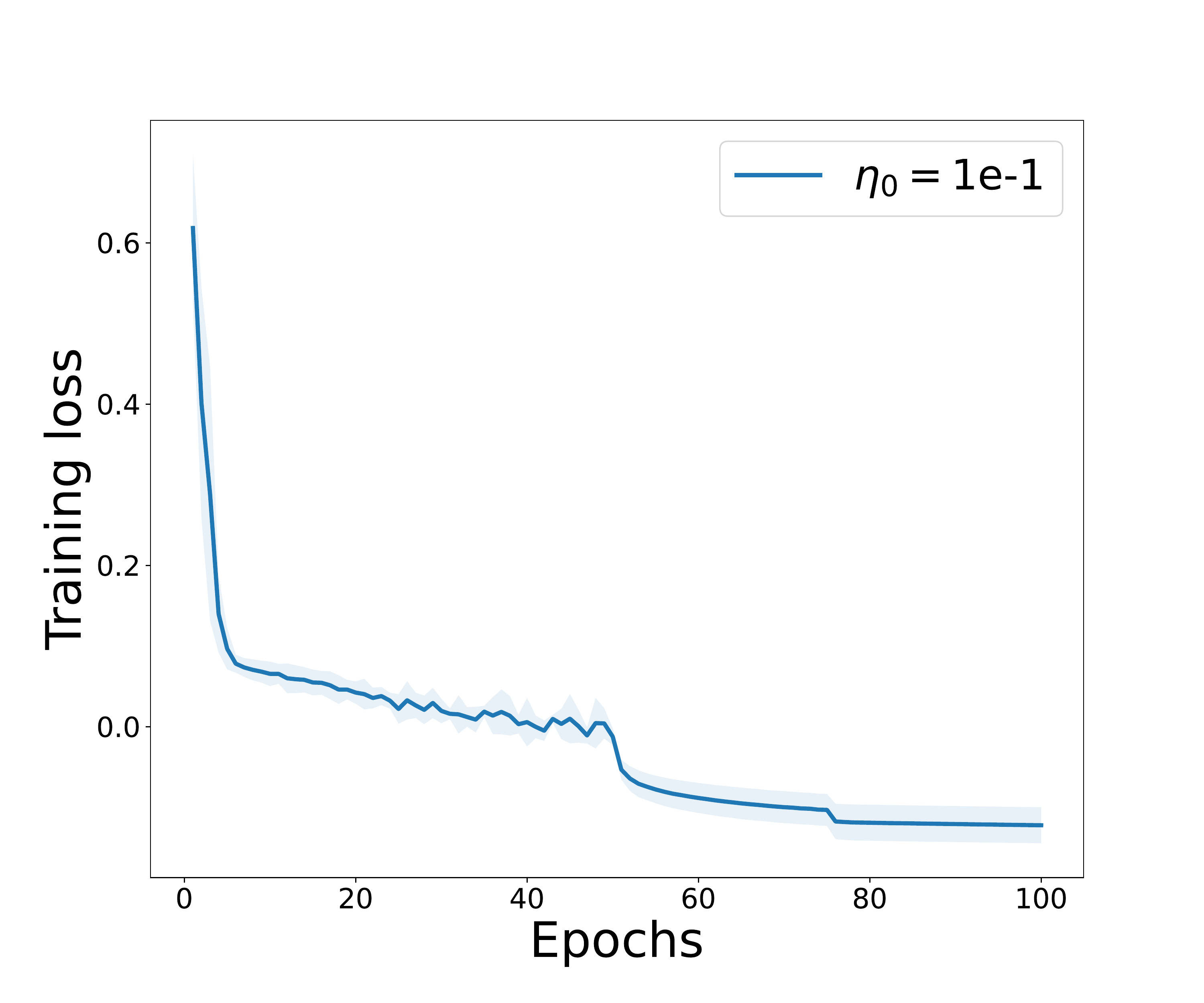}}
\subfigure[U(0.9), $c=10$]{\includegraphics[scale=0.1]{figures/aldr-vowel-uni09-cvg-top-m1.pdf}}
\caption{Training loss convergence for ALDR-KL (Algorithm 1) on Vowel dataset }\label{fig:aldr-vowel-cvg}
\end{figure}






\subsection{More Synthetic Experiments}\label{sec:more-synthetic}
{We additionally provide the synthetic experiments for training from scratch setting in Figure~\ref{fig:syn-scratch}, and LDR-KL loss in Figure~\ref{fig:syn-seq} and Figure~\ref{fig:syn-scratch}. The experimental settings are kept as the same, except for training from scratch setting, the total training epoch is fixed as 200 for all methods. The $\lambda$ parameter for LDR-KL loss is set as the mean value from the optimized $\lambda$ from ALDR-KL loss. From the results, the ALDR-KL loss consistently enjoys the most robust decision boundary.}

\begin{figure}
    \centering
\includegraphics[scale=0.5]{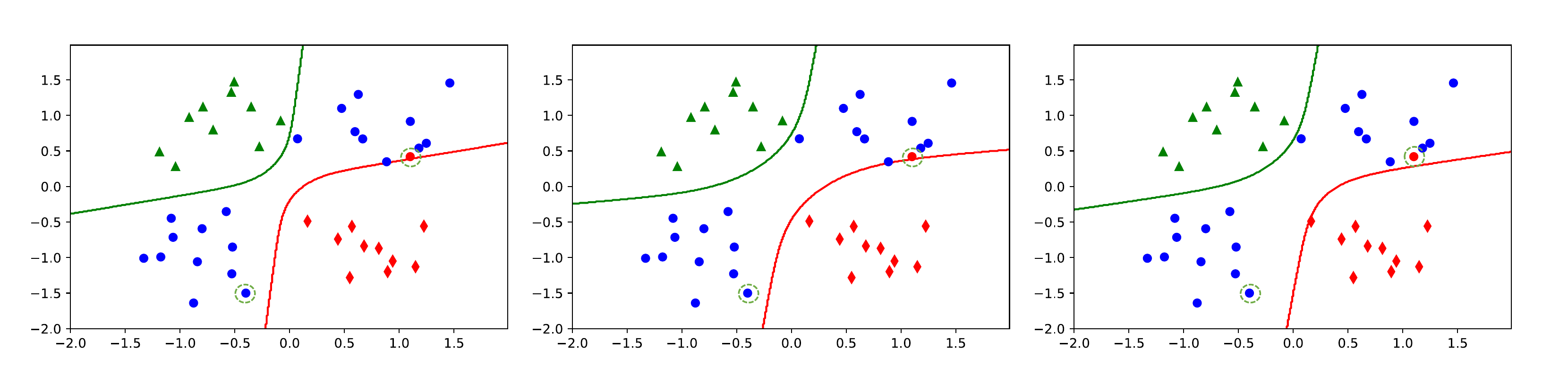}
     \vspace*{-0.2in}\caption{Design: training from scratch. Left: decision boundary for CE loss; Middle: decision boundary for LDR-KL loss with $\lambda$ set as $\mathbb E_{\x_i, t}[\lambda^i_t]$ from ALDR-KL loss; Right: decision boundary for ALDR-KL loss.   
     }\label{fig:syn-scratch}
\end{figure}
\begin{figure}
    \centering
\includegraphics[scale=0.5]{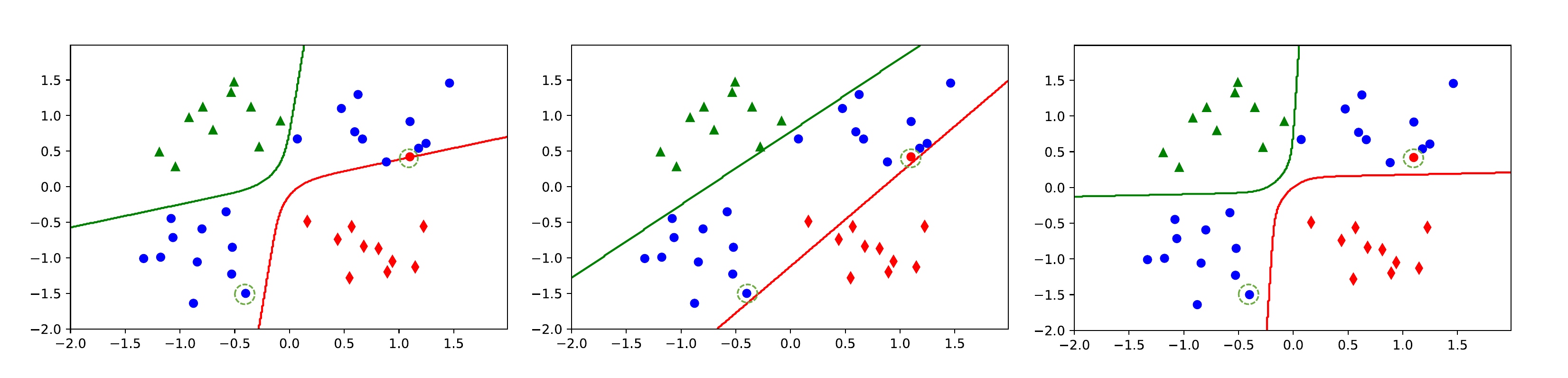}
     \vspace*{-0.2in}\caption{Design: pretrained from CE loss on clean data. Left: decision boundary for pretrained model with CE loss; Middle: decision boundary for LDR-KL loss with $\lambda$ set as $\mathbb E_{\x_i, t}[\lambda^i_t]$ from ALDR-KL loss; Right: decision boundary for ALDR-KL loss.   
     }\label{fig:syn-seq}
\end{figure}

\subsection{Running Time Cost}\label{sec:running-time}
{We provide the empirical running time cost for all baselines in Table~\ref{tab:running-time}. Each entry stands for mean and standard deviation for 100 consecutive epochs running on a x86\_64 GNU/Linux cluster with NVIDIA GeForce GTX 1080 Ti GPU card.}
\begin{table}[htbp]
  \centering
  \caption{Running time for 15 baselines on 7 benchmark datasets (seconds per epoch).}
  \resizebox{1.\textwidth}{!}{
    \begin{tabular}{llllllll}
    \toprule
    Loss  & Vowel & Letter & News20 & ALOI  & Kuzushiji-49 & CIFAR100 & Tiny-ImageNet \\
    \hline
    ALDR-KL & 0.176(0.009) & 0.933(0.033) & 2.002(0.068) & 7.029(0.091) & 121.123(1.767) & 31.923(0.381) & 61.204(0.26) \\
    LDR-KL & 0.172(0.006) & 0.783(0.033) & 1.868(0.107) & 5.779(0.098) & 116.774(0.512) & 31.507(0.826) & 59.106(0.318) \\
    CE    & 0.17(0.008) & 0.719(0.027) & 1.846(0.095) & 5.157(0.074) & 116.612(0.296) & 31.49(0.289) & 58.687(0.141) \\
    SCE   & 0.167(0.008) & 0.756(0.024) & 2.082(0.081) & 5.903(0.175) & 120.376(0.275) & 29.97(0.156) & 60.336(0.456) \\
    GCE   & 0.165(0.008) & 0.771(0.029) & 1.869(0.079) & 5.772(0.42) & 122.294(0.268) & 30.373(0.187) & 61.026(0.213) \\
    TGCE  & 0.178(0.009) & 0.769(0.03) & 1.921(0.077) & 5.813(0.193) & 120.343(1.339) & 30.293(0.163) & 58.869(0.765) \\
    WW    & 0.155(0.009) & 0.715(0.028) & 2.067(0.068) & 5.255(0.165) & 127.304(0.49) & 29.386(0.148) & 63.249(0.666) \\
    JS    & 0.161(0.011) & 0.876(0.034) & 1.942(0.108) & 6.648(0.1) & 118.172(0.279) & 31.618(0.259) & 59.562(0.14) \\
    CS    & 0.158(0.012) & 0.73(0.036) & 1.909(0.084) & 5.393(0.178) & 120.702(1.923) & 29.278(0.139) & 60.391(1.28) \\
    RLL   & 0.172(0.008) & 0.794(0.024) & 2.118(0.077) & 6.198(0.157) & 122.021(0.339) & 30.35(0.48) & 60.609(0.785) \\
    NCE+RCE  & 0.177(0.007) & 0.818(0.037) & 2.264(0.085) & 6.326(0.194) & 116.682(1.909) & 30.139(0.124) & 59.967(0.589) \\
    NCE+AUL  & 0.174(0.007) & 0.835(0.031) & 2.154(0.121) & 6.464(0.194) & 115.613(0.766) & 29.951(0.12) & 60.022(0.49) \\
    NCE+AGCE  & 0.169(0.01) & 0.835(0.039) & 2.21(0.122) & 6.418(0.203) & 121.702(1.55) & 29.632(0.063) & 63.533(0.499) \\
    MSE   & 0.158(0.006) & 0.702(0.025) & 1.932(0.089) & 4.983(0.158) & 117.601(1.364) & 29.951(0.46) & 59.44(0.684) \\
    MAE   & 0.163(0.011) & 0.687(0.031) & 1.884(0.263) & 4.843(0.179) & 119.065(1.854) & 29.404(0.092) & 61.176(1.256) \\
    \bottomrule
    \end{tabular}}%
  \label{tab:running-time}%
\end{table}%

\subsection{Ablation Study for $\lambda$ and $\lambda_0$}\label{sec:diff-lams}
{We addtionally show the validation results for LDR-KL and ALDR-KL losses on CIFAR100 dataset with different $\lambda$ and $\lambda_0$ values in Table~\ref{tab:diff-lams}. When noisy level is higher, LDR-KL or ALDR-KL loss prefer larger $\lambda$ or $\lambda_0$ (1 over 0.1) on CIFAR100 dataset.}

\begin{table}[htbp]
  \centering
  \caption{The mean and standard deviation for top-1 validation accurcay on CIFAR100 dataset. The best choice is marked as bold.}
  \resizebox{1.\textwidth}{!}{
    \begin{tabular}{l|lll|lll}
    \toprule
    Loss  & \multicolumn{3}{c}{LDR-KL} & \multicolumn{3}{c}{ALDR-KL} \\
    \hline
    Noise setting & $\lambda$=0.1 & $\lambda$=1 & $\lambda$=10 & $\lambda_0$=0.1 & $\lambda_0$=1 & $\lambda_0$=10 \\
    \hline
    Clean & \textbf{68.948(0.211)} & 66.848(0.465) & 61.528(0.665) & \textbf{68.456(0.351)} & 66.54(0.304) & 60.702(0.332) \\
    Uniform 0.3 & \textbf{40.658(0.218)} & 38.918(0.337) & 36.914(0.466) & \textbf{40.872(0.259)} & 38.888(0.24) & 36.958(0.255) \\
    Uniform 0.6 & 13.462(0.831) & \textbf{16.084(0.467)} & 14.766(0.306) & 13.196(0.821) & \textbf{16.038(0.55)} & 14.976(0.158) \\
    Uniform 0.9 & 1.414(0.086) & \textbf{1.46(0.088)} & 1.348(0.063) & 1.444(0.084) & \textbf{1.468(0.044)} & 1.366(0.118) \\
    CD 0.1 & \textbf{59.346(0.258)} & 56.478(0.376) & 51.696(0.718) & \textbf{59.368(0.32)} & 55.99(0.328) & 52.252(0.416) \\
    CD 0.3 & \textbf{43.198(0.434)} & 40.098(0.375) & 33.89(1.045) & \textbf{43.272(0.34)} & 39.66(0.445) & 34.354(1.212) \\
    CD 0.5 & 32.026(0.373) & \textbf{32.344(0.373)} & 25.672(0.379) & 31.916(0.286) & \textbf{32.378(0.333)} & 26.396(0.637) \\
    \bottomrule
    \end{tabular}}
  \label{tab:diff-lams}%
\end{table}%

\end{document}